\documentclass{article}

    \PassOptionsToPackage{numbers, compress}{natbib}
 \usepackage[preprint]{neurips_2026}

\usepackage[utf8]{inputenc} 
\usepackage[T1]{fontenc}    
\usepackage{hyperref}       
\usepackage{url}            
\usepackage{booktabs}       
\usepackage{amsfonts}       
\usepackage{nicefrac}       
\usepackage{microtype}      
\usepackage{xcolor}         

\usepackage{graphicx}
\usepackage{algorithm}
\usepackage{algorithmic}
\usepackage{amsthm}
\usepackage{amssymb}   
\usepackage{amsthm}
\usepackage{wrapfig}
\newtheoremstyle{myremark}%
  {}{}
  {\normalfont}
  {}
  {\bfseries}
  {.}
  { }
  {}

\theoremstyle{plain}
\newtheorem{theorem}{Theorem}

\theoremstyle{myremark}
\newtheorem{remark}{Remark}
\newtheorem{assumption}{Assumption}
\usepackage{amsmath}
\usepackage{svg} 

\title{RareCP: Regime-Aware Retrieval for Efficient Conformal Prediction}

\author{
    Manuel Heurich \\
    Institute for Computer Science\\
    Freie Universität Berlin\\
    Arnimallee 7 14195 Berlin\\
    \texttt{manuel.heurich@fu-berlin.de} \\
    \And
    Maximilian Granz \\
    Institute for Computer Science\\
    Freie Universität Berlin\\
    Arnimallee 7 14195 Berlin\\
    \texttt{maximilian.granz@fu-berlin.de} \\
    \And
    Tim Landgraf \\
    Institute for Computer Science\\ 
    Freie Universität Berlin\\
    Arnimallee 7 14195 Berlin\\
    \texttt{tim.landgraf@fu-berlin.de}}
\begin{document}

\maketitle

\begin{abstract}
    Recent advances in uncertainty quantification for time series forecasting show that conformal prediction can provide reliable prediction intervals, yet standard conformal methods are often inefficient under temporal dependence, drift, and heterogeneous error behavior. Existing methods typically either update miscoverage rates over time or learn unconstrained calibration weights, without explicitly separating two central sources of nonstationarity: smoothly drifting error distributions and co-existing distinct error regimes. We introduce \textit{RareCP}, a regime-aware retrieval method for adaptive conformal time series prediction. RareCP learns local calibration representations through a mixture of cosine-attention experts that each capture distinct error regimes, while a compact hypernetwork adapts the kernel parameters to track temporal drift. Given a new forecasting context, RareCP retrieves the top-k most relevant calibration examples, assigns similarity weights, and forms a weighted conformal quantile over their signed residuals, yielding asymmetric prediction intervals. The adaptive kernel is trained using a smooth interval score objective, with a parameter-space anchor to a lightweight teacher kernel to preserve stable local representations. On the GIFT-Eval benchmark, RareCP improves interval efficiency over recent conformal baselines and foundation model uncertainty estimates while maintaining empirical coverage. Ablations confirm that regime-specific experts, drift-adaptive kernels, sparse retrieval, and teacher anchoring each contribute to the final performance.
\end{abstract}

\section{Introduction}
\label{sec:intro}

Time series data is ubiquitous across critical domains including finance, healthcare, and industrial operations. While deep neural networks have dramatically advanced forecasting capabilities, reliable uncertainty quantification (UQ) remains essential for risk-aware decision-making. Traditional probabilistic forecasting methods rely on asymptotic guarantees or impose restrictive assumptions on the data-generating process, such as distributional families, linearity, independence, or stationarity \cite{Weron2014}. This gap between forecasting accuracy and reliable uncertainty estimation motivates distribution-free UQ methods that remain valid under realistic conditions.

Conformal Prediction (CP) offers a powerful, model-agnostic UQ framework that provides finite-sample coverage guarantees without distributional requirements, under the exchangeability assumption that the joint distribution of samples is invariant under reordering \cite{05cp}. These properties have led to widespread adoption across diverse domains, including image classification, graph-based modeling, and natural language processing. However, the sequential structure of time series data fundamentally violates exchangeability, introducing substantial challenges for CP in forecasting \cite{22agaci,21enbpi}.

Existing adaptations generally fall into two categories: (1) adaptive miscoverage methods that dynamically adjust the target coverage rate based on observed under- or overcoverage, and (2) weighted CP approaches that assign calibration weights to leverage varying sample relevance. Adaptive conformal methods explicitly address temporal drift by updating the miscoverage level over time, and weighted conformal prediction tackles distribution shift by reweighting calibration samples. However, adaptive methods typically react through changes in the target miscoverage level, while weighted approaches often rely on fixed or non adaptive similarity weighting schemes. We instead argue that time series forecasting errors, i.e., the signed residuals between point forecasts and realized values, exhibit structural biases that can be built directly into the calibration design. These structural biases take two forms: (i) gradual drift, where the similarities of residuals evolves as the underlying process changes, and (ii) regime structure, where residual distributions switch between distinct contexts rather than varying smoothly along a single trajectory.

\begin{figure}
    \centering
    \includegraphics[width=0.8\linewidth]{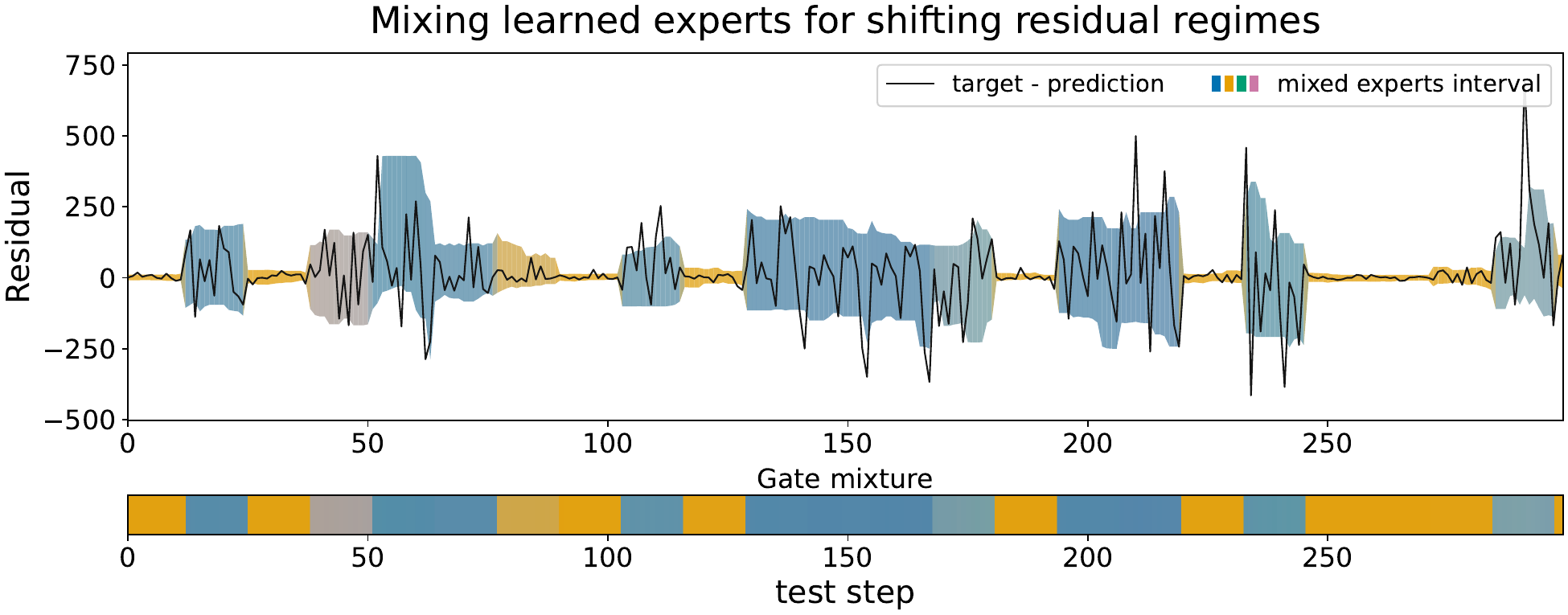}
    \caption{RareCP adapts its 80\%-interval prediction locally to the residual regime on m4\_daily. Different residual regimes recruit different experts, so the interval tightens in calm stretches and widens across volatile ones rather than inflating uniformly. See Appendix Figure~\ref{fig:full_gating} for other examples.}
    \label{fig:entry_figure}
\end{figure}

We address these limitations with \textit{RareCP}, a regime-aware retrieval
adapter for conformal time series prediction. RareCP calibrates a fixed
forecasting backbone by retrieving past signed residuals that are most relevant
to the current forecasting context under a learned linear cosine-attention
similarity. The retrieved residuals define local lower and upper conformal
quantiles, yielding asymmetric prediction intervals that adapt to temporal drift
and recurring error regimes (see Figure~\ref{fig:entry_figure}). This gives conformal prediction a retrieval-based
mechanism for improving interval efficiency while retaining a direct connection
to coverage through the size and locality of the calibration set. Our key contributions are:

\begin{itemize}
    \item \textbf{Conformal prediction is competitive with state-of-the-art quantile forecasters.}
    In the single-step-horizon quantile prediction setting, we show that
    conformal prediction equipped with RareCP is competitive with the native
    quantile-based prediction intervals of state-of-the-art time series
    forecasting models, while remaining backbone-agnostic and applicable to
    forecasters that only provide point predictions.

    \item \textbf{Retrieval conformal prediction with an explicit efficiency-coverage trade-off.}
    We introduce sparse residual retrieval for conformal calibration and analyze
    its effect through a local-CDF framework. The trade-off is controlled by
    $K_n$, the number of retrieved calibration residuals from a calibration set
    of size $n$: smaller $K_n$ can improve efficiency by focusing on local
    residual behavior, whereas increasing $K_n$ improves effective sample size
    and supports asymptotic local quantile and coverage consistency under
    stated assumptions.

    \item \textbf{An architecture matched to time series residual structure.}
    We build a linear cosine-attention retrieval architecture designed for two
    structural biases of forecasting residuals: time-drifting similarities and clustered error regimes. This allows RareCP to adapt conformal
    intervals locally without relying on unconstrained calibration weights.
\end{itemize}

\section{Related work}
\subsection{Conformal prediction}

Conformal prediction (CP) is a model-agnostic, distribution-free framework with finite-sample coverage \cite{05cp,07tutorialcp}. Classical variants include \textit{SplitCP}, which calibrates on a held-out set under exchangeability \cite{17scp}; cross-conformal prediction, which aggregates multiple calibration folds for tighter regions \cite{12crossconformal}; and \textit{Jackknife+}, which builds intervals from leave-one-out predictions while retaining finite-sample guarantees \cite{20jackknifeplus}. A broad literature extends CP across online, quantile, and tree-based settings \cite{22stableCP,21cascadeCP,24aggregationCP,20adaptivecoverage,23onlinecp,19quantileRegCP,13dtCP,14regCP}. A separate line tightens intervals via differentiable conformal surrogates, either end-to-end \cite{22conformalClassifiers,22diffCP} or as post-hoc boosting on a fixed model \cite{24boostedCP}.

\subsection{Conformal prediction in sequential settings}
\label{sec:CPseq}

Sequential data violates exchangeability, motivating CP variants that adapt calibration to drift \cite{23nexcp}. One line adjusts the miscoverage rate online: \textit{ACI} updates $\alpha$ at a fixed learning rate \citep{21aci}, \textit{AgACI} removes manual rate tuning by aggregating multiple learning-rate experts \cite{22agaci}, \textit{DtACI} reweights those experts by recent performance \cite{23dtaci}, and \textit{Conformal PID Control} treats the threshold update as a PID loop \cite{23conformalPID}.

A second line upweights calibration points by relevance to the test distribution: via covariate or label likelihood ratios \cite{20wcpcovshift,21wcplabel}, recency decay \cite{23nexcp}, or kernel-based weighted quantile regression on nonconformity scores with asymptotic guarantees under mixing \cite{26kowcpi}. Learned variants replace these hand-crafted weights with similarity in a representation space or attention over latent nonconformity scores \cite{23hopcpt,23fcp,24ctssf}, and \citet{26moecp} structure that similarity weighting as a mixture-of-experts gate over calibration regimes. A separate residual-modeling line builds intervals from bootstrap or ensemble residuals (\textit{EnbPI}) and from residual-distribution estimates (\textit{SPCI}) \cite{21enbpi,23spci}.

Retrieval augmentation \cite{21rag} has been adopted in time series forecasting via multi-resolution retrieval, relational retrieval with synthesis, diffusion guidance, and LLM-based conditioning \cite{25raft,25tsrag,22retime,24radt,24timerag}, all aimed at improving mean predictions. RareCP applies a similar retrieval principle to past signed residuals under a linear cosine-attention similarity, trained for interval efficiency rather than point accuracy.

\subsection{Hypernetworks}

Beyond the retrieval mechanism, RareCP's other architectural choice is a hypernetwork: a network that emits the weights of a primary network from a low-dimensional input \cite{16hyper}, decoupling capacity from parameter count and enabling parameter sharing across tasks. The template has been used in continual learning, where catastrophic forgetting \cite{17ewc, 17gem} has motivated task-conditioned hypernetworks that emit task specific parameters \cite{22clhyper}, and in personalised federated learning, where client-conditioned hypernetworks produce per-client weights from compact descriptors \cite{21pfedhn}. To our knowledge, RareCP is the first method to amortise the parameters of a conformal weighting rule with a hypernetwork. We train it by teacher-student distillation, but in contrast to standard distillation, which operates in output space \cite{15distill} or on intermediate features \cite{15fitnets}, our distillation acts directly in parameter space.

\section{Method}
\label{sec:method}

RareCP is a retrieval-based conformal adapter for a fixed forecasting backbone. Given a
point forecast and its recent history, RareCP retrieves past calibration residuals from
similar forecasting contexts, forms a local signed-residual distribution, and converts its
lower and upper quantiles into an asymmetric prediction interval. The retrieval rule is
sparse, cosine-normalized, and regime-aware: several low-capacity experts each construct a
local residual CDF, and a gate mixes these CDFs in residual space. Figure~\ref{fig:fig_pipeline}
summarizes the full pipeline.

\begin{figure}[t]
  \centering
  \includegraphics[width=\linewidth]{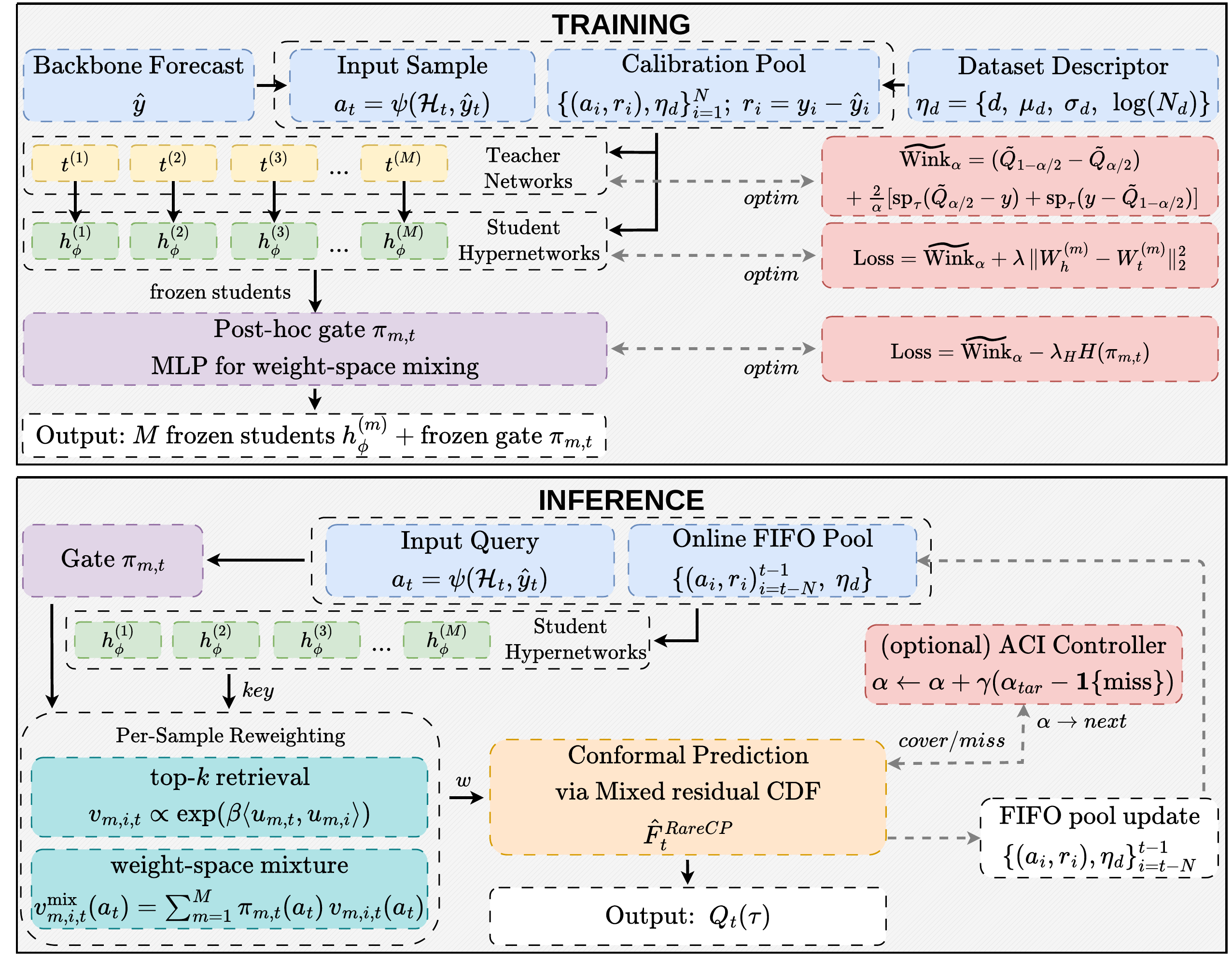}
  \caption{RareCP overview. A forecasting backbone provides a point prediction and
  history context. Dataset-conditioned retrieval experts produce normalized query and
  calibration keys, retrieve a sparse top \(k\) support by cosine similarity, and assign
  temperature-softmax weights over the retrieved residuals. A lightweight gate mixes the
  expert residual distributions, and the resulting signed-residual quantiles define the
  prediction interval.}
  \label{fig:fig_pipeline}
\end{figure}

\subsection{Local signed-residual calibration}
\label{sec:method_signed_residuals}

For dataset \(d\), let \(\mu\) be the forecasting backbone. At prediction time \(t\), the
backbone receives history \(\mathcal H_t\) and produces
\begin{equation}
\label{eq:rarecp_context}
    \widehat y_t=\mu(\mathcal H_t),
    \qquad
    a_t=\psi(\mathcal H_t,\widehat y_t)\in\mathbb R^p,
\end{equation}
where \(\psi\) is the flattened input-window representation, optionally augmented with
the point prediction. For every past calibration index \(i\), RareCP stores
\begin{equation}
\label{eq:rarecp_calibration_context_residual}
    a_i=\psi(\mathcal H_i,\widehat y_i),
    \qquad
    r_i=y_i-\widehat y_i .
\end{equation}
We calibrate signed residuals rather than absolute residuals so that the lower and upper
tails can adapt separately.

Let \(\mathcal I_{d,t}\) be the calibration indices available immediately before
predicting time \(t\), and define
\begin{equation}
\label{eq:calibration_set_current}
    \mathcal C_{d,t}
    =
    \{(a_i,r_i):i\in\mathcal I_{d,t}\}.
\end{equation}
The initial set \(\mathcal C_{d,0}\) is used to fit the RareCP adapter and initialize
retrieval. During chronological evaluation, \(\mathcal C_{d,t}\) is updated only after
\(y_t\) is observed.

Given a query \(a_t\), a support \(S_t\subseteq\mathcal I_{d,t}\), and nonnegative
weights \(w_{i,t}\) with \(\sum_{i\in S_t}w_{i,t}=1\), RareCP forms
\begin{align}
\label{eq:rarecp_weighted_cdf}
    \widehat F_t(\rho)
    &=
    \sum_{i\in S_t}
    w_{i,t}\mathbf 1\{r_i\le \rho\}, \\
\label{eq:rarecp_weighted_quantile}
    Q_t(\tau)
    &=
    \inf\{\rho\in\mathbb R:\widehat F_t(\rho)\ge \tau\}.
\end{align}
For current miscoverage level \(\alpha_t\), the prediction interval is
\begin{equation}
\label{eq:rarecp_interval}
    C_t^{\alpha_t}
    =
    \left[
        \widehat y_t + Q_t(\alpha_t/2),
        \widehat y_t + Q_t(1-\alpha_t/2)
    \right].
\end{equation}

\subsection{Regime-aware top-\texorpdfstring{\(k\)}{k} retrieval}
\label{sec:method_regime_architecture}

RareCP learns several retrieval experts to represent co-existing residual regimes. Each
expert uses a low-capacity affine key map followed by cosine normalization. Crucially,
the affine map is produced by a hypernetwork conditioned on both the current query context
and a dataset descriptor. Thus, for each prediction time \(t\), the expert constructs a
query-specific retrieval metric and applies it to both the query and all candidate
calibration contexts.

Let
\begin{equation}
\label{eq:dataset_descriptor_method}
    N_d = |\mathcal I_{d,0}|,
    \qquad
    \mu_d =
    \frac{1}{N_d}\sum_{i\in\mathcal I_{d,0}} a_i,
    \qquad
    \sigma_d =
    \left(
    \frac{1}{N_d}
    \sum_{i\in\mathcal I_{d,0}}
    (a_i-\mu_d)^{\odot 2}
    \right)^{1/2}.
\end{equation}
We use the dataset descriptor
\begin{equation}
\label{eq:eta_descriptor_method}
    \eta_d
    =
    \{d,\mu_d,\sigma_d,\log(N_d)\},
\end{equation}
where \(d\) denotes the dataset identifier, implemented as the dataset embedding or
identifier available to the hypernetwork, and \(\mu_d,\sigma_d\) are componentwise
statistics of the initial calibration contexts.

For expert \(m\in\{1,\ldots,M\}\), the hypernetwork receives the current query
\(a_t\) and the dataset descriptor \(\eta_d\), and outputs an affine retrieval map
\begin{equation}
\label{eq:query_conditioned_expert_map}
    (A_{m,t,d},b_{m,t,d})
    =
    H_{\theta_m}^{(m)}(a_t,\eta_d).
\end{equation}
This map is then applied to any context \(a\) by
\begin{equation}
\label{eq:expert_key_map_method}
    \widetilde z_{m,t,d}(a)
    =
    A_{m,t,d}a+b_{m,t,d},
    \qquad
    u_{m,t,d}(a)
    =
    \frac{\widetilde z_{m,t,d}(a)}
    {\|\widetilde z_{m,t,d}(a)\|_2}.
\end{equation}
We write
\[
    q_{m,t}=u_{m,t,d}(a_t),
    \qquad
    k_{m,i\mid t}=u_{m,t,d}(a_i),
\]
to emphasize that the calibration key \(k_{m,i\mid t}\) is computed under the
query-conditioned map induced by the current test query \(a_t\).

Expert \(m\) retrieves the \(k\) largest cosine similarities,
\begin{equation}
\label{eq:expert_retrieval_method}
    S_{m,t}
    =
    \operatorname{TopK}_{i\in\mathcal I_{d,t}}
    \bigl(q_{m,t}^{\top}k_{m,i\mid t};k\bigr),
\end{equation}
and assigns temperature-softmax weights only on this support:
\begin{equation}
\label{eq:expert_weights_method}
    v_{m,i,t}
    =
    \frac{
        \exp(q_{m,t}^{\top}k_{m,i\mid t}/T_w)
        \mathbf 1\{i\in S_{m,t}\}
    }{
        \sum_{\ell\in S_{m,t}}
        \exp(q_{m,t}^{\top}k_{m,\ell\mid t}/T_w)
    }.
\end{equation}
The normalization makes the score a nearest-neighbor rule on the unit sphere, since
\begin{equation}
\label{eq:cosine_sphere_identity}
    \|u-u_i\|_2^2=2-2u^\top u_i .
\end{equation}
Thus top-\(k\) cosine retrieval keeps the calibration residuals closest to the query
under the learned query- and dataset-conditioned representations. Smaller \(k\) enforces
stronger locality and can reduce mass on unrelated regimes; larger \(k\) increases
effective sample size and supports stronger coverage behavior.

A lightweight gate maps the query and dataset descriptor to simplex weights
\begin{equation}
\label{eq:gate_method}
    (\pi_{1,t},\ldots,\pi_{M,t})
    =
    \operatorname{softmax}(g_{\omega}(a_t,\eta_d)),
    \qquad
    \sum_{m=1}^{M}\pi_{m,t}=1 .
\end{equation}
Each expert forms its own local residual CDF, and the gate mixes these CDFs:
\begin{equation}
\label{eq:rarecp_mixed_cdf_method}
    \widehat F_t^{\mathrm{RareCP}}(\rho)
    =
    \sum_{m=1}^{M}
    \sum_{i\in S_{m,t}}
    \pi_{m,t}v_{m,i,t}\mathbf 1\{r_i\le \rho\}.
\end{equation}
Equivalently, the final weight assigned to residual \(r_i\) is
\begin{equation}
\label{eq:effective_rarecp_weight}
    w_{i,t}^{\mathrm{RareCP}}
    =
    \sum_{m=1}^{M}
    \pi_{m,t}v_{m,i,t}\mathbf 1\{i\in S_{m,t}\}.
\end{equation}
The interval in Eq.~\eqref{eq:rarecp_interval} is obtained by applying the weighted
quantile in Eq.~\eqref{eq:rarecp_weighted_quantile} to
\(\widehat F_t^{\mathrm{RareCP}}\). Mixing residual distributions, rather than keys or
logits, preserves each expert's local neighborhood while allowing the adapter to switch
between clustered error regimes.

Unlike MoE-CP, which uses MoE gating probability vectors
as soft domain assignments and weights residuals by gating-vector similarity, RareCP makes
each expert a separate top-\(k\) retriever and lets the gate mix the resulting local CDFs
in residual-weight space~\cite{26moecp}.
\paragraph{Top-\texorpdfstring{\(K_n\)}{Kn} calibration behavior.}
We now formalize the trade-off between local adaptivity and calibration reliability induced
by this retrieval mechanism. Let
\(\widehat F_n(\rho\mid u)\) be the weighted residual CDF built from the top \(K_n\)
neighbors of a normalized query state \(u\), and let \(F_u\) be the corresponding
conditional signed-residual CDF. Theorem~\ref{thm:short_topk_coverage} gives
\begin{equation}
\label{eq:method_theory_decomposition}
    \Delta_n(u)
    :=
    \sup_{\rho\in\mathbb R}
    |\widehat F_n(\rho\mid u)-F_u(\rho)|
    \le
    R_n(u)+D_n(u)+\omega_u(h_n(u)).
\end{equation}
Here \(R_n(u)=O_p(K_n^{-1/2})\) is finite-support stochastic error, \(D_n(u)\) measures
leakage from temporal dependence, \(h_n(u)\) is the radius of the retrieved neighborhood,
and \(\omega_u(h_n(u))\) is localization bias. Under the local mass and H\"older
conditions stated in Appendix~\ref{subsec:normalized_learned_topk_short},
\begin{equation}
\label{eq:method_topk_rate}
    \Delta_n(u)
    =
    O_p\!\left(
        K_n^{-1/2}
        +
        \left(\frac{K_n}{n}\right)^{\beta/d_{\mathrm{loc}}}
    \right)
    +
    D_n(u).
\end{equation}
The conditions require \(K_n\to\infty\) and \(K_n/n\to0\): growing support controls
stochastic error, while vanishing relative support preserves locality. For a finite
expert mixture, the same argument applies expert-wise and the mixed-CDF error is bounded
by the gate-weighted sum of the expert errors (see Remark~\ref{rem:finite_mixture_topk}). Our experiments use a fixed finite \(k\).  Since \(K_n/n\to 0\) only requires
sublinear growth, we assume that fixed-\(k\) perform well for our finitely growing calibration sets. For benchmarks with more timesteps, \(k\) could be increased accordingly, but adaptive growing schedules are a natural extension to our work.

Finally, we note that the top-\(K_n\) analysis assumes a conservative split in
which the retrieval rule is fixed before the residuals used for conformal quantiles are
observed. Our experiments instead follow the data-efficient protocol of learned local CP
baselines such as HopCPT~\cite{23hopcpt}: we reuse the initial calibration period both to
learn the similarity and gating rules and to provide the residuals in
\(\mathcal C_{d,t}\). This violates the strict split condition used in the analysis, but
improves finite-sample efficiency, especially for small calibration sets where a separate
split leaves too little data for learning stable retrieval and gating representations. We
expect a sample-size crossover: with sufficiently large calibration sets, a fully split
protocol may become advantageous. We leave a detailed split-ablation to future work and
evaluate RareCP under the data-efficient protocol in Section~\ref{sec:4.2}.
\subsection{Training and inference}
\label{sec:method_training_inference}

RareCP is fitted on the initial calibration set \(\mathcal C_{d,0}\) before chronological
test evaluation. Training uses leave-one-out calibration episodes introduced by \cite{23hopcpt}. For each minibatch
\(\mathcal B\subseteq\mathcal I_{d,0}\) and query element \(j\in\mathcal B\), the model
excludes \(j\), retrieves residuals from \(\mathcal B\setminus\{j\}\), forms an interval
for \(j\), and scores that interval against the observed target \(y_j\). This mirrors
test-time calibration: the query target is unavailable when its interval is constructed.

The training objective is a smooth relaxation of the Winkler interval score. Evaluation
uses the hard score
\begin{equation}
\label{eq:winkler_score_method}
    \operatorname{Wink}_{\alpha}(\ell,u;y)
    =
    (u-\ell)
    +
    \frac{2}{\alpha}(\ell-y)\mathbf 1\{y<\ell\}
    +
    \frac{2}{\alpha}(y-u)\mathbf 1\{y>u\}.
\end{equation}
For optimization, the hard weighted quantile is replaced by a sigmoid-CDF quantile and
the miss indicators are replaced by softplus penalties. The full smooth surrogate is
defined in Appendix~\ref{subsec:smooth_winkler_loss}.

The retrieval experts are trained first. Each student hypernetwork is anchored to a
lightweight prefit teacher map
\begin{equation}
\label{eq:teacher_map_method}
    t_d^{(m)}(a)=B_{m,d}a+c_{m,d},
\end{equation}
which stabilizes the learned representations and discourages collapsed similarity maps. After
the experts are frozen, the gate is trained on leave-one-out mixed-CDF intervals, with an
entropy regularizer to discourage premature collapse to a single expert. The expert and
gate losses are given in Appendix~\ref{subsec:smooth_winkler_loss}.

At inference time, RareCP computes \(a_t\), retrieves top \(k\) residuals for each expert
using Eq.~\eqref{eq:expert_retrieval_method}, mixes their residual CDFs using
Eq.~\eqref{eq:rarecp_mixed_cdf_method}, and returns the interval in
Eq.~\eqref{eq:rarecp_interval}. Because the expert maps depend on the current query
\(a_t\), RareCP caches raw calibration contexts and residuals; the normalized calibration
keys are recomputed under the query-conditioned map for each test point. To account for distributional drift, we adapt the target miscoverage level at
test time using adaptive conformal inference~\cite{21aci}.
With
\(n_t=|\mathcal I_{d,t}|\), exact retrieval costs \(O(Mn_td_z)\) dot products per query,
and the final weighted quantile is computed over at most \(Mk\) residual entries.
\section{Experiments}
\label{sec:4.0}

We evaluate our RCP framework on the comprehensive GIFT-Eval benchmark \cite{24gifteval}. This section presents our evaluation results. All experiments were conducted on a single NVIDIA A6000 GPU (48GB VRAM) and a 64-core CPU with 256GB RAM.

\subsection{Setup}
\label{sec:4.1}

\paragraph{Benchmarks.}
We evaluate on GIFT-Eval in two configurations: \textit{Bench10} (10 core datasets) and \textit{Bench22} (the full 22-dataset suite, a superset of Bench10), covering energy, transport, climate, web, finance, and demographics. The datasets span sampling intervals from five minutes to weekly and heterogeneous noise, seasonality, scale, and drift. For each, we cap the test region at the most recent 30{,}000 samples (or the full series if shorter) and construct in chronological order train, calibration, and 7{,}500-sample evaluation splits. Bench10 is additionally reported at a 100{,}000 cap and the full set to probe scaling. HPO uses strictly earlier validation splits from the GIFT-Eval pre-train region so no test-region context enters tuning. All experiments use a 64-step context window and single-step-ahead targets and per-dataset statistics are reported in \ref{apx:bench}.

\paragraph{Baselines and backbones.}
RareCP is paired with four point-forecast backbones to demonstrate backbone-agnosticism: \textit{TiRex}, \textit{TimesFM-2.5} and \textit{TabPFN-TS} (zero-shot FMs) plus \textit{ARIMA} (AutoARIMA from \textit{statsforecast} \cite{22statsforecast}, default parameters). We compare RareCP against two baseline families on the same backbone predictions and calibration set: CP methods (Uniform SplitCP, NexCP, ACI, DtACI, HopCPT, KOWCPI, ResCP \cite{17scp, 23nexcp, 21aci, 23dtaci, 23hopcpt, 26kowcpi, 25rescp}) and FM native UQ heads (TiRex, Chronos-2, Chronos-Bolt, Moirai-2.0, TimesFM-2.5, TabPFN-TS, TOTO, Lag-Llama \cite{25tirex, 24chronos, 24moirai, 24timesfm, 25tabpfnts, 24toto, 24lagllama}). Backbone predictions are precomputed and cached, so every method scores identical point forecasts per query.

\paragraph{Training and calibration.}
RareCP's teacher encoders, student hypernetworks, and post-hoc gate are trained consecutively on the calibration split; hyperparameters come from a 100-trial HPO run on the Bench10 validation splits, matching the trial budget given to HopCPT, KOWCPI, and ResCP (full hyperparameters in \ref{apx:hyperparams}). The same Bench10-tuned hyperparameters are reused without modification on Bench22 and on the larger benchmark scenarios. At test time, the shared rolling-window calibration set $\mathcal{D}_{\text{cal}}$ is initialized from the most recent validation samples and updated first-in-first-out (FIFO) as each test query arrives. Non-deterministic methods are run on ten seeds.

\subsection{Results}
\label{sec:4.2}

Table \ref{tab:cp_methods_summary_small} summarises CP and FM-native uncertainty quantification performance at the $80\%$ target. Per-dataset breakdowns are in Tables \ref{tab:cp_methods_bench10} -\ref{tab:fm_native_uq_bench10_bench22}. We report \emph{nWink} (mean Winkler $/$ std$(y)$), \emph{nW} (mean width $/$ std$(y)$), and empirical coverage \emph{Cov}. Since std$(y)$ is a dataset-level constant, cells are directly comparable across methods and backbones. We treat normalised Winkler (Eq. \ref{eq:winkler_score_method}) as the primary metric. Width and coverage in isolation are exploitable as a method can shrink intervals simply by undercovering, whereas Winkler jointly penalises both. Among methods at the same coverage, the lower Winkler scoring one redistributes its width better (narrower on easy queries, wider on hard ones), which is exactly the behaviour we expect from an adaptive method.

\begin{table}[t]
  \caption{Result summary at 80\% target coverage. Cells show nWink (mean Winkler $/$ std$(y)$) and nW (mean width $/$ std$(y)$) $/$ Cov, where std$(y)$ is the standard deviation of the target on the eval split. \emph{Top:} per-backbone CP Average rows from Tables \ref{tab:cp_methods_bench10}, \ref{tab:cp_methods_bench22}, \ref{tab:cp_methods_bench10_100k}, \ref{tab:cp_methods_bench10_full}. \emph{Bottom:} FM native UQ from Table \ref{tab:fm_native_uq_bench10_bench22}. All rows show unweighted means over ten seeds for non-deterministic methods. Their std are displayed in the full result tables linked above. \textbf{Bold} marks the smallest displayed nWink per row among entries with $\mathrm{Cov} \ge 0.70$.}
  \label{tab:cp_methods_summary_small}
  \centering
  \resizebox{\textwidth}{!}{%
  \begin{tabular}{lcccccccccccc}
    \multicolumn{13}{c}{{\large\bfseries\itshape Main Results}} \\[2pt]
    \toprule
    & \multicolumn{2}{c}{Uniform} & \multicolumn{2}{c}{ACI} & \multicolumn{2}{c}{KOWCPI} & \multicolumn{2}{c}{HopCPT} & \multicolumn{2}{c}{ResCP} & \multicolumn{2}{c}{RareCP\,\textit{(Ours)}} \\
    \cmidrule(lr){2-3} \cmidrule(lr){4-5} \cmidrule(lr){6-7} \cmidrule(lr){8-9} \cmidrule(lr){10-11} \cmidrule(lr){12-13}
    \textbf{Backbone} & nWink $\downarrow$ & nW $\downarrow$ / Cov $\uparrow$ & nWink $\downarrow$ & nW $\downarrow$ / Cov $\uparrow$ & nWink $\downarrow$ & nW $\downarrow$ / Cov $\uparrow$ & nWink $\downarrow$ & nW $\downarrow$ / Cov $\uparrow$ & nWink $\downarrow$ & nW $\downarrow$ / Cov $\uparrow$ & nWink $\downarrow$ & nW $\downarrow$ / Cov $\uparrow$ \\
    \midrule
    \multicolumn{13}{c}{\textit{Bench10}} \\
    \addlinespace[2pt]
    TiRex & 0.72 & 0.36 / 0.79 & 0.69 & 0.38 / 0.80 & 0.74 & 0.27 / 0.65 & 0.70 & 0.38 / 0.78 & 0.68 & 0.36 / 0.77 & \textbf{0.61} & 0.38 / 0.80 \\
    TimesFM-2.5 & 0.77 & 0.38 / 0.79 & 0.73 & 0.40 / 0.80 & 0.74 & 0.28 / 0.67 & 0.77 & 0.39 / 0.78 & 0.70 & 0.36 / 0.76 & \textbf{0.62} & 0.40 / 0.80 \\
    TabPFN-TS & 0.81 & 0.39 / 0.79 & 0.77 & 0.41 / 0.80 & 0.73 & 0.30 / 0.70 & 0.81 & 0.40 / 0.79 & 0.72 & 0.37 / 0.72 & \textbf{0.63} & 0.39 / 0.79 \\
    ARIMA & 0.90 & 0.46 / 0.79 & 0.84 & 0.48 / 0.80 & 0.84 & 0.35 / 0.70 & 0.90 & 0.47 / 0.77 & 0.83 & 0.42 / 0.72 & \textbf{0.68} & 0.43 / 0.80 \\
    \midrule
    \multicolumn{13}{c}{\textit{Bench22}} \\
    \addlinespace[2pt]
    TiRex & 1.09 & 0.48 / 0.77 & 1.06 & 0.53 / 0.80 & 1.11 & 0.44 / 0.66 & 1.13 & 0.46 / 0.75 & 1.31 & 0.35 / 0.66 & \textbf{1.01} & 0.48 / 0.75 \\
    TimesFM-2.5 & 1.08 & 0.46 / 0.77 & 1.04 & 0.51 / 0.79 & 1.37 & 0.29 / 0.58 & 1.06 & 0.41 / 0.75 & 1.25 & 0.35 / 0.67 & \textbf{1.00} & 0.47 / 0.75 \\
    TabPFN-TS & 1.13 & 0.48 / 0.77 & \textbf{1.09} & 0.52 / 0.79 & 1.31 & 0.33 / 0.63 & 1.15 & 0.47 / 0.75 & 1.34 & 0.33 / 0.60 & 1.13 & 0.45 / 0.72 \\
    ARIMA & 1.29 & 0.56 / 0.77 & 1.24 & 0.61 / 0.80 & 1.56 & 0.39 / 0.61 & 1.34 & 0.52 / 0.74 & 1.53 & 0.46 / 0.61 & \textbf{1.10} & 0.51 / 0.73 \\
    \midrule
    \multicolumn{13}{c}{\textit{Bench10\_100k}} \\
    \addlinespace[2pt]
    TiRex & 0.67 & 0.32 / 0.78 & 0.64 & 0.34 / 0.80 & 0.68 & 0.21 / 0.63 & 0.65 & 0.34 / 0.78 & 0.62 & 0.33 / 0.78 & \textbf{0.52} & 0.33 / 0.80 \\
    ARIMA & 0.85 & 0.42 / 0.78 & 0.79 & 0.44 / 0.80 & 0.77 & 0.32 / 0.70 & 0.84 & 0.43 / 0.77 & 0.77 & 0.40 / 0.74 & \textbf{0.56} & 0.36 / 0.80 \\
    \midrule
    \multicolumn{13}{c}{\textit{Bench10\_full}} \\
    \addlinespace[2pt]
    TiRex & 0.68 & 0.33 / 0.78 & 0.64 & 0.34 / 0.80 & 0.69 & 0.20 / 0.62 & 0.66 & 0.35 / 0.78 & 0.63 & 0.33 / 0.77 & \textbf{0.52} & 0.33 / 0.80 \\
    ARIMA & 0.85 & 0.42 / 0.78 & 0.79 & 0.44 / 0.80 & 0.77 & 0.32 / 0.70 & 0.85 & 0.44 / 0.77 & 0.78 & 0.40 / 0.74 & \textbf{0.57} & 0.37 / 0.80 \\
    \addlinespace[10pt]
    \midrule[1pt]
    \addlinespace[6pt]
    \multicolumn{13}{c}{\textit{Foundation model native UQ}} \\
    \addlinespace[2pt]
    & \multicolumn{2}{c}{TiRex} & \multicolumn{2}{c}{Chronos-2} & \multicolumn{2}{c}{Moirai-2} & \multicolumn{2}{c}{TimesFM-2.5} & \multicolumn{2}{c}{TabPFN-TS} & \multicolumn{2}{c}{Toto} \\
    \cmidrule(lr){2-3} \cmidrule(lr){4-5} \cmidrule(lr){6-7} \cmidrule(lr){8-9} \cmidrule(lr){10-11} \cmidrule(lr){12-13}
    \textbf{Benchmark} & nWink $\downarrow$ & nW $\downarrow$ / Cov $\uparrow$ & nWink $\downarrow$ & nW $\downarrow$ / Cov $\uparrow$ & nWink $\downarrow$ & nW $\downarrow$ / Cov $\uparrow$ & nWink $\downarrow$ & nW $\downarrow$ / Cov $\uparrow$ & nWink $\downarrow$ & nW $\downarrow$ / Cov $\uparrow$ & nWink $\downarrow$ & nW $\downarrow$ / Cov $\uparrow$ \\
    \addlinespace[2pt]
    \textit{Bench10} & \textbf{0.59} & 0.39 / 0.80 & 0.64 & 0.38 / 0.76 & 0.63 & 0.40 / 0.80 & 0.66 & 0.42 / 0.83 & 0.66 & 0.38 / 0.78 & 0.71 & 0.43 / 0.76 \\
    \textit{Bench22} & 0.83 & 0.50 / 0.80 & 0.83 & 0.48 / 0.76 & \textbf{0.78} & 0.46 / 0.78 & 0.88 & 0.50 / 0.83 & 0.86 & 0.48 / 0.76 & 0.91 & 0.50 / 0.75 \\
    \bottomrule
  \end{tabular}%
  }
\end{table}

\paragraph{CP methods.}
RareCP attains the lowest average nWink on every Bench10 backbone and on three of four backbones in Bench22, trailing only ACI on TabPFN-TS in Bench22. ResCP is the consistent runner-up on Bench10; ACI takes that role on Bench22, reflecting the drift-adaptation strength of the same ACI controller that RareCP wraps. KOWCPI produces the narrowest intervals throughout but undercovers by at least $0.1$. Scaling Bench10's test region from the $30$k default to $100$k and to the full series further widens RareCP's lead on TiRex and ARIMA, where richer regime structure rewards retrieval over uniform smoothing. The strongest single per-dataset improvement is ARIMA on electricity\_H at the full series, where RareCP roughly halves Uniform's nWink (Tables~\ref{tab:cp_methods_bench10_100k}, \ref{tab:cp_methods_bench10_full}). Per-dataset gains concentrate on series with sharp regime shifts (M\_DENSE\_H, solar\_H, electricity\_H, m4\_quarterly, m4\_monthly). Recall that RareCP's hyperparameters are tuned only on Bench10, so the Bench22 and larger-set results measure transfer rather than fit.

\paragraph{Comparison with foundation model native UQ.}
On Bench10, pairing RareCP with an FM as backbone is competitive with the FM's own native UQ, despite adding only a head of $1.5$ to $5.5$M parameters on top of backbones spanning $11$M to $200$M (Table \ref{tab:model_param_counts}). The strongest native models pull ahead on Bench22, where RareCP carries over its Bench10 hyperparameters without refitting. Two structural advantages of the CP wrapper remain. First, CP provides finite-sample coverage guarantees that FM native UQ does not. Second, CP is agnostic to the backbone and wraps any point forecaster, including ARIMA, where no native UQ exists and still beats most zero-shot FMs. Within the native block, no single FM dominates: TiRex and Moirai-2 lead on different scenarios (Table \ref{tab:fm_native_uq_bench10_bench22}).

\subsection{Ablation study}
\label{sec:4.3}

Figure~\ref{fig:stepwise-ablation-waterfall} decomposes RareCP relative to
uniform weighting. A low-capacity trained dense similarity model with
cosine-normalized scores improves Winkler by $13.9\%$. Since this step changes
both the scorer and the similarity normalization, we treat it as a learned dense
reference rather than attributing the gain to cosine attention alone. On TiRex
Bench10, this variant reaches nWink $0.635$, compared to HopCPT at $0.70$ and
ResCP at $0.68$ in Table~\ref{tab:cp_methods_summary_small}. Adding hard
top-$k$ retrieval gives another $4.6\%$, showing that sparse local calibration
improves beyond learned dense weighting.

The remaining components add smaller but consistent gains: the hypernetwork adds
$1.7\%$ by conditioning retrieval on the current test context, teacher anchoring
adds $0.4\%$, residual-regime mixtures add $2.1\%$, and ACI adds $1.0\%$ for
online coverage-drift correction. Together, these steps yield the final RareCP
improvement of $23.6\%$ over uniform weighting.

Table~\ref{tab:norm-winkler-architecture} shows that the hypernetwork gain is
not simply due to added capacity. All MLP and hypernetwork variants were tuned
with Optuna TPE for $100$ trials. Fixed MLP encoders reuse calibration
embeddings across test points, while RareCP uses each test sample to emit a
small retrieval and weighting map. This makes retrieval time-dependent, but also
requires refactoring calibration samples for each query, motivating the
low-capacity emitted model. 
Our interpretation is that having too much capacity for the weighting mechanism can easily lead to overfitting. The solution, we propose, is to scale the parameters in the time dimension by allowing the current test target more control over the retrieved similarites while preventing similarity collapse. 
Empirically, the hypernetwork achieves the best
nWink, $0.6221 \pm 0.0014$, outperforming the best MLP,
$0.6325 \pm 0.0020$, and the larger MLP-5, $0.6338 \pm 0.0034$.

Finally, Table~\ref{tab:expert-count-normalized-uniform} varies only the number
of residual regime experts with the best RareCP hyperparameters fixed. Increasing
from one to three experts gives the main gain, reducing nWink from
$0.6195 \pm 0.0024$ to $0.6105 \pm 0.0014$. More experts improve only slightly,
reaching $0.6065 \pm 0.0025$ at $30$ experts, while runtime rises from
$182.4$s to $842.5$s. Thus, multiple residual regimes help, but most of the
benefit comes from separating a small number of regimes.

\begin{figure}[t]
\centering
\includegraphics[width=0.8\linewidth]{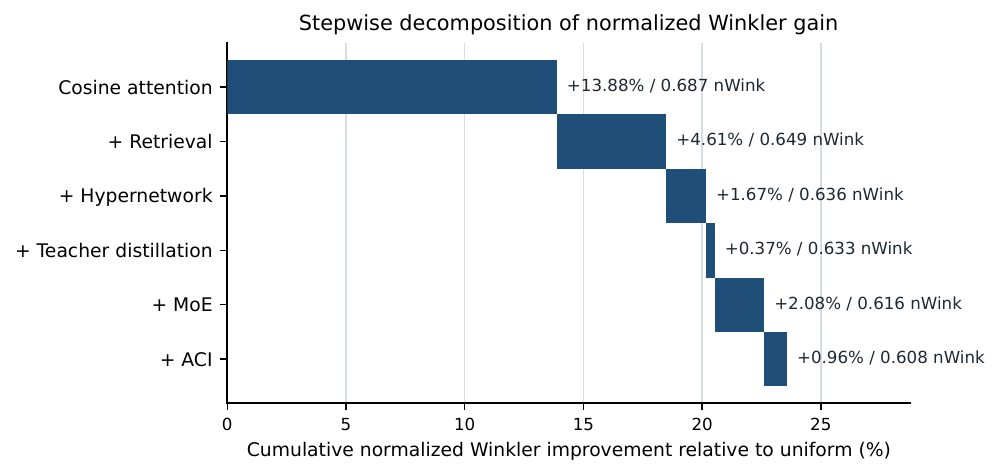}
\caption{
Stepwise RareCP ablation on Bench10. Bars show additional percentage
improvement in Winkler score over uniform weighting, averaged over forecasting
backbones. The first bar is a trained dense linear similarity model without hard
retrieval.
}
\label{fig:stepwise-ablation-waterfall}
\end{figure}

\section{Conclusion}
\label{sec:conclusion}

We introduced \textit{RareCP}, a backbone-agnostic conformal adapter for
time series forecasting that calibrates prediction intervals by retrieving
residuals from locally relevant error regimes. Its main contribution is a learned
local retrieval mechanism: sparse top-$k$ retrieval selects relevant calibration
contexts produced by time-conditioned local experts that are build for the structural bias of the residual distributions of time series forecasting models.


Across GIFT-Eval, RareCP improves normalized Winkler score over the strongest CP
baselines on Bench10 for all four forecasting backbones and remains competitive
on the larger Bench22 setting. The scaling experiments show stronger gains as
larger calibration streams expose more repeated residual regimes, while the
ablations confirm the importance of learned similarity, sparse retrieval,
hypernetwork conditioning, teacher anchoring, expert mixtures, and ACI
correction. Overall, the results suggest that conformal calibration for
time series forecasting benefits from modeling drift and residual-regime
structure separately.

\subsection{Limitations}
\label{sec:limitations}
RareCP still has important limitations. The theoretical analysis covers a clean
setting in which the retrieval rule is fixed before calibration, or learned from
data with residual labels distinct from the calibration residuals. It therefore
does not fully resolve the finite-sample trade-off between learning a stronger
retrieval rule and reserving data for calibration. Empirically, RareCP maintains
coverage close to the target on Bench10 and at larger Bench10 scales, but
coverage is weaker on the more heterogeneous Bench22 split, suggesting that
sparse or abrupt drift remains challenging. RareCP also has a higher offline
training cost than simpler conformal baselines.

\subsection{Future Work}
Future work should study the learning-calibration split at larger scale, where
using separate data to learn the retrieval metric and to compute conformal quantiles
may yield both cleaner theory and stronger empirical performance.
Another direction is to reduce training cost through cached calibration
embeddings, approximate retrieval, or lighter expert and gating
parameterizations. These extensions preserve the central principle of RareCP:
prediction intervals should be calibrated using residuals from the same local
error regime, while the retrieval geometry adapts as the time series evolves.

\clearpage




\bibliographystyle{unsrtnat}
\bibliography{references}




\appendix

\clearpage

\section{Smooth Winkler surrogate and training losses}
\label{subsec:smooth_winkler_loss}

This section defines the differentiable surrogate used to train RareCP. Evaluation uses
the hard weighted quantile in Eq.~\eqref{eq:rarecp_weighted_quantile} and the hard Winkler
score in Eq.~\eqref{eq:winkler_score_method}; the surrogate is used only for
optimization.

Consider one leave-one-out query with weighted residual support
\[
    P=\{(r_i,p_i)\}_{i=1}^{s},
    \qquad
    p_i\ge0,
    \qquad
    \sum_{i=1}^{s}p_i=1.
\]
Sort by residual value and write the sorted support as
\(\{(r_{(i)},p_{(i)})\}_{i=1}^{s}\). Define the left and right CDF endpoints of bin
\(i\) by
\begin{equation}
\label{eq:smooth_cdf_endpoints}
    C_{i-1}=\sum_{j<i}p_{(j)},
    \qquad
    C_i=\sum_{j\le i}p_{(j)},
    \qquad
    C_0=0.
\end{equation}
For quantile level \(q\in(0,1)\), sigmoid temperature \(\tau_q>0\), and
\(\sigma(x)=(1+\exp(-x))^{-1}\), define
\begin{equation}
\label{eq:smooth_quantile_bin_weight}
    b_i(q;\tau_q)
    =
    \left[
        \sigma\!\left(\frac{q-C_{i-1}}{\tau_q}\right)
        -
        \sigma\!\left(\frac{q-C_i}{\tau_q}\right)
    \right]_+,
    \qquad
    \lambda_i(q;\tau_q)
    =
    \frac{b_i(q;\tau_q)}
    {\sum_{\ell=1}^{s}b_\ell(q;\tau_q)} .
\end{equation}
The smooth weighted quantile is
\begin{equation}
\label{eq:smooth_cdf_sigmoid_quantile}
    \widetilde Q_q(P;\tau_q)
    =
    \sum_{i=1}^{s}
    \lambda_i(q;\tau_q)r_{(i)} .
\end{equation}
Large \(\tau_q\) spreads gradient across a broad CDF neighborhood; as
\(\tau_q\downarrow0\), the mass concentrates on the hard weighted-quantile bin.

For target residual \(r_j=y_j-\widehat y_j\), define the temperature-softplus penalty
\begin{equation}
\label{eq:temperature_softplus}
    \operatorname{softplus}_{\tau_p}(x)
    =
    \tau_p\log\bigl(1+\exp(x/\tau_p)\bigr),
    \qquad
    \tau_p>0.
\end{equation}
The smooth residual-space Winkler loss is
\begin{equation}
\label{eq:smooth_winkler_loss}
\begin{aligned}
    \widetilde{\operatorname{Wink}}_{\alpha}
    (P;r_j,\tau_q,\tau_p)
    &:=
    \widetilde Q_{1-\alpha/2}(P;\tau_q)
    -
    \widetilde Q_{\alpha/2}(P;\tau_q)
    \\
    &\quad+
    \frac{2}{\alpha}
    \left[
        \operatorname{softplus}_{\tau_p}
        \bigl(\widetilde Q_{\alpha/2}(P;\tau_q)-r_j\bigr)
        +
        \operatorname{softplus}_{\tau_p}
        \bigl(r_j-\widetilde Q_{1-\alpha/2}(P;\tau_q)\bigr)
    \right].
\end{aligned}
\end{equation}
The point forecast cancels in residual space, so this is equivalent to applying the
interval score to
\[
    \left[
        \widehat y_j+\widetilde Q_{\alpha/2}(P;\tau_q),
        \widehat y_j+\widetilde Q_{1-\alpha/2}(P;\tau_q)
    \right].
\]
As \(\tau_q,\tau_p\downarrow0\), Eq.~\eqref{eq:smooth_winkler_loss} recovers the hard
weighted quantiles and the hard Winkler score.

Training averages the smooth score over a small grid
\(\mathcal G_\alpha\) centered at the target miscoverage
\(\alpha_{\mathrm{tar}}\):
\begin{equation}
\label{eq:alpha_grid_loss_appendix}
    \mathcal W(P;r_j)
    =
    \frac{1}{|\mathcal G_\alpha|}
    \sum_{\alpha\in\mathcal G_\alpha}
    \widetilde{\operatorname{Wink}}_{\alpha}
    (P;r_j,\tau_q,\tau_p).
\end{equation}
In experiments, \(\tau_q\) follows a cyclic schedule: high-temperature phases provide
broad gradients for exploration of residual weights, and low-temperature phases refine
weights near the target CDF thresholds.

Let \(\mathfrak D\) denote the collection of training datasets. During leave-one-out
training, the current query is the held-out calibration context \(a_j\). Therefore, for
expert \(m\), the hypernetwork outputs the query- and dataset-conditioned affine map
\begin{equation}
\label{eq:appendix_query_conditioned_map}
    (A_{m,j,d},b_{m,j,d})
    =
    H_{\theta_m}^{(m)}(a_j,\eta_d).
\end{equation}
Here \(\eta_d=\{d,\mu_d,\sigma_d,\log(N_d)\}\) is the dataset descriptor defined in
Eq.~\eqref{eq:eta_descriptor_method}. The map
\((A_{m,j,d},b_{m,j,d})\) is applied to both the query \(a_j\) and each candidate
calibration context \(a_i\), \(i\neq j\), before computing cosine similarities. The
leave-one-out support for query \(j\in\mathcal I_{d,0}\) is
\[
    P_{m,j}^{(-j)}
    =
    \{(r_i,v_{m,i,j}^{(-j)}):i\in S_{m,j}^{(-j)}\},
\]
where \(S_{m,j}^{(-j)}\) and \(v_{m,i,j}^{(-j)}\) are computed by
Eqs.~\eqref{eq:expert_retrieval_method}--\eqref{eq:expert_weights_method}, with \(j\)
excluded and with the query-conditioned map
\((A_{m,j,d},b_{m,j,d})\).

Each student expert is anchored to a prefit teacher map for the same leave-one-out query,
\begin{equation}
\label{eq:teacher_anchor_map_appendix}
    t_{d,j}^{(m)}(a)
    =
    B_{m,d,j}a+c_{m,d,j}.
\end{equation}
The expert objective is
\begin{equation}
\label{eq:encoder_loss_appendix}
\begin{aligned}
    \mathcal L_{\mathrm{enc}}^{(m)}(\theta_m)
    &=
    \frac{1}{|\mathfrak D|}
    \sum_{d\in\mathfrak D}
    \frac{1}{|\mathcal I_{d,0}|}
    \sum_{j\in\mathcal I_{d,0}}
    \mathcal W(P_{m,j}^{(-j)};r_j)
    \\
    &\quad+
    \lambda_{\mathrm{anch}}
    \frac{1}{|\mathfrak D|}
    \sum_{d\in\mathfrak D}
    \frac{1}{|\mathcal I_{d,0}|}
    \sum_{j\in\mathcal I_{d,0}}
    \left(
        \|A_{m,j,d}-B_{m,d,j}\|_F^2
        +
        \|b_{m,j,d}-c_{m,d,j}\|_2^2
    \right).
\end{aligned}
\end{equation}
The first term trains interval efficiency under leave-one-out retrieval; the second term
regularizes the query-conditioned student geometry toward the corresponding teacher map.
If the teacher maps are dataset-level rather than query-level in an implementation, then
\(B_{m,d,j}\) and \(c_{m,d,j}\) reduce to \(B_{m,d}\) and \(c_{m,d}\).

After the experts are frozen, the gate is trained on leave-one-out mixed-CDF intervals.
For query \(j\), the gate uses the same query and descriptor,
\[
    (\pi_{1,j},\ldots,\pi_{M,j})
    =
    \operatorname{softmax}(g_\omega(a_j,\eta_d)).
\]
The RareCP leave-one-out residual weights are
\[
    w_{i,j}^{(-j)}
    =
    \sum_{m=1}^{M}
    \pi_{m,j}v_{m,i,j}^{(-j)}
    \mathbf 1\{i\in S_{m,j}^{(-j)}\}.
\]
If a calibration index appears in multiple expert supports, its weights are summed. The
mixed support is therefore
\[
    P_{\mathrm{RareCP},j}^{(-j)}
    =
    \{(r_i,w_{i,j}^{(-j)}):
    i\in\cup_{m=1}^{M}S_{m,j}^{(-j)}\}.
\]
Let \(H(\pi_j)=-\sum_{m=1}^{M}\pi_{m,j}\log\pi_{m,j}\). The gate objective is
\begin{equation}
\label{eq:gate_loss_appendix}
\begin{aligned}
    \mathcal L_{\mathrm{gate}}(\omega)
    &=
    \frac{1}{|\mathfrak D|}
    \sum_{d\in\mathfrak D}
    \frac{1}{|\mathcal I_{d,0}|}
    \sum_{j\in\mathcal I_{d,0}}
    \mathcal W(P_{\mathrm{RareCP},j}^{(-j)};r_j)
    \\
    &\quad-
    \lambda_{\mathrm{ent}}
    \frac{1}{|\mathfrak D|}
    \sum_{d\in\mathfrak D}
    \frac{1}{|\mathcal I_{d,0}|}
    \sum_{j\in\mathcal I_{d,0}}
    H(\pi_j).
\end{aligned}
\end{equation}
The entropy term encourages use of multiple experts during gate training and discourages
early collapse to a single residual regime. The outer average gives each dataset equal
weight before averaging over its calibration examples.

\clearpage

\section{Further ablations}
\begin{table}[h]
\centering
\caption{Normalized Winkler score on Tirex bench10. Scores are normalized by the test-set standard deviation per dataset and then averaged across datasets and seeds. Relative improvement is the percentage reduction in normalized Winkler versus Uniform. Lower nWink is better.}
\label{tab:norm-winkler-architecture}
\begin{tabular}{lrrr}
\toprule
Architecture & Params & nWink $\downarrow$ & vs. Uniform $\uparrow$ \\
\midrule
Uniform & $--$ & $0.7195$ & $--$ \\
MLP-2 & $28{,}416$ & $0.6374 \pm 0.0025$ & $11.4\%$ \\
MLP-3 & $58{,}496$ & $0.6325 \pm 0.0020$ & $12.1\%$ \\
MLP-4 & $149{,}376$ & $0.6421 \pm 0.0038$ & $10.8\%$ \\
MLP-5 & $3{,}305{,}276$ & $0.6338 \pm 0.0034$ & $11.9\%$ \\
Hypernetwork & $1{,}226{,}720$ & \textbf{0.6221 $\pm$ 0.0014} & $13.5\%$ \\
\bottomrule
\end{tabular}
\end{table}
\begin{table}[h]
\centering
\caption{Expert-count ablation for RareCP on the bench10. We vary the number of teacher/student experts while keeping the selected hyperparameter configuration fixed. Runtime is mean wall-clock seconds per seed.}
\label{tab:expert-count-normalized-uniform}
\small
\begin{tabular}{rcccc}
\toprule
Experts & Seeds & nWink & Uniform impr. (\%) & Avg. runtime (s) \\
\midrule
1 & 10 & 0.6195 $\pm$ 0.0024 & 13.91 $\pm$ 0.33 & 112.2 \\
3 & 10 & 0.6105 $\pm$ 0.0014 & 15.16 $\pm$ 0.20 & 182.4 \\
10 & 3 & 0.6078 $\pm$ 0.0007 & 15.53 $\pm$ 0.10 & 293.7 \\
30 & 3 & \textbf{0.6065 $\pm$ 0.0025} & \textbf{15.70 $\pm$ 0.35} & 842.5 \\
\bottomrule
\end{tabular}
\end{table}

\begin{table}[H]
\centering
\caption{Hyperparameter search spaces for the architecture ablation on Tirex bench10. We use $\mathcal{A}=\{0.20\}\cup\{0.20\pm0.02m\}:m=1,\ldots,5\}$ for the searched alpha sets.}
\label{tab:architecture-hpo-search-spaces}
\small
\begin{tabular}{llp{0.48\linewidth}}
\toprule
Group & Parameter & Search space \\
\midrule
Shared & Activation & $\{\mathrm{ReLU},\tanh,\mathrm{Softsign},\mathrm{LipSwish}\}$ \\
Shared & Top-$k$ & $\{8,16,24,32,48,64,96\}$ \\
Shared & Batch size & $\{128,256,384,512,768,1024,2048,4096\}$ \\
Shared & Training epochs & $\{50,100,200,300,500\}$ \\
Shared & Hopfield inverse temperature $\beta$ & $\log\mathcal{U}(0.5,24)$ \\
Shared & Learning rate & $\log\mathcal{U}(2{\times}10^{-5},5{\times}10^{-3})$ \\
Shared & Start temperature $\tau_{\mathrm{start}}$ & $\log\mathcal{U}(0.005,0.20)$ \\
Shared & End temperature $\tau_{\mathrm{end}}$ & $\log\mathcal{U}(2{\times}10^{-5},5{\times}10^{-3})$ \\
Shared & Penalty temperature $\tau_{\mathrm{pen}}$ & $\log\mathcal{U}(5{\times}10^{-5},3{\times}10^{-3})$ \\
Shared & Alpha set & $\mathcal{A}$ \\
\midrule
Hypernetwork & Latent dimension $d$ & $\{16,24,32,48,64,96,128\}$ \\
Hypernetwork & Hidden dimension & $\{32,64,96,128,192\}$ \\
Hypernetwork & Hidden layers & $\{1,2,3,4\}$ \\
\midrule
MLP-$N$ & Depth & $N\in\{2,3,4,5\}$ with hidden layers $=N-1$ \\
MLP-$N$ & Latent dimension $d$ & $\{16,24,32,48,64,96,128,192,884\}$ \\
MLP-$N$ & Hidden dimension & $d$ \\
\bottomrule
\end{tabular}
\end{table}

\clearpage

\section{Theoretical background}

\subsection{Motivation: top \texorpdfstring{\(K\)}{K} as proposal efficiency control}
\label{subsec:topk_importance_sampling_motivation}

We use hard top \(K\) retrieval to concentrate the retrieval proposal on locally relevant
calibration points. This concentration is useful as a tradeoff. It can reduce mass on
irrelevant points when the learned scores rank local points well, but it can also remove
target relevant points if the support is truncated too aggressively.

Fix a normalized query state \(u_\star\), and let
\[
    \mathcal I_n:=\{1,\ldots,n\}
\]
be the calibration index set. Let \(\mathcal I_\star\subseteq\mathcal I_n\) denote the
calibration points from the same local error regime as \(u_\star\), and write
\[
    \mathcal I_{\mathrm{off}}
    :=
    \mathcal I_n\setminus \mathcal I_\star
\]
for the off regime points. Let \(s_\theta(u_\star,u_i)\) be a learned retrieval score.

Dense softmax assigns
\[
    q_{\mathrm{den},i}(\theta)
    =
    \frac{
        \exp(s_\theta(u_\star,u_i)/T_w)
    }{
        \sum_{j\in\mathcal I_n}\exp(s_\theta(u_\star,u_j)/T_w)
    },
    \qquad T_w>0.
\]
Hence, whenever \(\mathcal I_{\mathrm{off}}\neq\varnothing\),
\[
    \delta_{\mathrm{den}}(\theta)
    :=
    \sum_{i\in\mathcal I_{\mathrm{off}}}
    q_{\mathrm{den},i}(\theta)
    >
    0.
\]
Thus dense softmax always spends some retrieval mass on off regime calibration points.

By contrast, hard top \(K\) restricts the proposal to
\[
    S_K(\theta)
    :=
    \operatorname{TopK}_{i\in\mathcal I_n}(s_\theta(u_\star,u_i);K)
\]
and assigns
\[
    q_{\mathrm{top},i}(\theta)
    =
    \frac{
        \mathbf 1\{i\in S_K(\theta)\}
        \exp(s_\theta(u_\star,u_i)/T_w)
    }{
        \sum_{j\in S_K(\theta)}
        \exp(s_\theta(u_\star,u_j)/T_w)
    }.
\]
Thus \(q_{\mathrm{top},i}(\theta)=0\) for \(i\notin S_K(\theta)\). If the
selected support is mostly same regime, then
\[
    \delta_{\mathrm{top}}(\theta)
    :=
    \sum_{i\in\mathcal I_{\mathrm{off}}}
    q_{\mathrm{top},i}(\theta)
\]
can be small, and it is exactly zero when
\[
    S_K(\theta)\subseteq \mathcal I_\star.
\]

This is a proposal efficiency effect subject to a support constraint. Suppose the ideal
local target distribution over calibration indices is \(\nu_\star\) and satisfies
\[
    \operatorname{supp}(\nu_\star)\subseteq\mathcal I_\star.
\]
Consider a retrieval proposal \(q\) with the usual coverage condition
\[
    q_i>0
    \qquad\text{whenever}\qquad
    \nu_{\star,i}>0.
\]
Hard truncation must be interpreted under this condition. If top \(K\) removes a target
relevant point with positive \(\nu_{\star,i}\), then the proposal no longer covers the
target. For an importance sampling draw \(I\sim q\), the weight is
\[
    \rho_I:=\frac{\nu_{\star,I}}{q_I}.
\]
Off regime draws have \(\nu_{\star,I}=0\), so they carry no target mass. If
\[
    \delta_q:=q(\mathcal I_{\mathrm{off}}),
\]
then the probability of drawing a same regime point is \(1-\delta_q\).

Equivalently, write the proposal conditional on the same regime set as
\[
    \widetilde q_i
    :=
    \frac{q_i}{1-\delta_q},
    \qquad i\in\mathcal I_\star.
\]
Then
\[
\begin{aligned}
    \mathbb E_{I\sim q}[\rho_I^2]
    &=
    \sum_{i\in\mathcal I_\star}
    q_i
    \left(
        \frac{\nu_{\star,i}}{q_i}
    \right)^2
    \\
    &=
    \frac{1}{1-\delta_q}
    \sum_{i\in\mathcal I_\star}
    \frac{\nu_{\star,i}^2}{\widetilde q_i}.
\end{aligned}
\]
Thus, for a fixed conditional same regime proposal \(\widetilde q\), off regime proposal
mass multiplies the importance weight second moment by
\[
    \frac{1}{1-\delta_q}.
\]
In this sense, retrieval mass placed on off regime points reduces proposal efficiency.
Hard top \(K\) is therefore a tradeoff. When the learned scores rank locally relevant
points highly and the retained support still covers the target relevant calibration
points, truncation can reduce off regime mass. When it drops target relevant points, uses
too small a support, or distorts the same regime proposal, it can reduce efficiency or
invalidate the importance sampler.

\subsection{Normalized learned top \texorpdfstring{\(K_n\)}{Kn} retrieval}
\label{subsec:normalized_learned_topk_short}

The preceding discussion treats top \(K\) as a finite sample proposal choice. We now state
the corresponding asymptotic calibration tradeoff. The condition \(K_n\to\infty\)
preserves effective sample size and local support, while \(K_n/n\to0\) preserves
localization.

We give a compact ResCP style argument for learned top \(K_n\) retrieval when the learned
states are normalized before retrieval. Let \(A_t\) denote the random context vector
corresponding to the observed method context \(a_t\). Let
\[
    \widetilde Z_t:=\phi(A_t)\in\mathbb R^d,
    \qquad
    U_t:=\frac{\widetilde Z_t}{\|\widetilde Z_t\|_2}\in\mathbb S^{d-1},
\]
where \(\|\widetilde Z_t\|_2>0\) almost surely. The representation map \(\phi\), the
forecasting rule \(\widehat y\), the retrieval score, and the weighting rule are trained
or selected before calibration, or using data disjoint from the calibration residual
labels, and are treated as fixed. Let \(\mathcal I_n=\{1,\ldots,n\}\) be the calibration
index set, and let \(K_n\) be a deterministic integer sequence with \(1\le K_n\le n\)
eventually.

Define the signed residual
\[
    E_t:=Y_t-\widehat y_t,
    \qquad
    \widehat y_t:=\widehat y(A_t).
\]
Let
\[
    \mathcal U:=\operatorname{supp}(U_t).
\]
For \(v\in\mathcal U\), let
\[
    F_v(r):=\mathbb P(E_\star\le r\mid U_\star=v)
\]
be a fixed regular conditional version of the conditional signed residual CDF. All
statements below are for query points \(u\in\mathcal U\) at which this version satisfies
the stated assumptions. For \(p\in(0,1)\), define
\[
    \xi_{u,p}^\star:=F_u^{-1}(p)
    =
    \inf\{r:F_u(r)\ge p\}.
\]

For the query state \(u\), retrieve the \(K_n\) calibration states with largest cosine
scores
\[
    s_i(u):=U_i^\top u,
    \qquad i\in\mathcal I_n.
\]
Ties are broken by a deterministic rule that depends only on the states and not on the
residual labels. Because \(U_i,u\in\mathbb S^{d-1}\),
\[
    \|U_i-u\|_2^2=2-2U_i^\top u.
\]
Thus top \(K_n\) retrieval by dot product is exactly nearest neighbor retrieval on the
unit sphere. Let \(S_{K_n}(u)\) be the retrieved index set, and define its radius by
\[
    h_n(u):=\max_{i\in S_{K_n}(u)}\|U_i-u\|_2.
\]

Let the retrieval weights satisfy
\[
    q_{i,n}(u)\ge0,
    \qquad
    \sum_{i\in\mathcal I_n} q_{i,n}(u)=1,
    \qquad
    q_{i,n}(u)=0
    \quad\text{for }i\notin S_{K_n}(u).
\]
The weighted empirical CDF and weighted quantile are
\[
    \widehat F_n(r\mid u)
    :=
    \sum_{i\in\mathcal I_n} q_{i,n}(u)\mathbf 1\{E_i\le r\},
\]
and
\[
    \widehat \xi_{n,p}(u)
    :=
    \inf\{r:\widehat F_n(r\mid u)\ge p\}.
\]

\begin{assumption}[Stochastic admissibility]
\label{ass:short_topk_stochastic}
The process \(\{(U_t,E_t)\}_{t\in\mathbb Z}\) is strictly stationary. For each fixed query
state \(u\), let \(\mathcal G_n(u)\) contain the query state, the normalized calibration
states, the retrieved set, and the retrieval weights:
\[
    \sigma\bigl(
        u,U_1,\dots,U_n,S_{K_n}(u),q_{1,n}(u),\dots,q_{n,n}(u)
    \bigr)
    \subseteq
    \mathcal G_n(u).
\]
The weights are \(\mathcal G_n(u)\) measurable and do not use the residual labels
directly.

Define
\[
    m_{i,n}(r;u)
    :=
    \mathbb E
    \left[
        \mathbf 1\{E_i\le r\}
        \,\middle|\,
        \mathcal G_n(u)
    \right],
\]
and the leakage term
\[
    D_n(u)
    :=
    \sup_{r\in\mathbb R}
    \left|
        \sum_{i\in\mathcal I_n}
        q_{i,n}(u)
        \{m_{i,n}(r;u)-F_{U_i}(r)\}
    \right|.
\]
Assume
\[
    D_n(u)=o_p(1).
\]
Also define the centered empirical process term
\[
    R_n(u)
    :=
    \sup_{r\in\mathbb R}
    \left|
        \sum_{i\in\mathcal I_n}
        q_{i,n}(u)
        \left[
            \mathbf 1\{E_i\le r\}
            -
            m_{i,n}(r;u)
        \right]
    \right|.
\]
Assume
\[
    \frac{
        R_n(u)
    }{
        \left(\sum_{i\in\mathcal I_n} q_{i,n}(u)^2\right)^{1/2}
    }
    =
    O_p(1).
\]
\end{assumption}

\begin{remark}[Why no separate strong mixing assumption is needed]
\label{rem:short_topk_no_mixing}
Strong mixing is not used in the proof of the theorem below. The proof only uses the high
level stochastic bounds
\[
    D_n(u)=o_p(1),
    \qquad
    R_n(u)
    =
    O_p
    \left(
        \left\{\sum_{i\in\mathcal I_n}q_{i,n}(u)^2\right\}^{1/2}
    \right),
\]
together with localization of the retrieved states. A mixing condition could be one route
to verify these bounds in a concrete time series model, but it is not needed as a separate
assumption once \(D_n(u)\) and \(R_n(u)\) are imposed directly. Moreover, nearest neighbor
retrieval may select temporally clustered or overlapping histories, so strong mixing alone
would not by itself verify the required effective sample size bound.
\end{remark}

\begin{assumption}[Top \texorpdfstring{\(K_n\)}{Kn} localization and weight size]
\label{ass:short_topk_localization}
The retrieved set and weights satisfy
\[
    K_n\to\infty,
    \qquad
    \frac{K_n}{n}\to0,
    \qquad
    h_n(u)\overset{p}{\longrightarrow}0,
\]
and
\[
    \sum_{i\in\mathcal I_n} q_{i,n}(u)^2
    =
    O_p(K_n^{-1}).
\]
\end{assumption}

\begin{assumption}[Local CDF continuity and quantile identification]
\label{ass:short_topk_continuity_quantile}
Define the local CDF modulus
\[
    \omega_u(h)
    :=
    \sup_{\substack{v\in\mathcal U\\ \|v-u\|_2\le h}}
    \sup_{r\in\mathbb R}
    |F_v(r)-F_u(r)|.
\]
Assume
\[
    \omega_u(h)\to0
    \qquad
    \text{as }h\downarrow0.
\]
For the chosen level \(p\in(0,1)\), assume \(\xi_{u,p}^\star\) is finite,
\(F_u\) is continuous at \(\xi_{u,p}^\star\),
\[
    F_u(\xi_{u,p}^\star)=p,
\]
and, for every \(\varepsilon>0\),
\[
    F_u(\xi_{u,p}^\star-\varepsilon)
    <
    p
    <
    F_u(\xi_{u,p}^\star+\varepsilon).
\]
\end{assumption}

\begin{remark}[Why the assumptions are plausible for time series]
\label{rem:short_topk_time_series}
The stochastic assumption is a high level way to express that the selected residual
indicators have effective sample size of order
\[
    \left(\sum_{i\in\mathcal I_n} q_{i,n}(u)^2\right)^{-1},
\]
which is of order \(K_n\) under Assumption~\ref{ass:short_topk_localization}. The leakage
term \(D_n(u)\) also matters. Even if retrieval does not explicitly use residual labels,
the retrieved states may contain indirect residual information through overlapping
histories. Temporal gaps, thinning, or blocking are standard ways to verify these high
level conditions in dependent data.

The localization condition is implied by ergodicity when \(u\in\mathcal U\) and
\(K_n/n\to0\). Indeed, for every \(\varepsilon>0\),
\[
    \mathbb P(\|U_t-u\|_2\le\varepsilon)>0,
\]
so the number of calibration states in this ball is asymptotically proportional to \(n\)
and eventually exceeds \(K_n\) with high probability. Hence \(h_n(u)\le\varepsilon\) with
high probability, and \(h_n(u)\to_p0\).

The weight size condition holds for uniform top \(K_n\) weights and for nonnegative score
weights whose unnormalized weights are uniformly comparable on \(S_{K_n}(u)\).
\end{remark}

\begin{theorem}[Top \texorpdfstring{\(K_n\)}{Kn} CDF consistency, quantile consistency, and coverage]
\label{thm:short_topk_coverage}
Suppose Assumptions~\ref{ass:short_topk_stochastic},
\ref{ass:short_topk_localization}, and
\ref{ass:short_topk_continuity_quantile} hold. Then
\[
    \Delta_n(u)
    :=
    \sup_{r\in\mathbb R}
    |\widehat F_n(r\mid u)-F_u(r)|
\]
satisfies
\[
    \Delta_n(u)
    \le
    R_n(u)+D_n(u)+\omega_u(h_n(u)).
\]
Moreover,
\[
    R_n(u)=O_p(K_n^{-1/2}),
\]
and therefore
\[
    \Delta_n(u)\overset{p}{\longrightarrow}0.
\]
Consequently,
\[
    \widehat \xi_{n,p}(u)
    \overset{p}{\longrightarrow}
    \xi_{u,p}^\star,
\]
and
\[
    F_u(\widehat \xi_{n,p}(u))
    \overset{p}{\longrightarrow}
    p.
\]

Assume in addition that the same quantile convergence holds under the joint conditional
law of the calibration data and target residual given \(U_\star=u\), that is,
\[
    \widehat \xi_{n,p}(u)
    \overset{p}{\longrightarrow}
    \xi_{u,p}^\star
    \quad
    \text{under }
    \mathcal L\bigl((U_i,E_i)_{i\in\mathcal I_n},E_\star\mid U_\star=u\bigr).
\]
Then
\[
    \mathbb P
    \left(
        E_\star\le \widehat \xi_{n,p}(u)
        \mid U_\star=u
    \right)
    \longrightarrow p.
\]
Taking \(p=1-\alpha\), the one sided upper prediction set
\[
    \widehat C_{n,+}^{\alpha}
    :=
    (-\infty,\,
        \widehat y_\star+\widehat \xi_{n,1-\alpha}(U_\star)
    ]
\]
satisfies, for every such \(u\),
\[
    \mathbb P
    \left(
        Y_\star
        \in
        \widehat C_{n,+}^{\alpha}
        \mid U_\star=u
    \right)
    \longrightarrow
    1-\alpha.
\]
\end{theorem}

\begin{proof}
Decompose
\[
    \widehat F_n(r\mid u)-F_u(r)
    =
    A_n(r\mid u)+B_n(r\mid u),
\]
where
\[
    A_n(r\mid u)
    :=
    \sum_{i\in\mathcal I_n}
    q_{i,n}(u)
    \{\mathbf 1\{E_i\le r\}-F_{U_i}(r)\},
\]
and
\[
    B_n(r\mid u)
    :=
    \sum_{i\in\mathcal I_n}
    q_{i,n}(u)
    \{F_{U_i}(r)-F_u(r)\}.
\]

For the stochastic term, write
\[
\begin{aligned}
    A_n(r\mid u)
    &=
    \sum_{i\in\mathcal I_n}
    q_{i,n}(u)
    \left[
        \mathbf 1\{E_i\le r\}
        -
        m_{i,n}(r;u)
    \right]  \\
    &\quad+
    \sum_{i\in\mathcal I_n}
    q_{i,n}(u)
    \left[
        m_{i,n}(r;u)
        -
        F_{U_i}(r)
    \right].
\end{aligned}
\]
Therefore, by Assumption~\ref{ass:short_topk_stochastic},
\[
    \sup_r |A_n(r\mid u)|
    \le
    R_n(u)+D_n(u).
\]

For the localization term, all positive weight indices belong to \(S_{K_n}(u)\), so
\(\|U_i-u\|_2\le h_n(u)\). Hence
\[
\begin{aligned}
    \sup_r |B_n(r\mid u)|
    &\le
    \sum_{i\in\mathcal I_n}
    q_{i,n}(u)
    \sup_r |F_{U_i}(r)-F_u(r)|  \\
    &\le
    \omega_u(h_n(u)).
\end{aligned}
\]
Thus
\[
    \Delta_n(u)
    \le
    R_n(u)+D_n(u)+\omega_u(h_n(u)).
\]
By Assumptions~\ref{ass:short_topk_stochastic} and
\ref{ass:short_topk_localization},
\[
    R_n(u)
    =
    O_p
    \left(
        \left\{\sum_{i\in\mathcal I_n}q_{i,n}(u)^2\right\}^{1/2}
    \right)
    =
    O_p(K_n^{-1/2}).
\]
Since \(K_n\to\infty\), \(D_n(u)=o_p(1)\), \(h_n(u)\to_p0\), and
\(\omega_u(h)\to0\) as \(h\downarrow0\), we get
\[
    \Delta_n(u)\to_p0.
\]

It remains to transfer CDF consistency to quantile consistency. Fix \(\varepsilon>0\),
and set
\[
    \eta_u(\varepsilon)
    :=
    \min
    \left\{
        p-F_u(\xi_{u,p}^\star-\varepsilon),
        \,
        F_u(\xi_{u,p}^\star+\varepsilon)-p
    \right\}.
\]
By Assumption~\ref{ass:short_topk_continuity_quantile}, \(\eta_u(\varepsilon)>0\). On
the event \(\Delta_n(u)<\eta_u(\varepsilon)\),
\[
    \widehat F_n(\xi_{u,p}^\star-\varepsilon\mid u)<p,
    \qquad
    \widehat F_n(\xi_{u,p}^\star+\varepsilon\mid u)>p.
\]
By monotonicity of \(\widehat F_n(\cdot\mid u)\),
\[
    \xi_{u,p}^\star-\varepsilon
    <
    \widehat \xi_{n,p}(u)
    \le
    \xi_{u,p}^\star+\varepsilon.
\]
Thus
\[
    \widehat \xi_{n,p}(u)\to_p \xi_{u,p}^\star.
\]
Continuity of \(F_u\) at \(\xi_{u,p}^\star\) gives
\[
    F_u(\widehat \xi_{n,p}(u))\to_p F_u(\xi_{u,p}^\star)=p.
\]

Now suppose the same quantile convergence holds under the joint conditional law given
\(U_\star=u\). Let
\[
    \Xi_n:=\widehat \xi_{n,p}(u),
    \qquad
    \xi:=\xi_{u,p}^\star.
\]
For every \(\varepsilon>0\),
\[
\begin{aligned}
    \mathbb P(E_\star\le \Xi_n\mid U_\star=u)
    &\ge
    F_u(\xi-\varepsilon)
    -
    \mathbb P(\Xi_n<\xi-\varepsilon\mid U_\star=u),
\end{aligned}
\]
and
\[
\begin{aligned}
    \mathbb P(E_\star\le \Xi_n\mid U_\star=u)
    &\le
    F_u(\xi+\varepsilon)
    +
    \mathbb P(\Xi_n>\xi+\varepsilon\mid U_\star=u).
\end{aligned}
\]
The conditional convergence \(\Xi_n\to_p \xi\) and the continuity of \(F_u\) at \(\xi\)
imply
\[
    \mathbb P
    \left(
        E_\star\le \widehat \xi_{n,p}(u)
        \mid U_\star=u
    \right)
    \to p.
\]
For \(p=1-\alpha\),
\[
    E_\star\le \widehat \xi_{n,1-\alpha}(U_\star)
    \quad
    \Longleftrightarrow
    \quad
    Y_\star
    \le
    \widehat y_\star+\widehat \xi_{n,1-\alpha}(U_\star),
\]
which proves the prediction set statement.
\end{proof}

\begin{remark}[Finite mixtures of top \texorpdfstring{\(K_n\)}{Kn} retrievers]
\label{rem:finite_mixture_topk}
The same argument applies to a finite mixture once the encoders and the gate are fixed
before calibration. Suppose expert \(m\) produces a weighted CDF
\(\widehat F_{m,n}(\cdot\mid u)\) and satisfies
\[
    \Delta_{m,n}(u)
    :=
    \sup_r
    |\widehat F_{m,n}(r\mid u)-F_u(r)|.
\]
Let \(\pi_m(u)\ge0\) and \(\sum_{m=1}^{M}\pi_m(u)=1\). Then the mixed CDF satisfies
\[
\begin{aligned}
    \sup_r
    \left|
        \sum_{m=1}^{M}\pi_m(u)\widehat F_{m,n}(r\mid u)-F_u(r)
    \right|
    &=
    \sup_r
    \left|
        \sum_{m=1}^{M}\pi_m(u)
        \bigl(\widehat F_{m,n}(r\mid u)-F_u(r)\bigr)
    \right|
    \\
    &\le
    \sum_{m=1}^{M}\pi_m(u)\Delta_{m,n}(u)
    \le
    \max_{m\le M}\Delta_{m,n}(u).
\end{aligned}
\]
Thus a finite gate mixture preserves the expert wise asymptotic order.
\end{remark}

\begin{remark}[Top \texorpdfstring{\(K_n\)}{Kn} tradeoff]
The theorem gives the error decomposition
\[
    \Delta_n(u)
    \le
    R_n(u)
    +
    D_n(u)
    +
    \omega_u(h_n(u)).
\]
The first term is stochastic error. The middle term is leakage from conditioning on the
retrieved states. The last term is localization bias.

For uniform top \(K_n\) weights,
\[
    q_{i,n}(u)
    =
    \frac{1}{K_n}\mathbf 1\{i\in S_{K_n}(u)\},
\]
we have
\[
    \sum_{i\in\mathcal I_n}q_{i,n}(u)^2=\frac1{K_n}.
\]
For nonnegative normalized score weights on the retrieved set, for example
\[
    q_{i,n}(u)
    =
    \frac{
        \kappa(U_i^\top u)\mathbf 1\{i\in S_{K_n}(u)\}
    }{
        \sum_{j\in S_{K_n}(u)} \kappa(U_j^\top u)
    },
\]
where \(\kappa\ge0\) is continuous and \(\kappa(1)>0\), localization implies
\[
    U_i^\top u
    =
    1-\frac12\|U_i-u\|_2^2
    =
    1+o_p(1)
    \qquad
    \text{uniformly over }i\in S_{K_n}(u).
\]
Thus these weights are asymptotically uniform and have the same first order stochastic
rate.

If, near \(u\), the normalized learned state distribution satisfies a lower local mass
bound
\[
    \mathbb P(\|U_t-u\|_2\le h)\ge c h^{d_{\mathrm{loc}}}
\]
for small \(h\), and the corresponding ball counts concentrate, then \(K_n\to\infty\) and
\(K_n/n\to0\) imply
\[
    h_n(u)
    =
    O_p
    \left(
        \left(\frac{K_n}{n}\right)^{1/d_{\mathrm{loc}}}
    \right).
\]
If the CDFs are locally H\"older,
\[
    \omega_u(h)\le Lh^\beta,
\]
then
\[
    \Delta_n(u)
    =
    O_p
    \left(
        K_n^{-1/2}
        +
        \left(\frac{K_n}{n}\right)^{\beta/d_{\mathrm{loc}}}
    \right)
    +
    D_n(u).
\]
\end{remark}

\begin{remark}[Signed residuals]
Using signed residuals does not change the proof. The argument only uses the indicators
\(\mathbf 1\{E_i\le r\}\), weighted CDF convergence, monotonicity, and continuity of the
target CDF at the relevant quantile.

With signed residuals and \(p=1-\alpha\), the resulting prediction set is a one sided
upper set. A central two sided interval can be obtained by estimating both signed residual
quantiles:
\[
    \widehat C_n^\alpha
    :=
    \left[
        \widehat y_\star+\widehat \xi_{n,\alpha/2}(U_\star),
        \widehat y_\star+\widehat \xi_{n,1-\alpha/2}(U_\star)
    \right].
\]
Under the same assumptions applied at both levels, with joint conditional convergence of
the two estimated quantiles and continuity at both target quantiles, this interval has
asymptotic conditional coverage \(1-\alpha\) given \(U_\star=u\).
\end{remark}

\begin{remark}[Conditioning]
The guarantee is conditional on the normalized retrieval state:
\[
    \mathbb P
    \left(
        Y_\star\in \widehat C_n^\alpha
        \mid U_\star=u
    \right)
    \to
    1-\alpha.
\]
This conditioning is understood through the regular conditional distribution of
\(Y_\star\) given \(U_\star=u\). It is not automatically conditional on the unnormalized
representation \(\widetilde Z_\star\) or on the full context \(A_\star\). To lift the
guarantee, one needs an additional sufficiency condition, for example
\[
    \mathcal L(E_\star\mid A_\star)
    =
    \mathcal L(E_\star\mid U_\star).
\]
\end{remark}

\clearpage

\section{Experiments and evaluation}
\label{apx:experiments}

This section provides additional experimental details and evaluation results that complement the main paper.

\subsection{Benchmark statistics}
\label{apx:bench}

\begin{table}[h]
  \caption{Per-dataset benchmark statistics for all configurations used in the paper. $n$ is the chronologically ordered series length fed to the pipeline, capped at 30k / 100k samples or full length. Each series is split 60/15/25 into train/calibration/test, with mean, std, min, max, and Interquantile Range (IQR) reported on the test split. Sections cover the headline Bench10 split (30k cap), the twelve datasets extending Bench10 to Bench22, and the scaled Bench10 variants at 100k and full-series.}
  \label{tab:dataset_stats}
  \centering
  \small
  \begin{tabular}{lrrrrrr}
    \toprule
    \textbf{Dataset} & \multicolumn{1}{c}{$n$} & \multicolumn{1}{c}{mean} & \multicolumn{1}{c}{std} & \multicolumn{1}{c}{min} & \multicolumn{1}{c}{max} & \multicolumn{1}{c}{IQR} \\
    \midrule
    \multicolumn{7}{c}{\textit{Bench10 (default cap, 30k)}} \\
    \addlinespace[2pt]
    electricity\_H & 30,000 & 1,390.9 & 1,398.2 & 152.81 & 7,000 & 916.57 \\
    electricity\_15T & 30,000 & 335.37 & 333.75 & 14.81 & 1,674.1 & 235.09 \\
    solar\_H & 30,000 & 23.17 & 41.99 & 0.00 & 223.50 & 28.45 \\
    LOOP\_SEATTLE\_5T & 30,000 & 62.45 & 4.357 & 30.46 & 72.19 & 4.467 \\
    SZ\_TAXI\_15T & 30,000 & 10.94 & 13.21 & 0.00 & 61.73 & 15.11 \\
    M\_DENSE\_H & 28,770 & 863.80 & 849.91 & 0.00 & 4,725 & 1,260 \\
    temperature\_rain & 30,000 & 11.31 & 18.00 & 2.10 & 862.60 & 14.20 \\
    m4\_hourly & 19,458 & 95.45 & 138.39 & 11.00 & 1,744.0 & 85.10 \\
    m4\_daily & 30,000 & 8,640.8 & 6,250.2 & 256.10 & 24,361.3 & 12,347.7 \\
    m4\_weekly & 3,024 & 5,116.6 & 5,327.1 & 317.84 & 28,522.0 & 5,945.2 \\
    \midrule
    \multicolumn{7}{c}{\textit{Bench22 extension (additional 12 datasets beyond Bench10)}} \\
    \addlinespace[2pt]
    solar\_10T & 30,000 & 2.932 & 6.016 & 0.00 & 35.35 & 2.600 \\
    LOOP\_SEATTLE\_H & 30,000 & 61.31 & 4.964 & 16.63 & 70.73 & 4.083 \\
    m4\_monthly & 25,959 & 7,135.9 & 6,537.2 & 240.00 & 52,410.6 & 8,603 \\
    electricity\_D & 23,542 & 33,448.7 & 34,480.9 & 0.00 & 198,887.3 & 34,982.8 \\
    SZ\_TAXI\_H & 14,820 & 11.07 & 11.54 & 0.00 & 57.01 & 15.56 \\
    m4\_quarterly & 308 & 9,247.4 & 4,353.4 & 150.10 & 19,220 & 489.63 \\
    saugeenday\_D & 599 & 27.17 & 26.05 & 7.6 & 187 & 20.20 \\
    M\_DENSE\_D & 2,670 & 895.11 & 559.24 & 0.00 & 1,951 & 1,166.3 \\
    us\_births\_D & 599 & 10,924.3 & 1,306.4 & 7,835 & 12,851 & 2,154 \\
    saugeenday\_W & 159 & 40.72 & 47.45 & 7.97 & 223.26 & 31.51 \\
    us\_births\_W & 111 & 77,022.2 & 3,195.4 & 69,577 & 82,962 & 5,014 \\
    saugeenday\_M & 83 & 32.49 & 31.35 & 9.577 & 116.75 & 23.88 \\
    \midrule
    \multicolumn{7}{c}{\textit{Bench10 (100k cap)}} \\
    \addlinespace[2pt]
    electricity\_H & 100,000 & 1,362.3 & 1,586.6 & 0.00 & 11,617.8 & 1,274.8 \\
    electricity\_15T & 100,000 & 333.10 & 405.66 & 0.00 & 2,938 & 313.31 \\
    solar\_H & 100,000 & 25.14 & 46.12 & 0.00 & 272.15 & 32.15 \\
    LOOP\_SEATTLE\_5T & 100,000 & 61.23 & 7.995 & 1.39 & 76.64 & 5.242 \\
    SZ\_TAXI\_15T & 52,260 & 11.04 & 11.94 & 0.00 & 61.73 & 15.73 \\
    M\_DENSE\_H & 28,770 & 863.80 & 849.91 & 0.00 & 4,725 & 1,260 \\
    temperature\_rain & 100,000 & 12.60 & 13.70 & -2.10 & 862.60 & 16.43 \\
    m4\_hourly & 19,458 & 95.45 & 138.39 & 11.00 & 1,744 & 85.10 \\
    m4\_daily & 38,740 & 8,167.8 & 6,197.9 & 256.10 & 24,361.3 & 12,080.4 \\
    m4\_weekly & 3,024 & 5,116.6 & 5,327.1 & 317.84 & 28,522 & 5,945.2 \\
    \midrule
    \multicolumn{7}{c}{\textit{Bench10 (full-series)}} \\
    \addlinespace[2pt]
    electricity\_H & 151,522 & 1,264.1 & 1,423.5 & 0.00 & 11,617.8 & 1,258.8 \\
    electricity\_15T & 151,522 & 312.04 & 363.12 & 0.00 & 2,938 & 298.84 \\
    solar\_H & 124,807 & 26.28 & 48.08 & 0.00 & 272.15 & 34.05 \\
    LOOP\_SEATTLE\_5T & 309,757 & 59.47 & 8.492 & 1.393 & 76.93 & 7.213 \\
    SZ\_TAXI\_15T & 52,260 & 11.04 & 11.94 & 0.00 & 61.73 & 15.73 \\
    M\_DENSE\_H & 28,770 & 863.80 & 849.91 & 0.00 & 4,725 & 1,260 \\
    temperature\_rain & 101,193 & 12.62 & 13.67 & -2.1 & 862.60 & 16.48 \\
    m4\_hourly & 19,458 & 95.45 & 138.39 & 11.00 & 1,744 & 85.10 \\
    m4\_daily & 38,740 & 8,167.8 & 6,197.9 & 256.10 & 24,361.3 & 12,080.4 \\
    m4\_weekly & 3,024 & 5,116.6 & 5,327.1 & 317.84 & 28,522 & 5,945.2 \\
    \bottomrule
  \end{tabular}
\end{table}

\clearpage

\subsection{Full experiment results}
\label{apx:results}

\subsubsection{Full conformal prediction baseline results - Bench10}

\begin{table}[h]
  \caption{CP methods at 80\% target coverage on Gift-Eval Bench10 (10 datasets, 4 backbones: TiRex, TimesFM-2.5, TabPFN-TS, ARIMA). Cells show nWink (mean Winkler $/$ std$(y)$) and nW (mean width $/$ std$(y)$) $/$ Cov, where std$(y)$ is the standard deviation of the target on the eval split. ResCP, HopCPT and RareCP report mean$\,\pm\,$std over ten seeds, other methods are deterministic. \textbf{Bold} marks the smallest displayed nWink per row among methods with $\mathrm{Cov} \ge 0.70$. Each backbone block ends with an \textit{Average} row over its 10 datasets.}
  \label{tab:cp_methods_bench10}
  \centering
  \resizebox{\textwidth}{!}{%
  \begin{tabular}{lcccccccccccccccc}
    \multicolumn{17}{c}{{\large\bfseries\itshape Bench10}} \\[2pt]
    \toprule
    & \multicolumn{2}{c}{Uniform} & \multicolumn{2}{c}{ACI} & \multicolumn{2}{c}{DtACI} & \multicolumn{2}{c}{NexCP} & \multicolumn{2}{c}{KOWCPI} & \multicolumn{2}{c}{HopCPT} & \multicolumn{2}{c}{ResCP} & \multicolumn{2}{c}{RareCP\,\textit{(Ours)}} \\
    \cmidrule(lr){2-3} \cmidrule(lr){4-5} \cmidrule(lr){6-7} \cmidrule(lr){8-9} \cmidrule(lr){10-11} \cmidrule(lr){12-13} \cmidrule(lr){14-15} \cmidrule(lr){16-17}
    \textbf{Dataset} & nWink $\downarrow$ & nW $\downarrow$ / Cov $\uparrow$ & nWink $\downarrow$ & nW $\downarrow$ / Cov $\uparrow$ & nWink $\downarrow$ & nW $\downarrow$ / Cov $\uparrow$ & nWink $\downarrow$ & nW $\downarrow$ / Cov $\uparrow$ & nWink $\downarrow$ & nW $\downarrow$ / Cov $\uparrow$ & nWink $\downarrow$ & nW $\downarrow$ / Cov $\uparrow$ & nWink $\downarrow$ & nW $\downarrow$ / Cov $\uparrow$ & nWink $\downarrow$ & nW $\downarrow$ / Cov $\uparrow$ \\
    \midrule
    \multicolumn{17}{c}{\textit{TiRex backbone}} \\
    \addlinespace[2pt]
    electricity\_H & 0.26 & 0.15 / 0.80 & 0.23 & 0.13 / 0.80 & 0.27 & 0.17 / 0.78 & 0.22 & 0.12 / 0.80 & 0.25 & 0.11 / 0.78 & 0.24\,{\scriptsize $\pm$\,0.01} & 0.12 / 0.78\,{\scriptsize $\pm$\,0.01} & 0.26\,{\scriptsize $\pm$\,0.01} & 0.11 / 0.76\,{\scriptsize $\pm$\,0.06} & \textbf{0.20}\,{\scriptsize $\pm$\,0.00} & 0.12 / 0.80\,{\scriptsize $\pm$\,0.00} \\
    electricity\_15T & 0.33 & 0.19 / 0.81 & 0.30 & 0.16 / 0.80 & 0.33 & 0.20 / 0.79 & 0.29 & 0.16 / 0.80 & 0.33 & 0.14 / 0.70 & 0.31\,{\scriptsize $\pm$\,0.01} & 0.17 / 0.79\,{\scriptsize $\pm$\,0.00} & 0.31\,{\scriptsize $\pm$\,0.01} & 0.14 / 0.76\,{\scriptsize $\pm$\,0.06} & \textbf{0.25}\,{\scriptsize $\pm$\,0.00} & 0.15 / 0.80\,{\scriptsize $\pm$\,0.00} \\
    solar\_H & 0.74 & 0.23 / 0.79 & 0.74 & 0.25 / 0.80 & 0.72 & 0.27 / 0.79 & 0.74 & 0.25 / 0.80 & 0.75 & 0.15 / 0.64 & 0.54\,{\scriptsize $\pm$\,0.01} & 0.25 / 0.78\,{\scriptsize $\pm$\,0.01} & 0.52\,{\scriptsize $\pm$\,0.00} & 0.25 / 0.79\,{\scriptsize $\pm$\,0.00} & \textbf{0.46}\,{\scriptsize $\pm$\,0.00} & 0.27 / 0.80\,{\scriptsize $\pm$\,0.00} \\
    LOOP\_SEATTLE\_5T & 2.24 & 1.40 / 0.80 & 2.21 & 1.43 / 0.80 & 2.19 & 1.44 / 0.79 & 2.22 & 1.42 / 0.80 & 2.27 & 1.46 / 0.81 & 2.26\,{\scriptsize $\pm$\,0.00} & 1.41 / 0.80\,{\scriptsize $\pm$\,0.00} & 2.26\,{\scriptsize $\pm$\,0.01} & 1.50 / 0.82\,{\scriptsize $\pm$\,0.00} & \textbf{2.11}\,{\scriptsize $\pm$\,0.00} & 1.45 / 0.80\,{\scriptsize $\pm$\,0.00} \\
    SZ\_TAXI\_15T & 1.00 & 0.52 / 0.78 & \textbf{0.85} & 0.52 / 0.81 & 1.00 & 0.64 / 0.79 & \textbf{0.85} & 0.50 / 0.81 & 1.04 & 0.26 / 0.61 & 1.03\,{\scriptsize $\pm$\,0.00} & 0.53 / 0.77\,{\scriptsize $\pm$\,0.00} & 0.91\,{\scriptsize $\pm$\,0.01} & 0.50 / 0.81\,{\scriptsize $\pm$\,0.01} & 0.90\,{\scriptsize $\pm$\,0.00} & 0.60 / 0.80\,{\scriptsize $\pm$\,0.00} \\
    M\_DENSE\_H & 0.96 & 0.49 / 0.79 & 0.92 & 0.50 / 0.80 & 0.99 & 0.58 / 0.79 & 0.92 & 0.50 / 0.80 & 1.02 & 0.20 / 0.54 & 0.95\,{\scriptsize $\pm$\,0.03} & 0.57 / 0.83\,{\scriptsize $\pm$\,0.03} & 0.90\,{\scriptsize $\pm$\,0.01} & 0.50 / 0.79\,{\scriptsize $\pm$\,0.01} & \textbf{0.63}\,{\scriptsize $\pm$\,0.00} & 0.44 / 0.80\,{\scriptsize $\pm$\,0.00} \\
    temperature\_rain & 0.82 & 0.35 / 0.79 & 0.85 & 0.39 / 0.80 & 0.83 & 0.37 / 0.79 & 0.81 & 0.35 / 0.80 & 0.82 & 0.16 / 0.44 & 0.83\,{\scriptsize $\pm$\,0.00} & 0.36 / 0.80\,{\scriptsize $\pm$\,0.01} & \textbf{0.74}\,{\scriptsize $\pm$\,0.01} & 0.30 / 0.81\,{\scriptsize $\pm$\,0.00} & 0.83\,{\scriptsize $\pm$\,0.02} & 0.34 / 0.80\,{\scriptsize $\pm$\,0.00} \\
    m4\_hourly & 0.63 & 0.22 / 0.73 & 0.59 & 0.29 / 0.80 & 0.63 & 0.36 / 0.76 & 0.59 & 0.26 / 0.79 & 0.70 & 0.14 / 0.57 & 0.63\,{\scriptsize $\pm$\,0.00} & 0.24 / 0.75\,{\scriptsize $\pm$\,0.00} & 0.64\,{\scriptsize $\pm$\,0.02} & 0.19 / 0.67\,{\scriptsize $\pm$\,0.01} & \textbf{0.52}\,{\scriptsize $\pm$\,0.00} & 0.31 / 0.80\,{\scriptsize $\pm$\,0.00} \\
    m4\_daily & 0.06 & 0.03 / 0.81 & 0.06 & 0.03 / 0.80 & 0.06 & 0.04 / 0.79 & 0.06 & 0.03 / 0.80 & 0.06 & 0.02 / 0.66 & 0.06\,{\scriptsize $\pm$\,0.01} & 0.03 / 0.77\,{\scriptsize $\pm$\,0.08} & 0.05\,{\scriptsize $\pm$\,0.00} & 0.03 / 0.78\,{\scriptsize $\pm$\,0.03} & \textbf{0.04}\,{\scriptsize $\pm$\,0.00} & 0.03 / 0.80\,{\scriptsize $\pm$\,0.00} \\
    m4\_weekly & \textbf{0.14} & 0.06 / 0.78 & \textbf{0.14} & 0.06 / 0.79 & \textbf{0.14} & 0.06 / 0.77 & \textbf{0.14} & 0.06 / 0.79 & 0.15 & 0.06 / 0.74 & 0.16\,{\scriptsize $\pm$\,0.00} & 0.09 / 0.78\,{\scriptsize $\pm$\,0.01} & 0.15\,{\scriptsize $\pm$\,0.00} & 0.07 / 0.71\,{\scriptsize $\pm$\,0.02} & \textbf{0.14}\,{\scriptsize $\pm$\,0.00} & 0.07 / 0.78\,{\scriptsize $\pm$\,0.00} \\
    \addlinespace[2pt]
    \cmidrule(lr){1-17}
    \textit{Average} & 0.72 & 0.36 / 0.79 & 0.69 & 0.38 / 0.80 & 0.72 & 0.41 / 0.78 & 0.68 & 0.36 / 0.80 & 0.74 & 0.27 / 0.65 & 0.70 & 0.38 / 0.78 & 0.68 & 0.36 / 0.77 & \textbf{0.61} & 0.38 / 0.80 \\
    \midrule
    \multicolumn{17}{c}{\textit{TimesFM-2.5 backbone}} \\
    \addlinespace[2pt]
    electricity\_H & 0.34 & 0.20 / 0.80 & 0.28 & 0.17 / 0.80 & 0.35 & 0.22 / 0.79 & 0.28 & 0.16 / 0.80 & 0.31 & 0.09 / 0.63 & 0.31\,{\scriptsize $\pm$\,0.00} & 0.17 / 0.78\,{\scriptsize $\pm$\,0.00} & 0.33\,{\scriptsize $\pm$\,0.01} & 0.15 / 0.80\,{\scriptsize $\pm$\,0.01} & \textbf{0.21}\,{\scriptsize $\pm$\,0.00} & 0.13 / 0.80\,{\scriptsize $\pm$\,0.00} \\
    electricity\_15T & 0.37 & 0.20 / 0.81 & 0.34 & 0.17 / 0.80 & 0.37 & 0.21 / 0.79 & 0.33 & 0.17 / 0.80 & 0.35 & 0.10 / 0.65 & 0.36\,{\scriptsize $\pm$\,0.00} & 0.18 / 0.79\,{\scriptsize $\pm$\,0.00} & 0.35\,{\scriptsize $\pm$\,0.01} & 0.17 / 0.81\,{\scriptsize $\pm$\,0.02} & \textbf{0.25}\,{\scriptsize $\pm$\,0.00} & 0.16 / 0.80\,{\scriptsize $\pm$\,0.00} \\
    solar\_H & 1.03 & 0.36 / 0.79 & 1.02 & 0.39 / 0.80 & 0.98 & 0.42 / 0.79 & 1.02 & 0.38 / 0.80 & 0.79 & 0.28 / 0.75 & 1.04\,{\scriptsize $\pm$\,0.09} & 0.41 / 0.80\,{\scriptsize $\pm$\,0.02} & 0.76\,{\scriptsize $\pm$\,0.01} & 0.32 / 0.78\,{\scriptsize $\pm$\,0.01} & \textbf{0.55}\,{\scriptsize $\pm$\,0.00} & 0.36 / 0.80\,{\scriptsize $\pm$\,0.00} \\
    LOOP\_SEATTLE\_5T & 2.22 & 1.38 / 0.80 & 2.19 & 1.41 / 0.80 & 2.17 & 1.41 / 0.79 & 2.19 & 1.40 / 0.80 & 2.21 & 1.16 / 0.73 & 2.27\,{\scriptsize $\pm$\,0.14} & 1.38 / 0.78\,{\scriptsize $\pm$\,0.04} & 2.16\,{\scriptsize $\pm$\,0.01} & 1.40 / 0.79\,{\scriptsize $\pm$\,0.00} & \textbf{2.11}\,{\scriptsize $\pm$\,0.00} & 1.47 / 0.80\,{\scriptsize $\pm$\,0.00} \\
    SZ\_TAXI\_15T & 0.99 & 0.51 / 0.78 & 0.84 & 0.51 / 0.81 & 0.98 & 0.63 / 0.79 & \textbf{0.83} & 0.49 / 0.81 & 0.97 & 0.30 / 0.66 & 1.02\,{\scriptsize $\pm$\,0.00} & 0.51 / 0.76\,{\scriptsize $\pm$\,0.00} & 0.84\,{\scriptsize $\pm$\,0.00} & 0.46 / 0.80\,{\scriptsize $\pm$\,0.01} & 0.90\,{\scriptsize $\pm$\,0.00} & 0.61 / 0.80\,{\scriptsize $\pm$\,0.00} \\
    M\_DENSE\_H & 0.97 & 0.49 / 0.79 & 0.93 & 0.52 / 0.80 & 1.00 & 0.57 / 0.79 & 0.93 & 0.51 / 0.80 & 0.99 & 0.30 / 0.63 & 0.96\,{\scriptsize $\pm$\,0.00} & 0.51 / 0.80\,{\scriptsize $\pm$\,0.00} & 0.90\,{\scriptsize $\pm$\,0.00} & 0.46 / 0.76\,{\scriptsize $\pm$\,0.01} & \textbf{0.61}\,{\scriptsize $\pm$\,0.00} & 0.44 / 0.80\,{\scriptsize $\pm$\,0.00} \\
    temperature\_rain & 0.83 & 0.35 / 0.79 & 0.81 & 0.35 / 0.80 & 0.83 & 0.38 / 0.79 & 0.81 & 0.35 / 0.80 & 0.87 & 0.28 / 0.71 & 0.83\,{\scriptsize $\pm$\,0.00} & 0.36 / 0.79\,{\scriptsize $\pm$\,0.00} & \textbf{0.77}\,{\scriptsize $\pm$\,0.00} & 0.33 / 0.81\,{\scriptsize $\pm$\,0.00} & 0.82\,{\scriptsize $\pm$\,0.03} & 0.38 / 0.80\,{\scriptsize $\pm$\,0.00} \\
    m4\_hourly & 0.72 & 0.24 / 0.72 & 0.67 & 0.33 / 0.80 & 0.71 & 0.39 / 0.76 & 0.67 & 0.29 / 0.79 & 0.70 & 0.17 / 0.58 & 0.72\,{\scriptsize $\pm$\,0.00} & 0.27 / 0.73\,{\scriptsize $\pm$\,0.00} & 0.68\,{\scriptsize $\pm$\,0.03} & 0.21 / 0.67\,{\scriptsize $\pm$\,0.01} & \textbf{0.57}\,{\scriptsize $\pm$\,0.00} & 0.35 / 0.80\,{\scriptsize $\pm$\,0.00} \\
    m4\_daily & 0.06 & 0.03 / 0.81 & 0.06 & 0.03 / 0.80 & 0.06 & 0.04 / 0.79 & 0.06 & 0.03 / 0.80 & 0.06 & 0.02 / 0.67 & 0.07\,{\scriptsize $\pm$\,0.01} & 0.03 / 0.76\,{\scriptsize $\pm$\,0.08} & 0.06\,{\scriptsize $\pm$\,0.00} & 0.03 / 0.78\,{\scriptsize $\pm$\,0.01} & \textbf{0.04}\,{\scriptsize $\pm$\,0.00} & 0.03 / 0.80\,{\scriptsize $\pm$\,0.00} \\
    m4\_weekly & \textbf{0.15} & 0.07 / 0.79 & \textbf{0.15} & 0.07 / 0.79 & \textbf{0.15} & 0.06 / 0.78 & \textbf{0.15} & 0.07 / 0.80 & 0.16 & 0.04 / 0.67 & 0.16\,{\scriptsize $\pm$\,0.02} & 0.08 / 0.79\,{\scriptsize $\pm$\,0.02} & 0.16\,{\scriptsize $\pm$\,0.01} & 0.05 / 0.63\,{\scriptsize $\pm$\,0.07} & \textbf{0.15}\,{\scriptsize $\pm$\,0.00} & 0.09 / 0.78\,{\scriptsize $\pm$\,0.00} \\
    \addlinespace[2pt]
    \cmidrule(lr){1-17}
    \textit{Average} & 0.77 & 0.38 / 0.79 & 0.73 & 0.40 / 0.80 & 0.76 & 0.43 / 0.78 & 0.73 & 0.39 / 0.80 & 0.74 & 0.28 / 0.67 & 0.77 & 0.39 / 0.78 & 0.70 & 0.36 / 0.76 & \textbf{0.62} & 0.40 / 0.80 \\
    \midrule
    \multicolumn{17}{c}{\textit{TabPFN-TS backbone}} \\
    \addlinespace[2pt]
    electricity\_H & 0.32 & 0.17 / 0.80 & 0.27 & 0.15 / 0.80 & 0.32 & 0.20 / 0.79 & 0.27 & 0.14 / 0.80 & 0.28 & 0.10 / 0.68 & 0.30\,{\scriptsize $\pm$\,0.01} & 0.16 / 0.77\,{\scriptsize $\pm$\,0.00} & 0.30\,{\scriptsize $\pm$\,0.01} & 0.12 / 0.74\,{\scriptsize $\pm$\,0.04} & \textbf{0.23}\,{\scriptsize $\pm$\,0.00} & 0.14 / 0.79\,{\scriptsize $\pm$\,0.00} \\
    electricity\_15T & 0.35 & 0.19 / 0.81 & 0.32 & 0.17 / 0.80 & 0.35 & 0.20 / 0.79 & 0.32 & 0.16 / 0.80 & 0.33 & 0.12 / 0.69 & 0.34\,{\scriptsize $\pm$\,0.01} & 0.17 / 0.79\,{\scriptsize $\pm$\,0.00} & 0.34\,{\scriptsize $\pm$\,0.00} & 0.14 / 0.74\,{\scriptsize $\pm$\,0.04} & \textbf{0.25}\,{\scriptsize $\pm$\,0.00} & 0.16 / 0.80\,{\scriptsize $\pm$\,0.00} \\
    solar\_H & 1.17 & 0.35 / 0.80 & 1.16 & 0.39 / 0.80 & 1.12 & 0.39 / 0.79 & 1.16 & 0.39 / 0.80 & 0.68 & 0.25 / 0.74 & 1.12\,{\scriptsize $\pm$\,0.03} & 0.40 / 0.81\,{\scriptsize $\pm$\,0.05} & 0.66\,{\scriptsize $\pm$\,0.01} & 0.32 / 0.75\,{\scriptsize $\pm$\,0.02} & \textbf{0.51}\,{\scriptsize $\pm$\,0.00} & 0.32 / 0.80\,{\scriptsize $\pm$\,0.00} \\
    LOOP\_SEATTLE\_5T & 2.27 & 1.43 / 0.80 & 2.24 & 1.45 / 0.80 & 2.24 & 1.48 / 0.79 & 2.25 & 1.45 / 0.80 & 2.23 & 1.27 / 0.76 & 2.31\,{\scriptsize $\pm$\,0.01} & 1.46 / 0.80\,{\scriptsize $\pm$\,0.00} & 2.26\,{\scriptsize $\pm$\,0.01} & 1.51 / 0.80\,{\scriptsize $\pm$\,0.00} & \textbf{2.14}\,{\scriptsize $\pm$\,0.00} & 1.45 / 0.80\,{\scriptsize $\pm$\,0.00} \\
    SZ\_TAXI\_15T & 1.00 & 0.52 / 0.79 & 0.86 & 0.51 / 0.81 & 1.01 & 0.64 / 0.79 & \textbf{0.85} & 0.50 / 0.81 & 0.92 & 0.36 / 0.71 & 1.01\,{\scriptsize $\pm$\,0.01} & 0.52 / 0.77\,{\scriptsize $\pm$\,0.01} & 0.86\,{\scriptsize $\pm$\,0.00} & 0.46 / 0.74\,{\scriptsize $\pm$\,0.01} & 0.93\,{\scriptsize $\pm$\,0.00} & 0.63 / 0.81\,{\scriptsize $\pm$\,0.00} \\
    M\_DENSE\_H & 1.08 & 0.52 / 0.79 & 1.04 & 0.54 / 0.80 & 1.11 & 0.60 / 0.79 & 1.04 & 0.53 / 0.80 & 1.06 & 0.36 / 0.66 & 1.08\,{\scriptsize $\pm$\,0.00} & 0.54 / 0.80\,{\scriptsize $\pm$\,0.01} & 0.97\,{\scriptsize $\pm$\,0.00} & 0.53 / 0.75\,{\scriptsize $\pm$\,0.01} & \textbf{0.65}\,{\scriptsize $\pm$\,0.00} & 0.47 / 0.80\,{\scriptsize $\pm$\,0.00} \\
    temperature\_rain & 0.86 & 0.38 / 0.79 & 0.84 & 0.39 / 0.80 & 0.87 & 0.41 / 0.79 & 0.84 & 0.38 / 0.80 & 0.82 & 0.30 / 0.74 & 0.86\,{\scriptsize $\pm$\,0.00} & 0.38 / 0.79\,{\scriptsize $\pm$\,0.01} & 0.77\,{\scriptsize $\pm$\,0.00} & 0.33 / 0.81\,{\scriptsize $\pm$\,0.00} & \textbf{0.75}\,{\scriptsize $\pm$\,0.01} & 0.28 / 0.79\,{\scriptsize $\pm$\,0.00} \\
    m4\_hourly & 0.82 & 0.26 / 0.72 & 0.76 & 0.36 / 0.80 & 0.81 & 0.45 / 0.76 & 0.77 & 0.31 / 0.79 & 0.74 & 0.21 / 0.63 & 0.82\,{\scriptsize $\pm$\,0.00} & 0.29 / 0.74\,{\scriptsize $\pm$\,0.00} & 0.80\,{\scriptsize $\pm$\,0.04} & 0.22 / 0.65\,{\scriptsize $\pm$\,0.04} & \textbf{0.64}\,{\scriptsize $\pm$\,0.00} & 0.38 / 0.80\,{\scriptsize $\pm$\,0.00} \\
    m4\_daily & 0.06 & 0.03 / 0.80 & 0.06 & 0.03 / 0.80 & 0.07 & 0.04 / 0.79 & 0.06 & 0.03 / 0.80 & 0.06 & 0.02 / 0.71 & 0.06\,{\scriptsize $\pm$\,0.01} & 0.03 / 0.80\,{\scriptsize $\pm$\,0.04} & 0.06\,{\scriptsize $\pm$\,0.00} & 0.03 / 0.75\,{\scriptsize $\pm$\,0.04} & \textbf{0.04}\,{\scriptsize $\pm$\,0.00} & 0.03 / 0.80\,{\scriptsize $\pm$\,0.00} \\
    m4\_weekly & 0.16 & 0.06 / 0.79 & 0.16 & 0.06 / 0.79 & 0.16 & 0.06 / 0.76 & 0.16 & 0.06 / 0.80 & 0.16 & 0.05 / 0.68 & 0.17\,{\scriptsize $\pm$\,0.01} & 0.07 / 0.79\,{\scriptsize $\pm$\,0.02} & 0.20\,{\scriptsize $\pm$\,0.01} & 0.04 / 0.48\,{\scriptsize $\pm$\,0.11} & \textbf{0.15}\,{\scriptsize $\pm$\,0.00} & 0.07 / 0.76\,{\scriptsize $\pm$\,0.00} \\
    \addlinespace[2pt]
    \cmidrule(lr){1-17}
    \textit{Average} & 0.81 & 0.39 / 0.79 & 0.77 & 0.41 / 0.80 & 0.81 & 0.45 / 0.78 & 0.77 & 0.40 / 0.80 & 0.73 & 0.30 / 0.70 & 0.81 & 0.40 / 0.79 & 0.72 & 0.37 / 0.72 & \textbf{0.63} & 0.39 / 0.79 \\
    \midrule
    \multicolumn{17}{c}{\textit{ARIMA backbone}} \\
    \addlinespace[2pt]
    electricity\_H & 0.56 & 0.34 / 0.80 & 0.46 & 0.30 / 0.80 & 0.57 & 0.38 / 0.79 & 0.45 & 0.27 / 0.80 & 0.45 & 0.18 / 0.70 & 0.48\,{\scriptsize $\pm$\,0.00} & 0.28 / 0.78\,{\scriptsize $\pm$\,0.00} & 0.51\,{\scriptsize $\pm$\,0.02} & 0.19 / 0.70\,{\scriptsize $\pm$\,0.02} & \textbf{0.25}\,{\scriptsize $\pm$\,0.00} & 0.17 / 0.80\,{\scriptsize $\pm$\,0.00} \\
    electricity\_15T & 0.36 & 0.20 / 0.81 & 0.33 & 0.18 / 0.80 & 0.37 & 0.22 / 0.79 & 0.32 & 0.17 / 0.81 & 0.33 & 0.13 / 0.70 & 0.34\,{\scriptsize $\pm$\,0.01} & 0.17 / 0.79\,{\scriptsize $\pm$\,0.00} & 0.36\,{\scriptsize $\pm$\,0.01} & 0.14 / 0.74\,{\scriptsize $\pm$\,0.03} & \textbf{0.26}\,{\scriptsize $\pm$\,0.00} & 0.17 / 0.80\,{\scriptsize $\pm$\,0.00} \\
    solar\_H & 1.25 & 0.58 / 0.79 & 1.23 & 0.63 / 0.80 & 1.27 & 0.64 / 0.79 & 1.23 & 0.61 / 0.80 & 1.01 & 0.38 / 0.74 & 1.23\,{\scriptsize $\pm$\,0.02} & 0.63 / 0.79\,{\scriptsize $\pm$\,0.00} & 0.87\,{\scriptsize $\pm$\,0.01} & 0.45 / 0.79\,{\scriptsize $\pm$\,0.01} & \textbf{0.59}\,{\scriptsize $\pm$\,0.00} & 0.35 / 0.80\,{\scriptsize $\pm$\,0.00} \\
    LOOP\_SEATTLE\_5T & 2.29 & 1.46 / 0.80 & 2.25 & 1.48 / 0.80 & 2.25 & 1.48 / 0.79 & 2.27 & 1.48 / 0.80 & 2.23 & 1.29 / 0.75 & 2.34\,{\scriptsize $\pm$\,0.02} & 1.47 / 0.79\,{\scriptsize $\pm$\,0.00} & 2.30\,{\scriptsize $\pm$\,0.01} & 1.54 / 0.80\,{\scriptsize $\pm$\,0.00} & \textbf{2.16}\,{\scriptsize $\pm$\,0.00} & 1.47 / 0.80\,{\scriptsize $\pm$\,0.00} \\
    SZ\_TAXI\_15T & 1.02 & 0.53 / 0.78 & 0.87 & 0.53 / 0.80 & 1.00 & 0.65 / 0.79 & \textbf{0.86} & 0.52 / 0.81 & 0.94 & 0.38 / 0.72 & 1.02\,{\scriptsize $\pm$\,0.01} & 0.54 / 0.75\,{\scriptsize $\pm$\,0.00} & 0.88\,{\scriptsize $\pm$\,0.00} & 0.48 / 0.73\,{\scriptsize $\pm$\,0.02} & 0.93\,{\scriptsize $\pm$\,0.01} & 0.64 / 0.80\,{\scriptsize $\pm$\,0.00} \\
    M\_DENSE\_H & 1.34 & 0.72 / 0.79 & 1.26 & 0.73 / 0.80 & 1.38 & 0.84 / 0.80 & 1.27 & 0.72 / 0.81 & 1.25 & 0.53 / 0.69 & 1.31\,{\scriptsize $\pm$\,0.00} & 0.72 / 0.79\,{\scriptsize $\pm$\,0.00} & 1.19\,{\scriptsize $\pm$\,0.01} & 0.68 / 0.78\,{\scriptsize $\pm$\,0.00} & \textbf{0.67}\,{\scriptsize $\pm$\,0.00} & 0.47 / 0.80\,{\scriptsize $\pm$\,0.00} \\
    temperature\_rain & 1.01 & 0.36 / 0.79 & 0.95 & 0.46 / 0.80 & 0.99 & 0.37 / 0.78 & 0.94 & 0.38 / 0.80 & 1.03 & 0.29 / 0.73 & 1.03\,{\scriptsize $\pm$\,0.00} & 0.36 / 0.76\,{\scriptsize $\pm$\,0.00} & \textbf{0.93}\,{\scriptsize $\pm$\,0.01} & 0.34 / 0.78\,{\scriptsize $\pm$\,0.00} & 1.04\,{\scriptsize $\pm$\,0.07} & 0.50 / 0.79\,{\scriptsize $\pm$\,0.00} \\
    m4\_hourly & 0.95 & 0.33 / 0.73 & 0.85 & 0.43 / 0.80 & 0.93 & 0.53 / 0.76 & 0.88 & 0.40 / 0.79 & 0.88 & 0.26 / 0.61 & 0.98\,{\scriptsize $\pm$\,0.00} & 0.39 / 0.75\,{\scriptsize $\pm$\,0.00} & 0.99\,{\scriptsize $\pm$\,0.01} & 0.27 / 0.65\,{\scriptsize $\pm$\,0.02} & \textbf{0.69}\,{\scriptsize $\pm$\,0.00} & 0.44 / 0.80\,{\scriptsize $\pm$\,0.00} \\
    m4\_daily & 0.06 & 0.03 / 0.81 & 0.06 & 0.03 / 0.80 & 0.06 & 0.04 / 0.79 & 0.06 & 0.03 / 0.80 & 0.05 & 0.02 / 0.72 & 0.07\,{\scriptsize $\pm$\,0.01} & 0.03 / 0.76\,{\scriptsize $\pm$\,0.06} & 0.06\,{\scriptsize $\pm$\,0.00} & 0.03 / 0.75\,{\scriptsize $\pm$\,0.03} & \textbf{0.04}\,{\scriptsize $\pm$\,0.00} & 0.03 / 0.80\,{\scriptsize $\pm$\,0.00} \\
    m4\_weekly & 0.17 & 0.06 / 0.77 & 0.17 & 0.07 / 0.79 & 0.17 & 0.07 / 0.78 & 0.17 & 0.07 / 0.79 & 0.18 & 0.05 / 0.64 & 0.17\,{\scriptsize $\pm$\,0.01} & 0.08 / 0.78\,{\scriptsize $\pm$\,0.02} & 0.20\,{\scriptsize $\pm$\,0.01} & 0.05 / 0.48\,{\scriptsize $\pm$\,0.10} & \textbf{0.16}\,{\scriptsize $\pm$\,0.00} & 0.09 / 0.78\,{\scriptsize $\pm$\,0.00} \\
    \addlinespace[2pt]
    \cmidrule(lr){1-17}
    \textit{Average} & 0.90 & 0.46 / 0.79 & 0.84 & 0.48 / 0.80 & 0.90 & 0.52 / 0.79 & 0.84 & 0.47 / 0.80 & 0.84 & 0.35 / 0.70 & 0.90 & 0.47 / 0.77 & 0.83 & 0.42 / 0.72 & \textbf{0.68} & 0.43 / 0.80 \\
    \bottomrule
  \end{tabular}%
  }
\end{table}

\clearpage

\subsubsection{Full conformal prediction baseline results - Bench22}

\begin{table}[h]
  \caption{CP methods at 80\% target coverage on Gift-Eval Bench22 (22 datasets, 4 backbones: TiRex, TimesFM-2.5, TabPFN-TS, ARIMA). Cells show nWink (mean Winkler $/$ std$(y)$) and nW (mean width $/$ std$(y)$) $/$ Cov, where std$(y)$ is the standard deviation of the target on the eval split. ResCP, HopCPT and RareCP report mean$\,\pm\,$std over ten seeds, other methods are deterministic. \textbf{Bold} marks the smallest displayed nWink per row among methods with $\mathrm{Cov} \ge 0.70$. Each backbone block ends with an \textit{Average} row over its 22 datasets.}
  \label{tab:cp_methods_bench22}
  \centering
  \resizebox{\textwidth}{!}{%
  \begin{tabular}{lcccccccccccccccc}
    \multicolumn{17}{c}{{\large\bfseries\itshape Bench22}} \\[2pt]
    \toprule
    & \multicolumn{2}{c}{Uniform} & \multicolumn{2}{c}{ACI} & \multicolumn{2}{c}{DtACI} & \multicolumn{2}{c}{NexCP} & \multicolumn{2}{c}{KOWCPI} & \multicolumn{2}{c}{HopCPT} & \multicolumn{2}{c}{ResCP} & \multicolumn{2}{c}{RareCP\,\textit{(Ours)}} \\
    \cmidrule(lr){2-3} \cmidrule(lr){4-5} \cmidrule(lr){6-7} \cmidrule(lr){8-9} \cmidrule(lr){10-11} \cmidrule(lr){12-13} \cmidrule(lr){14-15} \cmidrule(lr){16-17}
    \textbf{Dataset} & nWink $\downarrow$ & nW $\downarrow$ / Cov $\uparrow$ & nWink $\downarrow$ & nW $\downarrow$ / Cov $\uparrow$ & nWink $\downarrow$ & nW $\downarrow$ / Cov $\uparrow$ & nWink $\downarrow$ & nW $\downarrow$ / Cov $\uparrow$ & nWink $\downarrow$ & nW $\downarrow$ / Cov $\uparrow$ & nWink $\downarrow$ & nW $\downarrow$ / Cov $\uparrow$ & nWink $\downarrow$ & nW $\downarrow$ / Cov $\uparrow$ & nWink $\downarrow$ & nW $\downarrow$ / Cov $\uparrow$ \\
    \midrule
    \multicolumn{17}{c}{\textit{TiRex backbone}} \\
    \addlinespace[2pt]
    electricity\_H & 0.26 & 0.15 / 0.80 & 0.23 & 0.13 / 0.80 & 0.27 & 0.17 / 0.78 & 0.22 & 0.12 / 0.80 & 0.25 & 0.11 / 0.78 & 0.24\,{\scriptsize $\pm$\,0.01} & 0.12 / 0.78\,{\scriptsize $\pm$\,0.00} & 0.26\,{\scriptsize $\pm$\,0.01} & 0.11 / 0.76\,{\scriptsize $\pm$\,0.06} & \textbf{0.21}\,{\scriptsize $\pm$\,0.00} & 0.13 / 0.80\,{\scriptsize $\pm$\,0.00} \\
    electricity\_15T & 0.33 & 0.19 / 0.81 & 0.30 & 0.16 / 0.80 & 0.33 & 0.20 / 0.79 & 0.29 & 0.16 / 0.80 & 0.33 & 0.14 / 0.70 & 0.31\,{\scriptsize $\pm$\,0.01} & 0.16 / 0.79\,{\scriptsize $\pm$\,0.00} & 0.31\,{\scriptsize $\pm$\,0.01} & 0.14 / 0.76\,{\scriptsize $\pm$\,0.06} & \textbf{0.25}\,{\scriptsize $\pm$\,0.00} & 0.17 / 0.80\,{\scriptsize $\pm$\,0.00} \\
    solar\_H & 0.74 & 0.23 / 0.79 & 0.74 & 0.25 / 0.80 & 0.72 & 0.27 / 0.79 & 0.74 & 0.25 / 0.80 & 0.75 & 0.15 / 0.64 & 0.54\,{\scriptsize $\pm$\,0.01} & 0.25 / 0.78\,{\scriptsize $\pm$\,0.01} & 0.52\,{\scriptsize $\pm$\,0.00} & 0.25 / 0.79\,{\scriptsize $\pm$\,0.00} & \textbf{0.48}\,{\scriptsize $\pm$\,0.00} & 0.29 / 0.80\,{\scriptsize $\pm$\,0.00} \\
    LOOP\_SEATTLE\_5T & 2.24 & 1.40 / 0.80 & 2.21 & 1.43 / 0.80 & 2.19 & 1.44 / 0.79 & 2.22 & 1.42 / 0.80 & 2.27 & 1.46 / 0.81 & 2.26\,{\scriptsize $\pm$\,0.00} & 1.41 / 0.80\,{\scriptsize $\pm$\,0.00} & 2.26\,{\scriptsize $\pm$\,0.01} & 1.50 / 0.82\,{\scriptsize $\pm$\,0.00} & \textbf{2.16}\,{\scriptsize $\pm$\,0.01} & 1.50 / 0.80\,{\scriptsize $\pm$\,0.00} \\
    SZ\_TAXI\_15T & 1.00 & 0.52 / 0.78 & \textbf{0.85} & 0.52 / 0.81 & 1.00 & 0.64 / 0.79 & \textbf{0.85} & 0.50 / 0.81 & 1.04 & 0.26 / 0.61 & 1.03\,{\scriptsize $\pm$\,0.00} & 0.53 / 0.77\,{\scriptsize $\pm$\,0.00} & 0.91\,{\scriptsize $\pm$\,0.01} & 0.50 / 0.81\,{\scriptsize $\pm$\,0.01} & 0.93\,{\scriptsize $\pm$\,0.00} & 0.63 / 0.80\,{\scriptsize $\pm$\,0.00} \\
    M\_DENSE\_H & 0.96 & 0.49 / 0.79 & 0.92 & 0.50 / 0.80 & 0.99 & 0.58 / 0.79 & 0.92 & 0.50 / 0.80 & 1.02 & 0.20 / 0.54 & 0.94\,{\scriptsize $\pm$\,0.04} & 0.57 / 0.83\,{\scriptsize $\pm$\,0.03} & 0.90\,{\scriptsize $\pm$\,0.01} & 0.50 / 0.79\,{\scriptsize $\pm$\,0.01} & \textbf{0.63}\,{\scriptsize $\pm$\,0.00} & 0.46 / 0.80\,{\scriptsize $\pm$\,0.00} \\
    temperature\_rain & 0.82 & 0.35 / 0.79 & 0.85 & 0.39 / 0.80 & 0.83 & 0.37 / 0.79 & 0.81 & 0.35 / 0.80 & 0.82 & 0.16 / 0.44 & 0.83\,{\scriptsize $\pm$\,0.00} & 0.36 / 0.80\,{\scriptsize $\pm$\,0.01} & \textbf{0.74}\,{\scriptsize $\pm$\,0.01} & 0.30 / 0.81\,{\scriptsize $\pm$\,0.00} & 0.80\,{\scriptsize $\pm$\,0.06} & 0.36 / 0.79\,{\scriptsize $\pm$\,0.00} \\
    m4\_hourly & 0.63 & 0.22 / 0.73 & 0.59 & 0.29 / 0.80 & 0.63 & 0.36 / 0.76 & 0.59 & 0.26 / 0.79 & 0.70 & 0.14 / 0.57 & 0.63\,{\scriptsize $\pm$\,0.00} & 0.24 / 0.75\,{\scriptsize $\pm$\,0.00} & 0.64\,{\scriptsize $\pm$\,0.02} & 0.19 / 0.67\,{\scriptsize $\pm$\,0.01} & \textbf{0.56}\,{\scriptsize $\pm$\,0.00} & 0.35 / 0.79\,{\scriptsize $\pm$\,0.00} \\
    m4\_daily & 0.06 & 0.03 / 0.81 & 0.06 & 0.03 / 0.80 & 0.06 & 0.04 / 0.79 & 0.06 & 0.03 / 0.80 & 0.06 & 0.02 / 0.66 & 0.06\,{\scriptsize $\pm$\,0.00} & 0.03 / 0.80\,{\scriptsize $\pm$\,0.01} & 0.05\,{\scriptsize $\pm$\,0.00} & 0.03 / 0.78\,{\scriptsize $\pm$\,0.03} & \textbf{0.04}\,{\scriptsize $\pm$\,0.00} & 0.03 / 0.80\,{\scriptsize $\pm$\,0.00} \\
    m4\_weekly & \textbf{0.14} & 0.06 / 0.78 & \textbf{0.14} & 0.06 / 0.79 & \textbf{0.14} & 0.06 / 0.77 & \textbf{0.14} & 0.06 / 0.79 & 0.15 & 0.06 / 0.74 & 0.16\,{\scriptsize $\pm$\,0.00} & 0.09 / 0.78\,{\scriptsize $\pm$\,0.02} & 0.15\,{\scriptsize $\pm$\,0.00} & 0.07 / 0.71\,{\scriptsize $\pm$\,0.02} & \textbf{0.14}\,{\scriptsize $\pm$\,0.00} & 0.08 / 0.78\,{\scriptsize $\pm$\,0.00} \\
    solar\_10T & 0.48 & 0.11 / 0.80 & 0.47 & 0.13 / 0.80 & 0.48 & 0.15 / 0.79 & 0.47 & 0.14 / 0.81 & 0.44 & 0.06 / 0.68 & 0.48\,{\scriptsize $\pm$\,0.16} & 0.19 / 0.76\,{\scriptsize $\pm$\,0.06} & \textbf{0.33}\,{\scriptsize $\pm$\,0.01} & 0.11 / 0.76\,{\scriptsize $\pm$\,0.04} & 0.36\,{\scriptsize $\pm$\,0.00} & 0.17 / 0.80\,{\scriptsize $\pm$\,0.00} \\
    LOOP\_SEATTLE\_H & 1.62 & 0.82 / 0.79 & 1.59 & 0.85 / 0.80 & 1.64 & 0.92 / 0.79 & 1.59 & 0.83 / 0.80 & 1.69 & 0.72 / 0.70 & 1.64\,{\scriptsize $\pm$\,0.00} & 0.85 / 0.80\,{\scriptsize $\pm$\,0.01} & 1.49\,{\scriptsize $\pm$\,0.01} & 0.94 / 0.82\,{\scriptsize $\pm$\,0.00} & \textbf{1.25}\,{\scriptsize $\pm$\,0.00} & 0.90 / 0.80\,{\scriptsize $\pm$\,0.00} \\
    m4\_monthly & 0.18 & 0.06 / 0.81 & 0.16 & 0.06 / 0.80 & 0.17 & 0.07 / 0.78 & 0.16 & 0.06 / 0.80 & 0.17 & 0.04 / 0.65 & 0.18\,{\scriptsize $\pm$\,0.01} & 0.06 / 0.77\,{\scriptsize $\pm$\,0.02} & 0.15\,{\scriptsize $\pm$\,0.01} & 0.07 / 0.81\,{\scriptsize $\pm$\,0.05} & \textbf{0.13}\,{\scriptsize $\pm$\,0.00} & 0.08 / 0.79\,{\scriptsize $\pm$\,0.00} \\
    electricity\_D & 0.24 & 0.08 / 0.77 & 0.21 & 0.10 / 0.80 & 0.25 & 0.13 / 0.78 & 0.21 & 0.09 / 0.80 & 0.23 & 0.05 / 0.58 & 0.26\,{\scriptsize $\pm$\,0.00} & 0.09 / 0.76\,{\scriptsize $\pm$\,0.00} & 0.22\,{\scriptsize $\pm$\,0.01} & 0.06 / 0.70\,{\scriptsize $\pm$\,0.02} & \textbf{0.17}\,{\scriptsize $\pm$\,0.00} & 0.11 / 0.80\,{\scriptsize $\pm$\,0.00} \\
    SZ\_TAXI\_H & 0.84 & 0.43 / 0.77 & \textbf{0.78} & 0.48 / 0.80 & 0.81 & 0.47 / 0.78 & \textbf{0.78} & 0.46 / 0.80 & 0.86 & 0.37 / 0.68 & 0.86\,{\scriptsize $\pm$\,0.03} & 0.51 / 0.78\,{\scriptsize $\pm$\,0.03} & 0.82\,{\scriptsize $\pm$\,0.00} & 0.47 / 0.79\,{\scriptsize $\pm$\,0.01} & 0.83\,{\scriptsize $\pm$\,0.00} & 0.57 / 0.80\,{\scriptsize $\pm$\,0.00} \\
    m4\_quarterly & 0.14 & 0.10 / 0.91 & 0.13 & 0.08 / 0.91 & 0.13 & 0.08 / 0.91 & 0.13 & 0.08 / 0.91 & 0.12 & 0.06 / 0.90 & 0.18\,{\scriptsize $\pm$\,0.09} & 0.17 / 0.97\,{\scriptsize $\pm$\,0.01} & 0.14\,{\scriptsize $\pm$\,0.01} & 0.03 / 0.55\,{\scriptsize $\pm$\,0.15} & \textbf{0.09}\,{\scriptsize $\pm$\,0.00} & 0.05 / 0.71\,{\scriptsize $\pm$\,0.01} \\
    saugeenday\_D & 0.93 & 0.20 / 0.71 & \textbf{0.90} & 0.22 / 0.72 & 0.93 & 0.19 / 0.72 & 0.92 & 0.19 / 0.70 & 0.92 & 0.18 / 0.67 & 0.95\,{\scriptsize $\pm$\,0.05} & 0.17 / 0.65\,{\scriptsize $\pm$\,0.04} & 0.90\,{\scriptsize $\pm$\,0.02} & 0.13 / 0.43\,{\scriptsize $\pm$\,0.03} & 1.07\,{\scriptsize $\pm$\,0.02} & 0.41 / 0.69\,{\scriptsize $\pm$\,0.01} \\
    M\_DENSE\_D & 0.86 & 0.17 / 0.58 & 0.78 & 0.34 / 0.80 & 0.80 & 0.22 / 0.67 & 0.78 & 0.29 / 0.76 & 0.73 & 0.12 / 0.36 & 0.89\,{\scriptsize $\pm$\,0.00} & 0.19 / 0.56\,{\scriptsize $\pm$\,0.00} & 0.64\,{\scriptsize $\pm$\,0.02} & 0.24 / 0.57\,{\scriptsize $\pm$\,0.02} & \textbf{0.49}\,{\scriptsize $\pm$\,0.01} & 0.33 / 0.80\,{\scriptsize $\pm$\,0.00} \\
    us\_births\_D & 1.35 & 0.56 / 0.71 & \textbf{1.33} & 0.63 / 0.74 & 1.34 & 0.58 / 0.72 & 1.34 & 0.60 / 0.74 & 1.57 & 0.55 / 0.64 & 1.44\,{\scriptsize $\pm$\,0.01} & 0.55 / 0.65\,{\scriptsize $\pm$\,0.00} & 1.80\,{\scriptsize $\pm$\,0.02} & 0.45 / 0.51\,{\scriptsize $\pm$\,0.01} & 1.61\,{\scriptsize $\pm$\,0.03} & 0.95 / 0.74\,{\scriptsize $\pm$\,0.00} \\
    saugeenday\_W & 3.29 & 0.81 / 0.80 & 3.34 & 0.96 / 0.83 & \textbf{3.28} & 0.92 / 0.83 & 3.37 & 0.89 / 0.80 & 3.84 & 1.30 / 0.71 & 3.25\,{\scriptsize $\pm$\,0.01} & 0.34 / 0.59\,{\scriptsize $\pm$\,0.03} & 4.31\,{\scriptsize $\pm$\,0.18} & 0.75 / 0.38\,{\scriptsize $\pm$\,0.06} & 3.05\,{\scriptsize $\pm$\,0.07} & 0.91 / 0.59\,{\scriptsize $\pm$\,0.02} \\
    us\_births\_W & 2.51 & 1.02 / 0.66 & 2.45 & 1.16 / 0.69 & 2.52 & 1.14 / 0.66 & 2.51 & 1.06 / 0.66 & \textbf{2.40} & 1.33 / 0.72 & 2.68\,{\scriptsize $\pm$\,0.00} & 0.76 / 0.55\,{\scriptsize $\pm$\,0.00} & 4.20\,{\scriptsize $\pm$\,0.31} & 0.63 / 0.34\,{\scriptsize $\pm$\,0.02} & 2.98\,{\scriptsize $\pm$\,0.04} & 1.40 / 0.62\,{\scriptsize $\pm$\,0.02} \\
    saugeenday\_M & 4.41 & 2.57 / 0.82 & 4.28 & 2.78 / 0.82 & 4.45 & 2.95 / 0.82 & 4.24 & 2.74 / 0.82 & \textbf{4.03} & 2.30 / 0.82 & 4.98\,{\scriptsize $\pm$\,2.84} & 2.45 / 0.76\,{\scriptsize $\pm$\,0.27} & 7.14\,{\scriptsize $\pm$\,0.45} & 0.22 / 0.12\,{\scriptsize $\pm$\,0.05} & 4.08\,{\scriptsize $\pm$\,0.34} & 0.70 / 0.42\,{\scriptsize $\pm$\,0.02} \\
    \addlinespace[2pt]
    \cmidrule(lr){1-17}
    \textit{Average} & 1.09 & 0.48 / 0.77 & 1.06 & 0.53 / 0.80 & 1.09 & 0.54 / 0.78 & 1.06 & 0.50 / 0.79 & 1.11 & 0.44 / 0.66 & 1.13 & 0.46 / 0.75 & 1.31 & 0.35 / 0.66 & \textbf{1.01} & 0.48 / 0.75 \\
    \midrule
    \multicolumn{17}{c}{\textit{TimesFM-2.5 backbone}} \\
    \addlinespace[2pt]
    electricity\_H & 0.34 & 0.20 / 0.80 & 0.28 & 0.17 / 0.80 & 0.35 & 0.22 / 0.79 & 0.28 & 0.16 / 0.80 & 0.31 & 0.09 / 0.63 & 0.31\,{\scriptsize $\pm$\,0.00} & 0.18 / 0.79\,{\scriptsize $\pm$\,0.00} & 0.33\,{\scriptsize $\pm$\,0.01} & 0.15 / 0.80\,{\scriptsize $\pm$\,0.01} & \textbf{0.22}\,{\scriptsize $\pm$\,0.00} & 0.14 / 0.80\,{\scriptsize $\pm$\,0.00} \\
    electricity\_15T & 0.37 & 0.20 / 0.81 & 0.34 & 0.17 / 0.80 & 0.37 & 0.21 / 0.79 & 0.33 & 0.17 / 0.80 & 0.35 & 0.10 / 0.65 & 0.36\,{\scriptsize $\pm$\,0.00} & 0.18 / 0.79\,{\scriptsize $\pm$\,0.00} & 0.35\,{\scriptsize $\pm$\,0.01} & 0.17 / 0.81\,{\scriptsize $\pm$\,0.02} & \textbf{0.25}\,{\scriptsize $\pm$\,0.00} & 0.17 / 0.80\,{\scriptsize $\pm$\,0.00} \\
    solar\_H & 1.03 & 0.36 / 0.79 & 1.02 & 0.39 / 0.80 & 0.98 & 0.42 / 0.79 & 1.02 & 0.38 / 0.80 & 0.79 & 0.28 / 0.75 & 0.98\,{\scriptsize $\pm$\,0.07} & 0.41 / 0.81\,{\scriptsize $\pm$\,0.03} & 0.76\,{\scriptsize $\pm$\,0.01} & 0.32 / 0.78\,{\scriptsize $\pm$\,0.01} & \textbf{0.57}\,{\scriptsize $\pm$\,0.00} & 0.38 / 0.80\,{\scriptsize $\pm$\,0.00} \\
    LOOP\_SEATTLE\_5T & 2.22 & 1.38 / 0.80 & 2.19 & 1.41 / 0.80 & 2.17 & 1.41 / 0.79 & 2.19 & 1.40 / 0.80 & 2.21 & 1.16 / 0.73 & 2.25\,{\scriptsize $\pm$\,0.09} & 1.39 / 0.79\,{\scriptsize $\pm$\,0.03} & 2.16\,{\scriptsize $\pm$\,0.01} & 1.40 / 0.79\,{\scriptsize $\pm$\,0.00} & \textbf{2.14}\,{\scriptsize $\pm$\,0.01} & 1.51 / 0.80\,{\scriptsize $\pm$\,0.00} \\
    SZ\_TAXI\_15T & 0.99 & 0.51 / 0.78 & 0.84 & 0.51 / 0.81 & 0.98 & 0.63 / 0.79 & \textbf{0.83} & 0.49 / 0.81 & 0.97 & 0.30 / 0.66 & 1.02\,{\scriptsize $\pm$\,0.00} & 0.51 / 0.76\,{\scriptsize $\pm$\,0.00} & 0.84\,{\scriptsize $\pm$\,0.00} & 0.46 / 0.80\,{\scriptsize $\pm$\,0.01} & 0.95\,{\scriptsize $\pm$\,0.01} & 0.65 / 0.80\,{\scriptsize $\pm$\,0.00} \\
    M\_DENSE\_H & 0.97 & 0.49 / 0.79 & 0.93 & 0.52 / 0.80 & 1.00 & 0.57 / 0.79 & 0.93 & 0.51 / 0.80 & 0.99 & 0.30 / 0.63 & 0.96\,{\scriptsize $\pm$\,0.00} & 0.51 / 0.80\,{\scriptsize $\pm$\,0.00} & 0.90\,{\scriptsize $\pm$\,0.00} & 0.46 / 0.76\,{\scriptsize $\pm$\,0.01} & \textbf{0.62}\,{\scriptsize $\pm$\,0.00} & 0.46 / 0.80\,{\scriptsize $\pm$\,0.00} \\
    temperature\_rain & 0.83 & 0.35 / 0.79 & 0.81 & 0.35 / 0.80 & 0.83 & 0.38 / 0.79 & 0.81 & 0.35 / 0.80 & 0.87 & 0.28 / 0.71 & 0.83\,{\scriptsize $\pm$\,0.00} & 0.36 / 0.79\,{\scriptsize $\pm$\,0.01} & \textbf{0.77}\,{\scriptsize $\pm$\,0.00} & 0.33 / 0.81\,{\scriptsize $\pm$\,0.00} & 0.90\,{\scriptsize $\pm$\,0.03} & 0.45 / 0.79\,{\scriptsize $\pm$\,0.00} \\
    m4\_hourly & 0.72 & 0.24 / 0.72 & 0.67 & 0.33 / 0.80 & 0.71 & 0.39 / 0.76 & 0.67 & 0.29 / 0.79 & 0.70 & 0.17 / 0.58 & 0.72\,{\scriptsize $\pm$\,0.00} & 0.27 / 0.73\,{\scriptsize $\pm$\,0.00} & 0.68\,{\scriptsize $\pm$\,0.03} & 0.21 / 0.67\,{\scriptsize $\pm$\,0.01} & \textbf{0.59}\,{\scriptsize $\pm$\,0.00} & 0.38 / 0.79\,{\scriptsize $\pm$\,0.00} \\
    m4\_daily & 0.06 & 0.03 / 0.81 & 0.06 & 0.03 / 0.80 & 0.06 & 0.04 / 0.79 & 0.06 & 0.03 / 0.80 & 0.06 & 0.02 / 0.67 & 0.06\,{\scriptsize $\pm$\,0.01} & 0.03 / 0.79\,{\scriptsize $\pm$\,0.02} & 0.06\,{\scriptsize $\pm$\,0.00} & 0.03 / 0.78\,{\scriptsize $\pm$\,0.01} & \textbf{0.04}\,{\scriptsize $\pm$\,0.00} & 0.03 / 0.80\,{\scriptsize $\pm$\,0.00} \\
    m4\_weekly & \textbf{0.15} & 0.07 / 0.79 & \textbf{0.15} & 0.07 / 0.79 & \textbf{0.15} & 0.06 / 0.78 & \textbf{0.15} & 0.07 / 0.80 & 0.16 & 0.04 / 0.67 & \textbf{0.15}\,{\scriptsize $\pm$\,0.00} & 0.07 / 0.79\,{\scriptsize $\pm$\,0.01} & 0.16\,{\scriptsize $\pm$\,0.01} & 0.05 / 0.63\,{\scriptsize $\pm$\,0.07} & \textbf{0.15}\,{\scriptsize $\pm$\,0.00} & 0.09 / 0.77\,{\scriptsize $\pm$\,0.00} \\
    solar\_10T & 0.50 & 0.13 / 0.80 & 0.49 & 0.15 / 0.80 & 0.47 & 0.14 / 0.79 & 0.49 & 0.16 / 0.81 & 0.36 & 0.09 / 0.64 & 0.50\,{\scriptsize $\pm$\,0.02} & 0.15 / 0.80\,{\scriptsize $\pm$\,0.02} & 0.38\,{\scriptsize $\pm$\,0.02} & 0.11 / 0.80\,{\scriptsize $\pm$\,0.02} & \textbf{0.36}\,{\scriptsize $\pm$\,0.00} & 0.17 / 0.80\,{\scriptsize $\pm$\,0.00} \\
    LOOP\_SEATTLE\_H & 1.59 & 0.83 / 0.80 & 1.56 & 0.86 / 0.80 & 1.64 & 0.95 / 0.79 & 1.56 & 0.85 / 0.80 & 1.51 & 0.72 / 0.75 & 1.60\,{\scriptsize $\pm$\,0.00} & 0.85 / 0.80\,{\scriptsize $\pm$\,0.00} & 1.46\,{\scriptsize $\pm$\,0.01} & 0.82 / 0.79\,{\scriptsize $\pm$\,0.01} & \textbf{1.22}\,{\scriptsize $\pm$\,0.00} & 0.88 / 0.80\,{\scriptsize $\pm$\,0.00} \\
    m4\_monthly & 0.18 & 0.06 / 0.81 & 0.16 & 0.06 / 0.80 & 0.17 & 0.07 / 0.78 & 0.16 & 0.06 / 0.80 & 0.15 & 0.04 / 0.67 & 0.18\,{\scriptsize $\pm$\,0.01} & 0.06 / 0.78\,{\scriptsize $\pm$\,0.02} & 0.14\,{\scriptsize $\pm$\,0.00} & 0.05 / 0.79\,{\scriptsize $\pm$\,0.01} & \textbf{0.13}\,{\scriptsize $\pm$\,0.00} & 0.08 / 0.79\,{\scriptsize $\pm$\,0.00} \\
    electricity\_D & 0.24 & 0.08 / 0.77 & 0.20 & 0.10 / 0.80 & 0.24 & 0.11 / 0.78 & 0.21 & 0.09 / 0.80 & 0.20 & 0.05 / 0.59 & 0.26\,{\scriptsize $\pm$\,0.00} & 0.09 / 0.76\,{\scriptsize $\pm$\,0.00} & 0.22\,{\scriptsize $\pm$\,0.01} & 0.07 / 0.76\,{\scriptsize $\pm$\,0.01} & \textbf{0.18}\,{\scriptsize $\pm$\,0.00} & 0.11 / 0.80\,{\scriptsize $\pm$\,0.00} \\
    SZ\_TAXI\_H & 0.84 & 0.44 / 0.77 & \textbf{0.78} & 0.49 / 0.80 & 0.81 & 0.47 / 0.77 & \textbf{0.78} & 0.47 / 0.80 & 0.89 & 0.28 / 0.61 & 0.86\,{\scriptsize $\pm$\,0.06} & 0.46 / 0.75\,{\scriptsize $\pm$\,0.05} & 0.79\,{\scriptsize $\pm$\,0.00} & 0.42 / 0.75\,{\scriptsize $\pm$\,0.01} & 0.83\,{\scriptsize $\pm$\,0.01} & 0.56 / 0.80\,{\scriptsize $\pm$\,0.00} \\
    m4\_quarterly & 0.12 & 0.07 / 0.90 & \textbf{0.11} & 0.06 / 0.90 & 0.12 & 0.07 / 0.88 & \textbf{0.11} & 0.06 / 0.90 & 0.13 & 0.03 / 0.62 & 0.17\,{\scriptsize $\pm$\,0.02} & 0.16 / 0.97\,{\scriptsize $\pm$\,0.01} & 0.10\,{\scriptsize $\pm$\,0.01} & 0.03 / 0.51\,{\scriptsize $\pm$\,0.07} & 0.08\,{\scriptsize $\pm$\,0.00} & 0.03 / 0.68\,{\scriptsize $\pm$\,0.01} \\
    saugeenday\_D & 1.03 & 0.19 / 0.70 & \textbf{0.99} & 0.23 / 0.73 & 1.02 & 0.19 / 0.70 & 1.01 & 0.19 / 0.70 & 0.85 & 0.17 / 0.51 & 1.10\,{\scriptsize $\pm$\,0.01} & 0.17 / 0.63\,{\scriptsize $\pm$\,0.01} & 0.91\,{\scriptsize $\pm$\,0.04} & 0.19 / 0.50\,{\scriptsize $\pm$\,0.05} & 1.06\,{\scriptsize $\pm$\,0.02} & 0.40 / 0.69\,{\scriptsize $\pm$\,0.00} \\
    M\_DENSE\_D & 0.82 & 0.16 / 0.58 & 0.73 & 0.31 / 0.80 & 0.75 & 0.22 / 0.66 & 0.73 & 0.28 / 0.78 & 0.70 & 0.19 / 0.61 & 0.86\,{\scriptsize $\pm$\,0.01} & 0.17 / 0.55\,{\scriptsize $\pm$\,0.01} & 0.72\,{\scriptsize $\pm$\,0.02} & 0.24 / 0.64\,{\scriptsize $\pm$\,0.05} & \textbf{0.49}\,{\scriptsize $\pm$\,0.01} & 0.34 / 0.79\,{\scriptsize $\pm$\,0.00} \\
    us\_births\_D & \textbf{1.17} & 0.53 / 0.76 & 1.18 & 0.53 / 0.77 & 1.18 & 0.52 / 0.77 & \textbf{1.17} & 0.54 / 0.77 & 1.37 & 0.38 / 0.58 & 1.20\,{\scriptsize $\pm$\,0.00} & 0.46 / 0.71\,{\scriptsize $\pm$\,0.00} & 1.33\,{\scriptsize $\pm$\,0.03} & 0.36 / 0.55\,{\scriptsize $\pm$\,0.02} & 1.61\,{\scriptsize $\pm$\,0.05} & 0.85 / 0.70\,{\scriptsize $\pm$\,0.01} \\
    saugeenday\_W & 2.87 & 0.60 / 0.76 & 2.87 & 0.91 / 0.76 & \textbf{2.79} & 0.82 / 0.76 & 2.89 & 0.73 / 0.76 & 4.90 & 0.75 / 0.32 & 2.79\,{\scriptsize $\pm$\,0.00} & 0.32 / 0.63\,{\scriptsize $\pm$\,0.00} & 4.40\,{\scriptsize $\pm$\,0.10} & 0.63 / 0.34\,{\scriptsize $\pm$\,0.01} & 3.22\,{\scriptsize $\pm$\,0.10} & 0.75 / 0.54\,{\scriptsize $\pm$\,0.02} \\
    us\_births\_W & 2.60 & 0.98 / 0.59 & 2.59 & 1.18 / 0.66 & 2.64 & 1.10 / 0.66 & 2.58 & 1.04 / 0.62 & 6.24 & 0.28 / 0.10 & 2.83\,{\scriptsize $\pm$\,0.00} & 0.72 / 0.52\,{\scriptsize $\pm$\,0.00} & 3.85\,{\scriptsize $\pm$\,0.12} & 0.72 / 0.41\,{\scriptsize $\pm$\,0.00} & 3.07\,{\scriptsize $\pm$\,0.06} & 1.27 / 0.68\,{\scriptsize $\pm$\,0.01} \\
    saugeenday\_M & 4.04 & 2.12 / 0.82 & 3.93 & 2.38 / 0.82 & 4.12 & 2.58 / 0.82 & 3.89 & 2.34 / 0.82 & 5.40 & 0.59 / 0.14 & \textbf{3.33}\,{\scriptsize $\pm$\,0.00} & 1.41 / 0.82\,{\scriptsize $\pm$\,0.00} & 6.13\,{\scriptsize $\pm$\,0.23} & 0.52 / 0.23\,{\scriptsize $\pm$\,0.00} & 3.37\,{\scriptsize $\pm$\,0.08} & 0.74 / 0.41\,{\scriptsize $\pm$\,0.04} \\
    \addlinespace[2pt]
    \cmidrule(lr){1-17}
    \textit{Average} & 1.08 & 0.46 / 0.77 & 1.04 & 0.51 / 0.79 & 1.07 & 0.53 / 0.77 & 1.04 & 0.48 / 0.79 & 1.37 & 0.29 / 0.58 & 1.06 & 0.41 / 0.75 & 1.25 & 0.35 / 0.67 & \textbf{1.00} & 0.47 / 0.75 \\
    \midrule
    \multicolumn{17}{c}{\textit{TabPFN-TS backbone}} \\
    \addlinespace[2pt]
    electricity\_H & 0.32 & 0.17 / 0.80 & 0.27 & 0.15 / 0.80 & 0.32 & 0.20 / 0.79 & 0.27 & 0.14 / 0.80 & 0.28 & 0.10 / 0.68 & 0.30\,{\scriptsize $\pm$\,0.01} & 0.16 / 0.77\,{\scriptsize $\pm$\,0.00} & 0.30\,{\scriptsize $\pm$\,0.01} & 0.12 / 0.74\,{\scriptsize $\pm$\,0.04} & \textbf{0.23}\,{\scriptsize $\pm$\,0.00} & 0.15 / 0.80\,{\scriptsize $\pm$\,0.00} \\
    electricity\_15T & 0.35 & 0.19 / 0.81 & 0.32 & 0.17 / 0.80 & 0.35 & 0.20 / 0.79 & 0.32 & 0.16 / 0.80 & 0.33 & 0.12 / 0.69 & 0.34\,{\scriptsize $\pm$\,0.00} & 0.17 / 0.79\,{\scriptsize $\pm$\,0.00} & 0.34\,{\scriptsize $\pm$\,0.00} & 0.14 / 0.74\,{\scriptsize $\pm$\,0.04} & \textbf{0.25}\,{\scriptsize $\pm$\,0.00} & 0.17 / 0.80\,{\scriptsize $\pm$\,0.00} \\
    solar\_H & 1.17 & 0.35 / 0.80 & 1.16 & 0.39 / 0.80 & 1.12 & 0.39 / 0.79 & 1.16 & 0.39 / 0.80 & 0.68 & 0.25 / 0.74 & 1.13\,{\scriptsize $\pm$\,0.04} & 0.24 / 0.76\,{\scriptsize $\pm$\,0.05} & 0.66\,{\scriptsize $\pm$\,0.01} & 0.32 / 0.75\,{\scriptsize $\pm$\,0.02} & \textbf{0.52}\,{\scriptsize $\pm$\,0.00} & 0.33 / 0.79\,{\scriptsize $\pm$\,0.00} \\
    LOOP\_SEATTLE\_5T & 2.27 & 1.43 / 0.80 & 2.24 & 1.45 / 0.80 & 2.24 & 1.48 / 0.79 & 2.25 & 1.45 / 0.80 & 2.23 & 1.27 / 0.76 & 2.31\,{\scriptsize $\pm$\,0.00} & 1.46 / 0.80\,{\scriptsize $\pm$\,0.00} & 2.26\,{\scriptsize $\pm$\,0.01} & 1.51 / 0.80\,{\scriptsize $\pm$\,0.00} & \textbf{2.16}\,{\scriptsize $\pm$\,0.01} & 1.49 / 0.80\,{\scriptsize $\pm$\,0.00} \\
    SZ\_TAXI\_15T & 1.00 & 0.52 / 0.79 & 0.86 & 0.51 / 0.81 & 1.01 & 0.64 / 0.79 & \textbf{0.85} & 0.50 / 0.81 & 0.92 & 0.36 / 0.71 & 1.01\,{\scriptsize $\pm$\,0.02} & 0.52 / 0.77\,{\scriptsize $\pm$\,0.01} & 0.86\,{\scriptsize $\pm$\,0.00} & 0.46 / 0.74\,{\scriptsize $\pm$\,0.01} & 0.95\,{\scriptsize $\pm$\,0.00} & 0.65 / 0.80\,{\scriptsize $\pm$\,0.00} \\
    M\_DENSE\_H & 1.08 & 0.52 / 0.79 & 1.04 & 0.54 / 0.80 & 1.11 & 0.60 / 0.79 & 1.04 & 0.53 / 0.80 & 1.06 & 0.36 / 0.66 & 1.08\,{\scriptsize $\pm$\,0.00} & 0.54 / 0.80\,{\scriptsize $\pm$\,0.00} & 0.97\,{\scriptsize $\pm$\,0.00} & 0.53 / 0.75\,{\scriptsize $\pm$\,0.01} & \textbf{0.67}\,{\scriptsize $\pm$\,0.00} & 0.48 / 0.80\,{\scriptsize $\pm$\,0.00} \\
    temperature\_rain & 0.86 & 0.38 / 0.79 & 0.84 & 0.39 / 0.80 & 0.87 & 0.41 / 0.79 & 0.84 & 0.38 / 0.80 & 0.82 & 0.30 / 0.74 & 0.86\,{\scriptsize $\pm$\,0.00} & 0.38 / 0.78\,{\scriptsize $\pm$\,0.01} & \textbf{0.77}\,{\scriptsize $\pm$\,0.00} & 0.33 / 0.81\,{\scriptsize $\pm$\,0.00} & \textbf{0.77}\,{\scriptsize $\pm$\,0.03} & 0.33 / 0.79\,{\scriptsize $\pm$\,0.00} \\
    m4\_hourly & 0.82 & 0.26 / 0.72 & 0.76 & 0.36 / 0.80 & 0.81 & 0.45 / 0.76 & 0.77 & 0.31 / 0.79 & 0.74 & 0.21 / 0.63 & 0.82\,{\scriptsize $\pm$\,0.00} & 0.29 / 0.74\,{\scriptsize $\pm$\,0.00} & 0.80\,{\scriptsize $\pm$\,0.04} & 0.22 / 0.65\,{\scriptsize $\pm$\,0.04} & \textbf{0.64}\,{\scriptsize $\pm$\,0.00} & 0.42 / 0.79\,{\scriptsize $\pm$\,0.00} \\
    m4\_daily & 0.06 & 0.03 / 0.80 & 0.06 & 0.03 / 0.80 & 0.07 & 0.04 / 0.79 & 0.06 & 0.03 / 0.80 & 0.06 & 0.02 / 0.71 & 0.06\,{\scriptsize $\pm$\,0.00} & 0.03 / 0.80\,{\scriptsize $\pm$\,0.06} & 0.06\,{\scriptsize $\pm$\,0.00} & 0.03 / 0.75\,{\scriptsize $\pm$\,0.04} & \textbf{0.04}\,{\scriptsize $\pm$\,0.00} & 0.03 / 0.80\,{\scriptsize $\pm$\,0.00} \\
    m4\_weekly & 0.16 & 0.06 / 0.79 & 0.16 & 0.06 / 0.79 & 0.16 & 0.06 / 0.76 & 0.16 & 0.06 / 0.80 & 0.16 & 0.05 / 0.68 & 0.17\,{\scriptsize $\pm$\,0.02} & 0.08 / 0.79\,{\scriptsize $\pm$\,0.03} & 0.20\,{\scriptsize $\pm$\,0.01} & 0.04 / 0.48\,{\scriptsize $\pm$\,0.11} & \textbf{0.15}\,{\scriptsize $\pm$\,0.00} & 0.08 / 0.77\,{\scriptsize $\pm$\,0.00} \\
    solar\_10T & 0.87 & 0.15 / 0.80 & 0.85 & 0.14 / 0.80 & 0.84 & 0.14 / 0.80 & 0.85 & 0.19 / 0.83 & \textbf{0.36} & 0.15 / 0.71 & 0.93\,{\scriptsize $\pm$\,0.10} & 0.44 / 0.82\,{\scriptsize $\pm$\,0.09} & 0.40\,{\scriptsize $\pm$\,0.01} & 0.15 / 0.76\,{\scriptsize $\pm$\,0.02} & \textbf{0.36}\,{\scriptsize $\pm$\,0.00} & 0.17 / 0.80\,{\scriptsize $\pm$\,0.00} \\
    LOOP\_SEATTLE\_H & 1.68 & 0.87 / 0.80 & 1.64 & 0.92 / 0.80 & 1.70 & 0.99 / 0.79 & 1.64 & 0.90 / 0.80 & 1.57 & 0.77 / 0.75 & 1.72\,{\scriptsize $\pm$\,0.00} & 0.89 / 0.79\,{\scriptsize $\pm$\,0.00} & 1.52\,{\scriptsize $\pm$\,0.01} & 0.90 / 0.78\,{\scriptsize $\pm$\,0.00} & \textbf{1.29}\,{\scriptsize $\pm$\,0.00} & 0.93 / 0.80\,{\scriptsize $\pm$\,0.00} \\
    m4\_monthly & 0.18 & 0.06 / 0.81 & 0.16 & 0.06 / 0.80 & 0.17 & 0.07 / 0.78 & 0.16 & 0.06 / 0.79 & 0.14 & 0.04 / 0.70 & 0.19\,{\scriptsize $\pm$\,0.00} & 0.06 / 0.77\,{\scriptsize $\pm$\,0.01} & 0.14\,{\scriptsize $\pm$\,0.00} & 0.05 / 0.71\,{\scriptsize $\pm$\,0.04} & \textbf{0.13}\,{\scriptsize $\pm$\,0.00} & 0.08 / 0.80\,{\scriptsize $\pm$\,0.00} \\
    electricity\_D & 0.25 & 0.08 / 0.77 & 0.21 & 0.11 / 0.80 & 0.26 & 0.13 / 0.78 & 0.22 & 0.10 / 0.80 & 0.21 & 0.06 / 0.65 & 0.27\,{\scriptsize $\pm$\,0.00} & 0.10 / 0.77\,{\scriptsize $\pm$\,0.00} & 0.24\,{\scriptsize $\pm$\,0.01} & 0.07 / 0.71\,{\scriptsize $\pm$\,0.04} & \textbf{0.18}\,{\scriptsize $\pm$\,0.00} & 0.11 / 0.80\,{\scriptsize $\pm$\,0.00} \\
    SZ\_TAXI\_H & 0.86 & 0.45 / 0.77 & \textbf{0.80} & 0.50 / 0.80 & 0.84 & 0.49 / 0.77 & \textbf{0.80} & 0.49 / 0.80 & 0.85 & 0.35 / 0.68 & 0.85\,{\scriptsize $\pm$\,0.00} & 0.45 / 0.75\,{\scriptsize $\pm$\,0.01} & 0.82\,{\scriptsize $\pm$\,0.00} & 0.45 / 0.74\,{\scriptsize $\pm$\,0.01} & 0.84\,{\scriptsize $\pm$\,0.00} & 0.56 / 0.79\,{\scriptsize $\pm$\,0.00} \\
    m4\_quarterly & 0.14 & 0.11 / 0.90 & 0.14 & 0.09 / 0.90 & 0.14 & 0.08 / 0.90 & 0.14 & 0.10 / 0.90 & 0.15 & 0.04 / 0.62 & 0.15\,{\scriptsize $\pm$\,0.00} & 0.13 / 0.93\,{\scriptsize $\pm$\,0.01} & 0.15\,{\scriptsize $\pm$\,0.02} & 0.02 / 0.26\,{\scriptsize $\pm$\,0.18} & \textbf{0.08}\,{\scriptsize $\pm$\,0.00} & 0.04 / 0.70\,{\scriptsize $\pm$\,0.01} \\
    saugeenday\_D & 1.21 & 0.21 / 0.66 & \textbf{1.15} & 0.25 / 0.72 & 1.20 & 0.22 / 0.68 & 1.19 & 0.21 / 0.66 & 1.00 & 0.17 / 0.50 & 1.27\,{\scriptsize $\pm$\,0.01} & 0.20 / 0.62\,{\scriptsize $\pm$\,0.05} & 1.23\,{\scriptsize $\pm$\,0.06} & 0.13 / 0.28\,{\scriptsize $\pm$\,0.14} & 1.23\,{\scriptsize $\pm$\,0.01} & 0.41 / 0.64\,{\scriptsize $\pm$\,0.01} \\
    M\_DENSE\_D & 0.90 & 0.15 / 0.56 & 0.82 & 0.33 / 0.80 & 0.84 & 0.23 / 0.65 & 0.83 & 0.28 / 0.76 & 0.75 & 0.21 / 0.61 & 0.93\,{\scriptsize $\pm$\,0.00} & 0.18 / 0.56\,{\scriptsize $\pm$\,0.00} & 0.77\,{\scriptsize $\pm$\,0.02} & 0.20 / 0.47\,{\scriptsize $\pm$\,0.06} & \textbf{0.50}\,{\scriptsize $\pm$\,0.01} & 0.35 / 0.79\,{\scriptsize $\pm$\,0.00} \\
    us\_births\_D & 1.30 & 0.54 / 0.74 & \textbf{1.29} & 0.55 / 0.75 & 1.30 & 0.54 / 0.75 & \textbf{1.29} & 0.55 / 0.75 & 1.42 & 0.37 / 0.50 & 1.48\,{\scriptsize $\pm$\,0.00} & 0.44 / 0.58\,{\scriptsize $\pm$\,0.00} & 1.39\,{\scriptsize $\pm$\,0.04} & 0.37 / 0.49\,{\scriptsize $\pm$\,0.01} & 1.79\,{\scriptsize $\pm$\,0.03} & 0.69 / 0.58\,{\scriptsize $\pm$\,0.01} \\
    saugeenday\_W & 3.16 & 0.72 / 0.78 & 3.23 & 0.86 / 0.78 & \textbf{3.08} & 0.85 / 0.78 & 3.17 & 0.80 / 0.78 & 4.17 & 0.50 / 0.34 & 3.13\,{\scriptsize $\pm$\,0.03} & 0.56 / 0.77\,{\scriptsize $\pm$\,0.05} & 5.02\,{\scriptsize $\pm$\,0.04} & 0.36 / 0.20\,{\scriptsize $\pm$\,0.03} & 4.05\,{\scriptsize $\pm$\,0.03} & 0.62 / 0.50\,{\scriptsize $\pm$\,0.01} \\
    us\_births\_W & 2.86 & 1.05 / 0.62 & 2.75 & 1.19 / 0.66 & 2.77 & 1.17 / 0.66 & 2.84 & 1.11 / 0.66 & 5.01 & 0.61 / 0.34 & 3.21\,{\scriptsize $\pm$\,0.00} & 0.79 / 0.52\,{\scriptsize $\pm$\,0.00} & 4.60\,{\scriptsize $\pm$\,0.22} & 0.51 / 0.29\,{\scriptsize $\pm$\,0.03} & 3.38\,{\scriptsize $\pm$\,0.10} & 1.45 / 0.49\,{\scriptsize $\pm$\,0.03} \\
    saugeenday\_M & 3.37 & 2.16 / 0.82 & 3.20 & 2.32 / 0.82 & 3.23 & 2.35 / 0.82 & 3.19 & 2.31 / 0.82 & 5.87 & 0.88 / 0.50 & \textbf{2.99}\,{\scriptsize $\pm$\,0.08} & 2.19 / 0.87\,{\scriptsize $\pm$\,0.05} & 6.06\,{\scriptsize $\pm$\,0.42} & 0.35 / 0.18\,{\scriptsize $\pm$\,0.02} & 4.69\,{\scriptsize $\pm$\,0.60} & 0.47 / 0.25\,{\scriptsize $\pm$\,0.05} \\
    \addlinespace[2pt]
    \cmidrule(lr){1-17}
    \textit{Average} & 1.13 & 0.48 / 0.77 & \textbf{1.09} & 0.52 / 0.79 & 1.11 & 0.53 / 0.77 & \textbf{1.09} & 0.50 / 0.79 & 1.31 & 0.33 / 0.63 & 1.15 & 0.47 / 0.75 & 1.34 & 0.33 / 0.60 & 1.13 & 0.45 / 0.72 \\
    \midrule
    \multicolumn{17}{c}{\textit{ARIMA backbone}} \\
    \addlinespace[2pt]
    electricity\_H & 0.56 & 0.34 / 0.80 & 0.46 & 0.30 / 0.80 & 0.57 & 0.38 / 0.79 & 0.45 & 0.27 / 0.80 & 0.45 & 0.18 / 0.70 & 0.48\,{\scriptsize $\pm$\,0.01} & 0.28 / 0.78\,{\scriptsize $\pm$\,0.00} & 0.51\,{\scriptsize $\pm$\,0.02} & 0.19 / 0.70\,{\scriptsize $\pm$\,0.02} & \textbf{0.23}\,{\scriptsize $\pm$\,0.00} & 0.16 / 0.80\,{\scriptsize $\pm$\,0.00} \\
    electricity\_15T & 0.36 & 0.20 / 0.81 & 0.33 & 0.18 / 0.80 & 0.37 & 0.22 / 0.79 & 0.32 & 0.17 / 0.81 & 0.33 & 0.13 / 0.70 & 0.34\,{\scriptsize $\pm$\,0.01} & 0.18 / 0.79\,{\scriptsize $\pm$\,0.00} & 0.36\,{\scriptsize $\pm$\,0.01} & 0.14 / 0.74\,{\scriptsize $\pm$\,0.03} & \textbf{0.26}\,{\scriptsize $\pm$\,0.00} & 0.17 / 0.80\,{\scriptsize $\pm$\,0.00} \\
    solar\_H & 1.25 & 0.58 / 0.79 & 1.23 & 0.63 / 0.80 & 1.27 & 0.64 / 0.79 & 1.23 & 0.61 / 0.80 & 1.01 & 0.38 / 0.74 & 1.22\,{\scriptsize $\pm$\,0.02} & 0.65 / 0.79\,{\scriptsize $\pm$\,0.00} & 0.87\,{\scriptsize $\pm$\,0.01} & 0.45 / 0.79\,{\scriptsize $\pm$\,0.01} & \textbf{0.62}\,{\scriptsize $\pm$\,0.00} & 0.40 / 0.80\,{\scriptsize $\pm$\,0.00} \\
    LOOP\_SEATTLE\_5T & 2.29 & 1.46 / 0.80 & 2.25 & 1.48 / 0.80 & 2.25 & 1.48 / 0.79 & 2.27 & 1.48 / 0.80 & 2.23 & 1.29 / 0.75 & 2.34\,{\scriptsize $\pm$\,0.01} & 1.47 / 0.79\,{\scriptsize $\pm$\,0.00} & 2.30\,{\scriptsize $\pm$\,0.01} & 1.54 / 0.80\,{\scriptsize $\pm$\,0.00} & \textbf{2.18}\,{\scriptsize $\pm$\,0.01} & 1.52 / 0.80\,{\scriptsize $\pm$\,0.00} \\
    SZ\_TAXI\_15T & 1.02 & 0.53 / 0.78 & 0.87 & 0.53 / 0.80 & 1.00 & 0.65 / 0.79 & \textbf{0.86} & 0.52 / 0.81 & 0.94 & 0.38 / 0.72 & 1.02\,{\scriptsize $\pm$\,0.00} & 0.54 / 0.75\,{\scriptsize $\pm$\,0.00} & 0.88\,{\scriptsize $\pm$\,0.00} & 0.48 / 0.73\,{\scriptsize $\pm$\,0.02} & 0.94\,{\scriptsize $\pm$\,0.01} & 0.66 / 0.80\,{\scriptsize $\pm$\,0.00} \\
    M\_DENSE\_H & 1.34 & 0.72 / 0.79 & 1.26 & 0.73 / 0.80 & 1.38 & 0.84 / 0.80 & 1.27 & 0.72 / 0.81 & 1.25 & 0.53 / 0.69 & 1.31\,{\scriptsize $\pm$\,0.01} & 0.72 / 0.78\,{\scriptsize $\pm$\,0.00} & 1.19\,{\scriptsize $\pm$\,0.01} & 0.68 / 0.78\,{\scriptsize $\pm$\,0.00} & \textbf{0.68}\,{\scriptsize $\pm$\,0.00} & 0.49 / 0.80\,{\scriptsize $\pm$\,0.00} \\
    temperature\_rain & 1.01 & 0.36 / 0.79 & 0.95 & 0.46 / 0.80 & 0.99 & 0.37 / 0.78 & 0.94 & 0.38 / 0.80 & 1.03 & 0.29 / 0.73 & 1.03\,{\scriptsize $\pm$\,0.00} & 0.35 / 0.75\,{\scriptsize $\pm$\,0.01} & \textbf{0.93}\,{\scriptsize $\pm$\,0.01} & 0.34 / 0.78\,{\scriptsize $\pm$\,0.00} & 1.07\,{\scriptsize $\pm$\,0.05} & 0.48 / 0.79\,{\scriptsize $\pm$\,0.00} \\
    m4\_hourly & 0.95 & 0.33 / 0.73 & 0.85 & 0.43 / 0.80 & 0.93 & 0.53 / 0.76 & 0.88 & 0.40 / 0.79 & 0.88 & 0.26 / 0.61 & 0.98\,{\scriptsize $\pm$\,0.00} & 0.39 / 0.75\,{\scriptsize $\pm$\,0.00} & 0.99\,{\scriptsize $\pm$\,0.01} & 0.27 / 0.65\,{\scriptsize $\pm$\,0.02} & \textbf{0.67}\,{\scriptsize $\pm$\,0.01} & 0.44 / 0.80\,{\scriptsize $\pm$\,0.00} \\
    m4\_daily & 0.06 & 0.03 / 0.81 & 0.06 & 0.03 / 0.80 & 0.06 & 0.04 / 0.79 & 0.06 & 0.03 / 0.80 & 0.05 & 0.02 / 0.72 & 0.06\,{\scriptsize $\pm$\,0.01} & 0.03 / 0.76\,{\scriptsize $\pm$\,0.06} & 0.06\,{\scriptsize $\pm$\,0.00} & 0.03 / 0.75\,{\scriptsize $\pm$\,0.03} & \textbf{0.04}\,{\scriptsize $\pm$\,0.00} & 0.03 / 0.80\,{\scriptsize $\pm$\,0.00} \\
    m4\_weekly & 0.17 & 0.06 / 0.77 & 0.17 & 0.07 / 0.79 & 0.17 & 0.07 / 0.78 & 0.17 & 0.07 / 0.79 & 0.18 & 0.05 / 0.64 & 0.17\,{\scriptsize $\pm$\,0.01} & 0.08 / 0.79\,{\scriptsize $\pm$\,0.02} & 0.20\,{\scriptsize $\pm$\,0.01} & 0.05 / 0.48\,{\scriptsize $\pm$\,0.10} & \textbf{0.15}\,{\scriptsize $\pm$\,0.00} & 0.09 / 0.78\,{\scriptsize $\pm$\,0.00} \\
    solar\_10T & 0.52 & 0.13 / 0.79 & 0.50 & 0.15 / 0.80 & 0.43 & 0.14 / 0.79 & 0.51 & 0.17 / 0.81 & \textbf{0.36} & 0.11 / 0.72 & 0.55\,{\scriptsize $\pm$\,0.00} & 0.16 / 0.78\,{\scriptsize $\pm$\,0.02} & 0.38\,{\scriptsize $\pm$\,0.01} & 0.12 / 0.73\,{\scriptsize $\pm$\,0.02} & 0.38\,{\scriptsize $\pm$\,0.00} & 0.15 / 0.80\,{\scriptsize $\pm$\,0.00} \\
    LOOP\_SEATTLE\_H & 1.95 & 0.98 / 0.79 & 1.90 & 1.00 / 0.80 & 1.94 & 1.09 / 0.79 & 1.90 & 1.00 / 0.80 & 1.74 & 0.88 / 0.75 & 1.98\,{\scriptsize $\pm$\,0.00} & 0.97 / 0.78\,{\scriptsize $\pm$\,0.00} & 1.77\,{\scriptsize $\pm$\,0.01} & 1.07 / 0.80\,{\scriptsize $\pm$\,0.00} & \textbf{1.33}\,{\scriptsize $\pm$\,0.00} & 0.94 / 0.80\,{\scriptsize $\pm$\,0.00} \\
    m4\_monthly & 0.21 & 0.06 / 0.81 & 0.19 & 0.07 / 0.80 & 0.20 & 0.08 / 0.78 & 0.19 & 0.07 / 0.79 & 0.17 & 0.05 / 0.70 & 0.22\,{\scriptsize $\pm$\,0.02} & 0.07 / 0.76\,{\scriptsize $\pm$\,0.02} & 0.17\,{\scriptsize $\pm$\,0.00} & 0.06 / 0.67\,{\scriptsize $\pm$\,0.04} & \textbf{0.14}\,{\scriptsize $\pm$\,0.00} & 0.09 / 0.80\,{\scriptsize $\pm$\,0.00} \\
    electricity\_D & 0.26 & 0.09 / 0.77 & 0.22 & 0.11 / 0.80 & 0.27 & 0.14 / 0.78 & 0.23 & 0.10 / 0.80 & 0.21 & 0.07 / 0.66 & 0.28\,{\scriptsize $\pm$\,0.00} & 0.10 / 0.76\,{\scriptsize $\pm$\,0.00} & 0.26\,{\scriptsize $\pm$\,0.00} & 0.07 / 0.72\,{\scriptsize $\pm$\,0.02} & \textbf{0.20}\,{\scriptsize $\pm$\,0.00} & 0.12 / 0.80\,{\scriptsize $\pm$\,0.00} \\
    SZ\_TAXI\_H & 0.88 & 0.46 / 0.77 & \textbf{0.81} & 0.52 / 0.80 & 0.85 & 0.51 / 0.77 & 0.82 & 0.50 / 0.80 & 0.85 & 0.37 / 0.68 & 0.89\,{\scriptsize $\pm$\,0.01} & 0.52 / 0.78\,{\scriptsize $\pm$\,0.00} & 0.85\,{\scriptsize $\pm$\,0.00} & 0.47 / 0.73\,{\scriptsize $\pm$\,0.01} & \textbf{0.81}\,{\scriptsize $\pm$\,0.00} & 0.54 / 0.80\,{\scriptsize $\pm$\,0.00} \\
    m4\_quarterly & 0.15 & 0.09 / 0.90 & \textbf{0.14} & 0.07 / 0.90 & \textbf{0.14} & 0.06 / 0.88 & \textbf{0.14} & 0.07 / 0.90 & 0.18 & 0.03 / 0.49 & 0.21\,{\scriptsize $\pm$\,0.02} & 0.20 / 0.97\,{\scriptsize $\pm$\,0.01} & 0.16\,{\scriptsize $\pm$\,0.02} & 0.02 / 0.21\,{\scriptsize $\pm$\,0.17} & 0.09\,{\scriptsize $\pm$\,0.00} & 0.04 / 0.61\,{\scriptsize $\pm$\,0.01} \\
    saugeenday\_D & 1.17 & 0.26 / 0.77 & \textbf{1.12} & 0.25 / 0.74 & \textbf{1.12} & 0.23 / 0.72 & 1.14 & 0.23 / 0.72 & 1.26 & 0.15 / 0.42 & 1.21\,{\scriptsize $\pm$\,0.05} & 0.41 / 0.76\,{\scriptsize $\pm$\,0.03} & 1.38\,{\scriptsize $\pm$\,0.13} & 0.17 / 0.27\,{\scriptsize $\pm$\,0.16} & 1.15\,{\scriptsize $\pm$\,0.03} & 0.43 / 0.63\,{\scriptsize $\pm$\,0.01} \\
    M\_DENSE\_D & 1.41 & 0.34 / 0.54 & 1.08 & 0.72 / 0.81 & 1.19 & 0.47 / 0.64 & 1.10 & 0.65 / 0.76 & 1.04 & 0.41 / 0.69 & 1.39\,{\scriptsize $\pm$\,0.04} & 0.48 / 0.58\,{\scriptsize $\pm$\,0.02} & 0.97\,{\scriptsize $\pm$\,0.03} & 0.44 / 0.59\,{\scriptsize $\pm$\,0.05} & \textbf{0.58}\,{\scriptsize $\pm$\,0.01} & 0.40 / 0.80\,{\scriptsize $\pm$\,0.00} \\
    us\_births\_D & 2.92 & 2.01 / 0.79 & 2.93 & 1.99 / 0.79 & 2.93 & 2.00 / 0.78 & 2.91 & 2.05 / 0.79 & 2.93 & 1.07 / 0.49 & 3.33\,{\scriptsize $\pm$\,0.04} & 1.72 / 0.68\,{\scriptsize $\pm$\,0.01} & 3.61\,{\scriptsize $\pm$\,0.05} & 1.53 / 0.58\,{\scriptsize $\pm$\,0.01} & \textbf{2.16}\,{\scriptsize $\pm$\,0.07} & 1.38 / 0.72\,{\scriptsize $\pm$\,0.01} \\
    saugeenday\_W & 3.56 & 1.02 / 0.78 & 3.57 & 1.13 / 0.83 & \textbf{3.53} & 1.05 / 0.80 & 3.56 & 1.11 / 0.80 & 5.10 & 0.39 / 0.29 & 3.70\,{\scriptsize $\pm$\,0.00} & 0.35 / 0.51\,{\scriptsize $\pm$\,0.00} & 4.40\,{\scriptsize $\pm$\,0.03} & 0.58 / 0.16\,{\scriptsize $\pm$\,0.01} & 3.79\,{\scriptsize $\pm$\,0.16} & 0.74 / 0.45\,{\scriptsize $\pm$\,0.02} \\
    us\_births\_W & 3.01 & 1.09 / 0.62 & 2.88 & 1.22 / 0.62 & 2.86 & 1.20 / 0.62 & 2.89 & 1.15 / 0.62 & 6.28 & 0.51 / 0.10 & 3.44\,{\scriptsize $\pm$\,0.00} & 0.72 / 0.48\,{\scriptsize $\pm$\,0.00} & 5.46\,{\scriptsize $\pm$\,0.09} & 0.61 / 0.32\,{\scriptsize $\pm$\,0.02} & 3.24\,{\scriptsize $\pm$\,0.08} & 1.26 / 0.52\,{\scriptsize $\pm$\,0.02} \\
    saugeenday\_M & 3.41 & 1.24 / 0.82 & 3.41 & 1.25 / 0.82 & 3.57 & 1.41 / 0.82 & 3.42 & 1.26 / 0.82 & 5.82 & 1.00 / 0.32 & \textbf{3.39}\,{\scriptsize $\pm$\,0.22} & 1.15 / 0.73\,{\scriptsize $\pm$\,0.01} & 5.92\,{\scriptsize $\pm$\,0.06} & 0.80 / 0.34\,{\scriptsize $\pm$\,0.03} & 3.49\,{\scriptsize $\pm$\,0.50} & 0.59 / 0.30\,{\scriptsize $\pm$\,0.02} \\
    \addlinespace[2pt]
    \cmidrule(lr){1-17}
    \textit{Average} & 1.29 & 0.56 / 0.77 & 1.24 & 0.61 / 0.80 & 1.27 & 0.62 / 0.77 & 1.24 & 0.59 / 0.79 & 1.56 & 0.39 / 0.61 & 1.34 & 0.52 / 0.74 & 1.53 & 0.46 / 0.61 & \textbf{1.10} & 0.51 / 0.73 \\
    \bottomrule
  \end{tabular}%
  }
\end{table}

\clearpage

\subsubsection{Full conformal prediction baseline results - Bench10\_100k}

\begin{table}[h]
  \caption{CP methods at 80\% target coverage on Gift-Eval Bench10 with the 100k test set (10 datasets, 2 backbones: TiRex and ARIMA). Cells show nWink (mean Winkler $/$ std$(y)$) and nW (mean width $/$ std$(y)$) $/$ Cov, where std$(y)$ is the standard deviation of the target on the eval split. ResCP and HopCPT report mean$\,\pm\,$std over ten seeds, RareCP over three seeds, other methods are deterministic. \textbf{Bold} marks the smallest displayed nWink per row among methods with $\mathrm{Cov} \ge 0.70$. Each backbone block ends with an \textit{Average} row over its 10 datasets.}
  \label{tab:cp_methods_bench10_100k}
  \centering
  \resizebox{\textwidth}{!}{%
  \begin{tabular}{lcccccccccccccccc}
    \multicolumn{17}{c}{{\large\bfseries\itshape Bench10\_100k}} \\[2pt]
    \toprule
    & \multicolumn{2}{c}{Uniform} & \multicolumn{2}{c}{ACI} & \multicolumn{2}{c}{DtACI} & \multicolumn{2}{c}{NexCP} & \multicolumn{2}{c}{KOWCPI} & \multicolumn{2}{c}{HopCPT} & \multicolumn{2}{c}{ResCP} & \multicolumn{2}{c}{RareCP\,\textit{(Ours)}} \\
    \cmidrule(lr){2-3} \cmidrule(lr){4-5} \cmidrule(lr){6-7} \cmidrule(lr){8-9} \cmidrule(lr){10-11} \cmidrule(lr){12-13} \cmidrule(lr){14-15} \cmidrule(lr){16-17}
    \textbf{Dataset} & nWink $\downarrow$ & nW $\downarrow$ / Cov $\uparrow$ & nWink $\downarrow$ & nW $\downarrow$ / Cov $\uparrow$ & nWink $\downarrow$ & nW $\downarrow$ / Cov $\uparrow$ & nWink $\downarrow$ & nW $\downarrow$ / Cov $\uparrow$ & nWink $\downarrow$ & nW $\downarrow$ / Cov $\uparrow$ & nWink $\downarrow$ & nW $\downarrow$ / Cov $\uparrow$ & nWink $\downarrow$ & nW $\downarrow$ / Cov $\uparrow$ & nWink $\downarrow$ & nW $\downarrow$ / Cov $\uparrow$ \\
    \midrule
    \multicolumn{17}{c}{\textit{TiRex backbone}} \\
    \addlinespace[2pt]
    electricity\_H & 0.25 & 0.11 / 0.77 & 0.21 & 0.12 / 0.80 & 0.26 & 0.17 / 0.79 & 0.21 & 0.11 / 0.80 & 0.26 & 0.06 / 0.66 & 0.24\,{\scriptsize $\pm$\,0.00} & 0.11 / 0.77\,{\scriptsize $\pm$\,0.00} & 0.25\,{\scriptsize $\pm$\,0.01} & 0.07 / 0.74\,{\scriptsize $\pm$\,0.04} & \textbf{0.19}\,{\scriptsize $\pm$\,0.00} & 0.12 / 0.80\,{\scriptsize $\pm$\,0.00} \\
    electricity\_15T & 0.29 & 0.13 / 0.78 & 0.25 & 0.14 / 0.80 & 0.31 & 0.20 / 0.79 & 0.25 & 0.13 / 0.80 & 0.30 & 0.08 / 0.64 & 0.28\,{\scriptsize $\pm$\,0.00} & 0.13 / 0.77\,{\scriptsize $\pm$\,0.00} & 0.29\,{\scriptsize $\pm$\,0.01} & 0.09 / 0.72\,{\scriptsize $\pm$\,0.06} & \textbf{0.22}\,{\scriptsize $\pm$\,0.00} & 0.14 / 0.80\,{\scriptsize $\pm$\,0.00} \\
    solar\_H & 0.71 & 0.23 / 0.80 & 0.70 & 0.25 / 0.80 & 0.68 & 0.28 / 0.79 & 0.70 & 0.25 / 0.80 & 0.70 & 0.13 / 0.61 & 0.54\,{\scriptsize $\pm$\,0.00} & 0.24 / 0.78\,{\scriptsize $\pm$\,0.01} & 0.49\,{\scriptsize $\pm$\,0.00} & 0.27 / 0.84\,{\scriptsize $\pm$\,0.00} & \textbf{0.39}\,{\scriptsize $\pm$\,0.00} & 0.24 / 0.80\,{\scriptsize $\pm$\,0.00} \\
    LOOP\_SEATTLE\_5T & 1.52 & 0.86 / 0.80 & 1.49 & 0.88 / 0.80 & 1.51 & 0.97 / 0.79 & 1.50 & 0.87 / 0.80 & 1.57 & 0.84 / 0.77 & 1.53\,{\scriptsize $\pm$\,0.00} & 0.87 / 0.80\,{\scriptsize $\pm$\,0.00} & 1.48\,{\scriptsize $\pm$\,0.01} & 1.00 / 0.85\,{\scriptsize $\pm$\,0.00} & \textbf{1.30}\,{\scriptsize $\pm$\,0.00} & 0.88 / 0.80\,{\scriptsize $\pm$\,0.00} \\
    SZ\_TAXI\_15T & 1.14 & 0.61 / 0.78 & 1.01 & 0.62 / 0.80 & 1.15 & 0.75 / 0.79 & \textbf{1.00} & 0.59 / 0.80 & 1.18 & 0.34 / 0.61 & 1.08\,{\scriptsize $\pm$\,0.01} & 0.60 / 0.76\,{\scriptsize $\pm$\,0.00} & 1.05\,{\scriptsize $\pm$\,0.01} & 0.63 / 0.82\,{\scriptsize $\pm$\,0.01} & 1.04\,{\scriptsize $\pm$\,0.00} & 0.70 / 0.80\,{\scriptsize $\pm$\,0.00} \\
    M\_DENSE\_H & 0.96 & 0.49 / 0.79 & 0.92 & 0.50 / 0.80 & 0.99 & 0.58 / 0.79 & 0.92 & 0.50 / 0.80 & 1.02 & 0.20 / 0.54 & 0.95\,{\scriptsize $\pm$\,0.03} & 0.57 / 0.83\,{\scriptsize $\pm$\,0.03} & 0.90\,{\scriptsize $\pm$\,0.01} & 0.50 / 0.79\,{\scriptsize $\pm$\,0.01} & \textbf{0.62}\,{\scriptsize $\pm$\,0.00} & 0.43 / 0.80\,{\scriptsize $\pm$\,0.00} \\
    temperature\_rain & 1.03 & 0.48 / 0.80 & 1.04 & 0.48 / 0.80 & 1.06 & 0.53 / 0.80 & 1.02 & 0.48 / 0.81 & 0.88 & 0.18 / 0.45 & 1.03\,{\scriptsize $\pm$\,0.00} & 0.49 / 0.80\,{\scriptsize $\pm$\,0.00} & 0.90\,{\scriptsize $\pm$\,0.01} & 0.46 / 0.84\,{\scriptsize $\pm$\,0.00} & \textbf{0.71}\,{\scriptsize $\pm$\,0.00} & 0.37 / 0.80\,{\scriptsize $\pm$\,0.00} \\
    m4\_hourly & 0.63 & 0.22 / 0.73 & 0.59 & 0.29 / 0.80 & 0.63 & 0.36 / 0.76 & 0.59 & 0.26 / 0.79 & 0.70 & 0.14 / 0.57 & 0.63\,{\scriptsize $\pm$\,0.00} & 0.24 / 0.75\,{\scriptsize $\pm$\,0.00} & 0.64\,{\scriptsize $\pm$\,0.02} & 0.19 / 0.67\,{\scriptsize $\pm$\,0.01} & \textbf{0.53}\,{\scriptsize $\pm$\,0.00} & 0.32 / 0.80\,{\scriptsize $\pm$\,0.00} \\
    m4\_daily & 0.06 & 0.03 / 0.81 & 0.06 & 0.03 / 0.80 & 0.06 & 0.03 / 0.79 & 0.06 & 0.03 / 0.80 & 0.06 & 0.02 / 0.67 & 0.06\,{\scriptsize $\pm$\,0.00} & 0.03 / 0.79\,{\scriptsize $\pm$\,0.01} & 0.05\,{\scriptsize $\pm$\,0.00} & 0.03 / 0.78\,{\scriptsize $\pm$\,0.04} & \textbf{0.04}\,{\scriptsize $\pm$\,0.00} & 0.03 / 0.80\,{\scriptsize $\pm$\,0.00} \\
    m4\_weekly & \textbf{0.14} & 0.06 / 0.78 & \textbf{0.14} & 0.06 / 0.79 & \textbf{0.14} & 0.06 / 0.77 & \textbf{0.14} & 0.06 / 0.79 & 0.15 & 0.06 / 0.74 & 0.16\,{\scriptsize $\pm$\,0.00} & 0.09 / 0.78\,{\scriptsize $\pm$\,0.01} & 0.15\,{\scriptsize $\pm$\,0.00} & 0.07 / 0.71\,{\scriptsize $\pm$\,0.02} & \textbf{0.14}\,{\scriptsize $\pm$\,0.00} & 0.08 / 0.78\,{\scriptsize $\pm$\,0.00} \\
    \addlinespace[2pt]
    \cmidrule(lr){1-17}
    \textit{Average} & 0.67 & 0.32 / 0.78 & 0.64 & 0.34 / 0.80 & 0.68 & 0.39 / 0.79 & 0.64 & 0.33 / 0.80 & 0.68 & 0.21 / 0.63 & 0.65 & 0.34 / 0.78 & 0.62 & 0.33 / 0.78 & \textbf{0.52} & 0.33 / 0.80 \\
    \midrule
    \multicolumn{17}{c}{\textit{ARIMA backbone}} \\
    \addlinespace[2pt]
    electricity\_H & 0.53 & 0.26 / 0.77 & 0.44 & 0.28 / 0.80 & 0.55 & 0.36 / 0.79 & 0.43 & 0.26 / 0.80 & 0.43 & 0.17 / 0.72 & 0.48\,{\scriptsize $\pm$\,0.00} & 0.25 / 0.77\,{\scriptsize $\pm$\,0.00} & 0.48\,{\scriptsize $\pm$\,0.03} & 0.17 / 0.73\,{\scriptsize $\pm$\,0.01} & \textbf{0.22}\,{\scriptsize $\pm$\,0.00} & 0.15 / 0.80\,{\scriptsize $\pm$\,0.00} \\
    electricity\_15T & 0.33 & 0.14 / 0.78 & 0.28 & 0.16 / 0.80 & 0.34 & 0.21 / 0.79 & 0.28 & 0.15 / 0.80 & 0.28 & 0.10 / 0.70 & 0.31\,{\scriptsize $\pm$\,0.00} & 0.15 / 0.77\,{\scriptsize $\pm$\,0.00} & 0.33\,{\scriptsize $\pm$\,0.01} & 0.10 / 0.73\,{\scriptsize $\pm$\,0.02} & \textbf{0.23}\,{\scriptsize $\pm$\,0.00} & 0.15 / 0.80\,{\scriptsize $\pm$\,0.00} \\
    solar\_H & 1.21 & 0.60 / 0.80 & 1.19 & 0.63 / 0.80 & 1.23 & 0.65 / 0.80 & 1.19 & 0.62 / 0.80 & 0.98 & 0.39 / 0.76 & 1.19\,{\scriptsize $\pm$\,0.01} & 0.63 / 0.79\,{\scriptsize $\pm$\,0.00} & 0.85\,{\scriptsize $\pm$\,0.01} & 0.51 / 0.84\,{\scriptsize $\pm$\,0.00} & \textbf{0.46}\,{\scriptsize $\pm$\,0.00} & 0.25 / 0.80\,{\scriptsize $\pm$\,0.00} \\
    LOOP\_SEATTLE\_5T & 1.56 & 0.90 / 0.80 & 1.53 & 0.91 / 0.80 & 1.54 & 0.99 / 0.79 & 1.54 & 0.91 / 0.80 & 1.47 & 0.79 / 0.75 & 1.60\,{\scriptsize $\pm$\,0.04} & 0.94 / 0.80\,{\scriptsize $\pm$\,0.01} & 1.51\,{\scriptsize $\pm$\,0.01} & 1.04 / 0.84\,{\scriptsize $\pm$\,0.00} & \textbf{1.31}\,{\scriptsize $\pm$\,0.00} & 0.89 / 0.80\,{\scriptsize $\pm$\,0.00} \\
    SZ\_TAXI\_15T & 1.16 & 0.64 / 0.78 & 1.03 & 0.63 / 0.80 & 1.15 & 0.75 / 0.80 & \textbf{1.02} & 0.62 / 0.80 & 1.07 & 0.45 / 0.69 & 1.15\,{\scriptsize $\pm$\,0.01} & 0.62 / 0.77\,{\scriptsize $\pm$\,0.00} & 1.04\,{\scriptsize $\pm$\,0.00} & 0.64 / 0.78\,{\scriptsize $\pm$\,0.01} & 1.06\,{\scriptsize $\pm$\,0.00} & 0.74 / 0.80\,{\scriptsize $\pm$\,0.00} \\
    M\_DENSE\_H & 1.34 & 0.72 / 0.79 & 1.26 & 0.73 / 0.80 & 1.38 & 0.84 / 0.80 & 1.27 & 0.72 / 0.81 & 1.25 & 0.53 / 0.69 & 1.31\,{\scriptsize $\pm$\,0.00} & 0.72 / 0.79\,{\scriptsize $\pm$\,0.00} & 1.19\,{\scriptsize $\pm$\,0.01} & 0.68 / 0.78\,{\scriptsize $\pm$\,0.00} & \textbf{0.66}\,{\scriptsize $\pm$\,0.00} & 0.46 / 0.80\,{\scriptsize $\pm$\,0.00} \\
    temperature\_rain & 1.15 & 0.52 / 0.80 & 1.11 & 0.55 / 0.80 & 1.12 & 0.56 / 0.79 & 1.10 & 0.53 / 0.81 & 1.09 & 0.43 / 0.76 & 1.17\,{\scriptsize $\pm$\,0.00} & 0.52 / 0.77\,{\scriptsize $\pm$\,0.01} & 1.09\,{\scriptsize $\pm$\,0.00} & 0.53 / 0.83\,{\scriptsize $\pm$\,0.00} & \textbf{0.81}\,{\scriptsize $\pm$\,0.02} & 0.45 / 0.80\,{\scriptsize $\pm$\,0.00} \\
    m4\_hourly & 0.95 & 0.33 / 0.73 & 0.85 & 0.43 / 0.80 & 0.93 & 0.53 / 0.76 & 0.88 & 0.40 / 0.79 & 0.88 & 0.26 / 0.61 & 0.98\,{\scriptsize $\pm$\,0.00} & 0.39 / 0.75\,{\scriptsize $\pm$\,0.00} & 0.99\,{\scriptsize $\pm$\,0.01} & 0.27 / 0.65\,{\scriptsize $\pm$\,0.02} & \textbf{0.67}\,{\scriptsize $\pm$\,0.00} & 0.43 / 0.80\,{\scriptsize $\pm$\,0.00} \\
    m4\_daily & 0.06 & 0.03 / 0.81 & 0.06 & 0.03 / 0.80 & 0.06 & 0.03 / 0.79 & 0.06 & 0.03 / 0.80 & 0.05 & 0.02 / 0.72 & 0.06\,{\scriptsize $\pm$\,0.01} & 0.03 / 0.74\,{\scriptsize $\pm$\,0.06} & 0.06\,{\scriptsize $\pm$\,0.00} & 0.03 / 0.76\,{\scriptsize $\pm$\,0.02} & \textbf{0.04}\,{\scriptsize $\pm$\,0.00} & 0.03 / 0.80\,{\scriptsize $\pm$\,0.00} \\
    m4\_weekly & 0.17 & 0.06 / 0.77 & 0.17 & 0.07 / 0.79 & 0.17 & 0.07 / 0.78 & 0.17 & 0.07 / 0.79 & 0.18 & 0.05 / 0.64 & 0.17\,{\scriptsize $\pm$\,0.00} & 0.08 / 0.79\,{\scriptsize $\pm$\,0.02} & 0.20\,{\scriptsize $\pm$\,0.01} & 0.05 / 0.48\,{\scriptsize $\pm$\,0.10} & \textbf{0.15}\,{\scriptsize $\pm$\,0.00} & 0.09 / 0.78\,{\scriptsize $\pm$\,0.00} \\
    \addlinespace[2pt]
    \cmidrule(lr){1-17}
    \textit{Average} & 0.85 & 0.42 / 0.78 & 0.79 & 0.44 / 0.80 & 0.85 & 0.50 / 0.79 & 0.79 & 0.43 / 0.80 & 0.77 & 0.32 / 0.70 & 0.84 & 0.43 / 0.77 & 0.77 & 0.40 / 0.74 & \textbf{0.56} & 0.36 / 0.80 \\
    \bottomrule
  \end{tabular}%
  }
\end{table}

\subsubsection{Full conformal prediction baseline results - Bench10\_full}

\begin{table}[h]
  \caption{CP methods at 80\% target coverage on Gift-Eval Bench10 with the full test set (10 datasets, 2 backbones: TiRex and ARIMA). Cells show nWink (mean Winkler $/$ std$(y)$) and nW (mean width $/$ std$(y)$) $/$ Cov, where std$(y)$ is the standard deviation of the target on the eval split. ResCP and HopCPT report mean$\,\pm\,$std over ten seeds, RareCP over three seeds, other methods are deterministic. \textbf{Bold} marks the smallest displayed nWink per row among methods with $\mathrm{Cov} \ge 0.70$. Each backbone block ends with an \textit{Average} row over its 10 datasets.}
  \label{tab:cp_methods_bench10_full}
  \centering
  \resizebox{\textwidth}{!}{%
  \begin{tabular}{lcccccccccccccccc}
    \multicolumn{17}{c}{{\large\bfseries\itshape Bench10\_full}} \\[2pt]
    \toprule
    & \multicolumn{2}{c}{Uniform} & \multicolumn{2}{c}{ACI} & \multicolumn{2}{c}{DtACI} & \multicolumn{2}{c}{NexCP} & \multicolumn{2}{c}{KOWCPI} & \multicolumn{2}{c}{HopCPT} & \multicolumn{2}{c}{ResCP} & \multicolumn{2}{c}{RareCP\,\textit{(Ours)}} \\
    \cmidrule(lr){2-3} \cmidrule(lr){4-5} \cmidrule(lr){6-7} \cmidrule(lr){8-9} \cmidrule(lr){10-11} \cmidrule(lr){12-13} \cmidrule(lr){14-15} \cmidrule(lr){16-17}
    \textbf{Dataset} & nWink $\downarrow$ & nW $\downarrow$ / Cov $\uparrow$ & nWink $\downarrow$ & nW $\downarrow$ / Cov $\uparrow$ & nWink $\downarrow$ & nW $\downarrow$ / Cov $\uparrow$ & nWink $\downarrow$ & nW $\downarrow$ / Cov $\uparrow$ & nWink $\downarrow$ & nW $\downarrow$ / Cov $\uparrow$ & nWink $\downarrow$ & nW $\downarrow$ / Cov $\uparrow$ & nWink $\downarrow$ & nW $\downarrow$ / Cov $\uparrow$ & nWink $\downarrow$ & nW $\downarrow$ / Cov $\uparrow$ \\
    \midrule
    \multicolumn{17}{c}{\textit{TiRex backbone}} \\
    \addlinespace[2pt]
    electricity\_H & 0.26 & 0.12 / 0.77 & 0.22 & 0.13 / 0.80 & 0.27 & 0.17 / 0.79 & 0.22 & 0.12 / 0.80 & 0.27 & 0.07 / 0.65 & 0.25\,{\scriptsize $\pm$\,0.00} & 0.12 / 0.77\,{\scriptsize $\pm$\,0.00} & 0.27\,{\scriptsize $\pm$\,0.01} & 0.08 / 0.72\,{\scriptsize $\pm$\,0.05} & \textbf{0.20}\,{\scriptsize $\pm$\,0.00} & 0.12 / 0.80\,{\scriptsize $\pm$\,0.00} \\
    electricity\_15T & 0.31 & 0.14 / 0.78 & 0.27 & 0.15 / 0.80 & 0.33 & 0.22 / 0.79 & 0.27 & 0.15 / 0.80 & 0.33 & 0.09 / 0.66 & 0.30\,{\scriptsize $\pm$\,0.00} & 0.15 / 0.77\,{\scriptsize $\pm$\,0.00} & 0.33\,{\scriptsize $\pm$\,0.01} & 0.10 / 0.70\,{\scriptsize $\pm$\,0.06} & \textbf{0.24}\,{\scriptsize $\pm$\,0.00} & 0.15 / 0.80\,{\scriptsize $\pm$\,0.00} \\
    solar\_H & 0.70 & 0.23 / 0.80 & 0.69 & 0.25 / 0.80 & 0.67 & 0.28 / 0.79 & 0.69 & 0.24 / 0.80 & 0.69 & 0.13 / 0.60 & 0.50\,{\scriptsize $\pm$\,0.01} & 0.25 / 0.79\,{\scriptsize $\pm$\,0.01} & 0.48\,{\scriptsize $\pm$\,0.00} & 0.26 / 0.84\,{\scriptsize $\pm$\,0.00} & \textbf{0.37}\,{\scriptsize $\pm$\,0.00} & 0.24 / 0.80\,{\scriptsize $\pm$\,0.00} \\
    LOOP\_SEATTLE\_5T & 1.54 & 0.88 / 0.80 & 1.51 & 0.89 / 0.80 & 1.56 & 1.01 / 0.79 & 1.52 & 0.89 / 0.80 & 1.65 & 0.77 / 0.70 & 1.58\,{\scriptsize $\pm$\,0.03} & 0.91 / 0.80\,{\scriptsize $\pm$\,0.01} & 1.51\,{\scriptsize $\pm$\,0.01} & 1.00 / 0.84\,{\scriptsize $\pm$\,0.00} & \textbf{1.31}\,{\scriptsize $\pm$\,0.00} & 0.90 / 0.80\,{\scriptsize $\pm$\,0.00} \\
    SZ\_TAXI\_15T & 1.14 & 0.61 / 0.78 & 1.01 & 0.62 / 0.80 & 1.15 & 0.75 / 0.79 & \textbf{1.00} & 0.59 / 0.80 & 1.18 & 0.34 / 0.61 & 1.08\,{\scriptsize $\pm$\,0.01} & 0.60 / 0.76\,{\scriptsize $\pm$\,0.00} & 1.05\,{\scriptsize $\pm$\,0.01} & 0.63 / 0.82\,{\scriptsize $\pm$\,0.01} & 1.03\,{\scriptsize $\pm$\,0.00} & 0.69 / 0.80\,{\scriptsize $\pm$\,0.00} \\
    M\_DENSE\_H & 0.96 & 0.49 / 0.79 & 0.92 & 0.50 / 0.80 & 0.99 & 0.58 / 0.79 & 0.92 & 0.50 / 0.80 & 1.02 & 0.20 / 0.54 & 0.95\,{\scriptsize $\pm$\,0.03} & 0.57 / 0.83\,{\scriptsize $\pm$\,0.03} & 0.90\,{\scriptsize $\pm$\,0.01} & 0.50 / 0.79\,{\scriptsize $\pm$\,0.01} & \textbf{0.61}\,{\scriptsize $\pm$\,0.00} & 0.43 / 0.80\,{\scriptsize $\pm$\,0.00} \\
    temperature\_rain & 1.03 & 0.48 / 0.80 & 1.04 & 0.48 / 0.80 & 1.06 & 0.53 / 0.80 & 1.02 & 0.48 / 0.81 & 0.88 & 0.18 / 0.45 & 1.04\,{\scriptsize $\pm$\,0.00} & 0.50 / 0.80\,{\scriptsize $\pm$\,0.01} & 0.90\,{\scriptsize $\pm$\,0.01} & 0.46 / 0.84\,{\scriptsize $\pm$\,0.00} & \textbf{0.71}\,{\scriptsize $\pm$\,0.01} & 0.37 / 0.80\,{\scriptsize $\pm$\,0.00} \\
    m4\_hourly & 0.63 & 0.22 / 0.73 & 0.59 & 0.29 / 0.80 & 0.63 & 0.36 / 0.76 & 0.59 & 0.26 / 0.79 & 0.70 & 0.14 / 0.57 & 0.63\,{\scriptsize $\pm$\,0.00} & 0.24 / 0.75\,{\scriptsize $\pm$\,0.00} & 0.64\,{\scriptsize $\pm$\,0.02} & 0.19 / 0.67\,{\scriptsize $\pm$\,0.01} & \textbf{0.53}\,{\scriptsize $\pm$\,0.00} & 0.32 / 0.80\,{\scriptsize $\pm$\,0.00} \\
    m4\_daily & 0.06 & 0.03 / 0.81 & 0.06 & 0.03 / 0.80 & 0.06 & 0.03 / 0.79 & 0.06 & 0.03 / 0.80 & 0.06 & 0.02 / 0.67 & 0.06\,{\scriptsize $\pm$\,0.00} & 0.03 / 0.79\,{\scriptsize $\pm$\,0.01} & 0.05\,{\scriptsize $\pm$\,0.00} & 0.03 / 0.78\,{\scriptsize $\pm$\,0.04} & \textbf{0.04}\,{\scriptsize $\pm$\,0.00} & 0.03 / 0.80\,{\scriptsize $\pm$\,0.00} \\
    m4\_weekly & 0.14 & 0.06 / 0.78 & 0.14 & 0.06 / 0.79 & 0.14 & 0.06 / 0.77 & 0.14 & 0.06 / 0.79 & 0.15 & 0.06 / 0.74 & 0.16\,{\scriptsize $\pm$\,0.00} & 0.09 / 0.78\,{\scriptsize $\pm$\,0.01} & 0.15\,{\scriptsize $\pm$\,0.00} & 0.07 / 0.71\,{\scriptsize $\pm$\,0.02} & \textbf{0.13}\,{\scriptsize $\pm$\,0.00} & 0.07 / 0.78\,{\scriptsize $\pm$\,0.00} \\
    \addlinespace[2pt]
    \cmidrule(lr){1-17}
    \textit{Average} & 0.68 & 0.33 / 0.78 & 0.64 & 0.34 / 0.80 & 0.69 & 0.40 / 0.79 & 0.64 & 0.33 / 0.80 & 0.69 & 0.20 / 0.62 & 0.66 & 0.35 / 0.78 & 0.63 & 0.33 / 0.77 & \textbf{0.52} & 0.33 / 0.80 \\
    \midrule
    \multicolumn{17}{c}{\textit{ARIMA backbone}} \\
    \addlinespace[2pt]
    electricity\_H & 0.51 & 0.26 / 0.78 & 0.43 & 0.27 / 0.80 & 0.53 & 0.36 / 0.79 & 0.42 & 0.26 / 0.80 & 0.43 & 0.17 / 0.70 & 0.47\,{\scriptsize $\pm$\,0.00} & 0.26 / 0.77\,{\scriptsize $\pm$\,0.00} & 0.49\,{\scriptsize $\pm$\,0.02} & 0.16 / 0.70\,{\scriptsize $\pm$\,0.02} & \textbf{0.23}\,{\scriptsize $\pm$\,0.00} & 0.16 / 0.80\,{\scriptsize $\pm$\,0.00} \\
    electricity\_15T & 0.35 & 0.16 / 0.78 & 0.30 & 0.17 / 0.80 & 0.36 & 0.23 / 0.79 & 0.30 & 0.16 / 0.80 & 0.30 & 0.11 / 0.69 & 0.33\,{\scriptsize $\pm$\,0.00} & 0.16 / 0.77\,{\scriptsize $\pm$\,0.00} & 0.36\,{\scriptsize $\pm$\,0.01} & 0.10 / 0.70\,{\scriptsize $\pm$\,0.01} & \textbf{0.24}\,{\scriptsize $\pm$\,0.00} & 0.16 / 0.80\,{\scriptsize $\pm$\,0.00} \\
    solar\_H & 1.20 & 0.60 / 0.80 & 1.19 & 0.63 / 0.80 & 1.22 & 0.64 / 0.80 & 1.18 & 0.62 / 0.80 & 0.97 & 0.38 / 0.75 & 1.18\,{\scriptsize $\pm$\,0.01} & 0.63 / 0.79\,{\scriptsize $\pm$\,0.01} & 0.85\,{\scriptsize $\pm$\,0.01} & 0.50 / 0.84\,{\scriptsize $\pm$\,0.00} & \textbf{0.45}\,{\scriptsize $\pm$\,0.00} & 0.24 / 0.80\,{\scriptsize $\pm$\,0.00} \\
    LOOP\_SEATTLE\_5T & 1.58 & 0.92 / 0.80 & 1.55 & 0.93 / 0.80 & 1.59 & 1.04 / 0.79 & 1.56 & 0.93 / 0.80 & 1.51 & 0.78 / 0.73 & 1.66\,{\scriptsize $\pm$\,0.04} & 0.97 / 0.79\,{\scriptsize $\pm$\,0.01} & 1.54\,{\scriptsize $\pm$\,0.01} & 1.04 / 0.84\,{\scriptsize $\pm$\,0.00} & \textbf{1.33}\,{\scriptsize $\pm$\,0.00} & 0.92 / 0.80\,{\scriptsize $\pm$\,0.00} \\
    SZ\_TAXI\_15T & 1.16 & 0.64 / 0.78 & 1.03 & 0.63 / 0.80 & 1.15 & 0.75 / 0.80 & \textbf{1.02} & 0.62 / 0.80 & 1.07 & 0.45 / 0.69 & 1.15\,{\scriptsize $\pm$\,0.01} & 0.62 / 0.77\,{\scriptsize $\pm$\,0.00} & 1.04\,{\scriptsize $\pm$\,0.00} & 0.64 / 0.78\,{\scriptsize $\pm$\,0.01} & 1.07\,{\scriptsize $\pm$\,0.01} & 0.74 / 0.80\,{\scriptsize $\pm$\,0.00} \\
    M\_DENSE\_H & 1.34 & 0.72 / 0.79 & 1.26 & 0.73 / 0.80 & 1.38 & 0.84 / 0.80 & 1.27 & 0.72 / 0.81 & 1.25 & 0.53 / 0.69 & 1.31\,{\scriptsize $\pm$\,0.01} & 0.72 / 0.79\,{\scriptsize $\pm$\,0.00} & 1.19\,{\scriptsize $\pm$\,0.01} & 0.68 / 0.78\,{\scriptsize $\pm$\,0.00} & \textbf{0.66}\,{\scriptsize $\pm$\,0.00} & 0.46 / 0.80\,{\scriptsize $\pm$\,0.00} \\
    temperature\_rain & 1.15 & 0.52 / 0.80 & 1.11 & 0.55 / 0.80 & 1.12 & 0.56 / 0.79 & 1.10 & 0.53 / 0.81 & 1.09 & 0.43 / 0.76 & 1.17\,{\scriptsize $\pm$\,0.01} & 0.53 / 0.78\,{\scriptsize $\pm$\,0.01} & 1.09\,{\scriptsize $\pm$\,0.00} & 0.53 / 0.83\,{\scriptsize $\pm$\,0.00} & \textbf{0.82}\,{\scriptsize $\pm$\,0.02} & 0.43 / 0.80\,{\scriptsize $\pm$\,0.00} \\
    m4\_hourly & 0.95 & 0.33 / 0.73 & 0.85 & 0.43 / 0.80 & 0.93 & 0.53 / 0.76 & 0.88 & 0.40 / 0.79 & 0.88 & 0.26 / 0.61 & 0.98\,{\scriptsize $\pm$\,0.00} & 0.39 / 0.75\,{\scriptsize $\pm$\,0.00} & 0.99\,{\scriptsize $\pm$\,0.01} & 0.27 / 0.65\,{\scriptsize $\pm$\,0.02} & \textbf{0.69}\,{\scriptsize $\pm$\,0.00} & 0.44 / 0.80\,{\scriptsize $\pm$\,0.00} \\
    m4\_daily & 0.06 & 0.03 / 0.81 & 0.06 & 0.03 / 0.80 & 0.06 & 0.03 / 0.79 & 0.06 & 0.03 / 0.80 & 0.05 & 0.02 / 0.72 & 0.06\,{\scriptsize $\pm$\,0.01} & 0.03 / 0.74\,{\scriptsize $\pm$\,0.07} & 0.06\,{\scriptsize $\pm$\,0.00} & 0.03 / 0.76\,{\scriptsize $\pm$\,0.02} & \textbf{0.04}\,{\scriptsize $\pm$\,0.00} & 0.03 / 0.80\,{\scriptsize $\pm$\,0.00} \\
    m4\_weekly & 0.17 & 0.06 / 0.77 & 0.17 & 0.07 / 0.79 & 0.17 & 0.07 / 0.78 & 0.17 & 0.07 / 0.79 & 0.18 & 0.05 / 0.64 & 0.17\,{\scriptsize $\pm$\,0.00} & 0.08 / 0.79\,{\scriptsize $\pm$\,0.01} & 0.20\,{\scriptsize $\pm$\,0.01} & 0.05 / 0.48\,{\scriptsize $\pm$\,0.10} & \textbf{0.16}\,{\scriptsize $\pm$\,0.00} & 0.09 / 0.78\,{\scriptsize $\pm$\,0.00} \\
    \addlinespace[2pt]
    \cmidrule(lr){1-17}
    \textit{Average} & 0.85 & 0.42 / 0.78 & 0.79 & 0.44 / 0.80 & 0.85 & 0.51 / 0.79 & 0.80 & 0.43 / 0.80 & 0.77 & 0.32 / 0.70 & 0.85 & 0.44 / 0.77 & 0.78 & 0.40 / 0.74 & \textbf{0.57} & 0.37 / 0.80 \\
    \bottomrule
  \end{tabular}%
  }
\end{table}

\clearpage

\subsubsection{Full foundation model results - Bench10 \& Bench22}

\begin{table}[h]
  \caption{Native UQ of foundation model forecasters at 80\% target coverage on Gift-Eval Bench10 (top, 10 datasets) and Bench22 (bottom, 12 additional datasets). Cells show nWink (mean Winkler $/$ std$(y)$) and nW (mean width $/$ std$(y)$) $/$ Cov, where std$(y)$ is the standard deviation of the target on the eval split. Toto and Lag-Llama report mean$\,\pm\,$std over ten seeds, other FMs are deterministic. \textbf{Bold} marks the smallest displayed nWink per row among FMs with $\mathrm{Cov} \ge 0.70$. Each block ends with an \textit{Average} row (Bench10 over 10 datasets, Bench22 over the full 22).}
  \label{tab:fm_native_uq_bench10_bench22}
  \centering
  \resizebox{\textwidth}{!}{%
  \begin{tabular}{lcccccccccccccccc}
    \toprule
    & \multicolumn{2}{c}{TiRex} & \multicolumn{2}{c}{Chronos-2} & \multicolumn{2}{c}{ChronosBolt} & \multicolumn{2}{c}{Moirai-2} & \multicolumn{2}{c}{TimesFM-2.5} & \multicolumn{2}{c}{TabPFN-TS} & \multicolumn{2}{c}{Toto} & \multicolumn{2}{c}{Lag-Llama} \\
    \cmidrule(lr){2-3} \cmidrule(lr){4-5} \cmidrule(lr){6-7} \cmidrule(lr){8-9} \cmidrule(lr){10-11} \cmidrule(lr){12-13} \cmidrule(lr){14-15} \cmidrule(lr){16-17}
    \textbf{Dataset} & nWink $\downarrow$ & nW $\downarrow$ / Cov $\uparrow$ & nWink $\downarrow$ & nW $\downarrow$ / Cov $\uparrow$ & nWink $\downarrow$ & nW $\downarrow$ / Cov $\uparrow$ & nWink $\downarrow$ & nW $\downarrow$ / Cov $\uparrow$ & nWink $\downarrow$ & nW $\downarrow$ / Cov $\uparrow$ & nWink $\downarrow$ & nW $\downarrow$ / Cov $\uparrow$ & nWink $\downarrow$ & nW $\downarrow$ / Cov $\uparrow$ & nWink $\downarrow$ & nW $\downarrow$ / Cov $\uparrow$ \\
    \midrule
    \multicolumn{17}{c}{\textit{bench10 (10 datasets)}} \\
    \addlinespace[2pt]
    electricity\_H & \textbf{0.20} & 0.13 / 0.81 & 0.22 & 0.14 / 0.80 & 0.22 & 0.15 / 0.85 & 0.23 & 0.15 / 0.81 & 0.26 & 0.18 / 0.82 & 0.24 & 0.14 / 0.78 & 0.29\,{\scriptsize $\pm$\,0.00} & 0.17 / 0.72\,{\scriptsize $\pm$\,0.00} & 0.51\,{\scriptsize $\pm$\,0.00} & 0.24 / 0.71\,{\scriptsize $\pm$\,0.00} \\
    electricity\_15T & \textbf{0.27} & 0.17 / 0.81 & 0.30 & 0.16 / 0.74 & 0.28 & 0.19 / 0.83 & 0.29 & 0.17 / 0.79 & 0.32 & 0.21 / 0.85 & 0.32 & 0.17 / 0.76 & 0.33\,{\scriptsize $\pm$\,0.00} & 0.15 / 0.65\,{\scriptsize $\pm$\,0.00} & 0.60\,{\scriptsize $\pm$\,0.00} & 0.29 / 0.62\,{\scriptsize $\pm$\,0.00} \\
    solar\_H & \textbf{0.43} & 0.30 / 0.82 & 0.51 & 0.36 / 0.92 & 0.51 & 0.39 / 0.94 & 0.66 & 0.44 / 0.90 & 0.78 & 0.55 / 0.88 & 0.65 & 0.36 / 0.89 & 0.60\,{\scriptsize $\pm$\,0.00} & 0.25 / 0.83\,{\scriptsize $\pm$\,0.00} & 0.88\,{\scriptsize $\pm$\,0.00} & 0.43 / 0.85\,{\scriptsize $\pm$\,0.00} \\
    LOOP\_SEATTLE\_5T & \textbf{2.12} & 1.45 / 0.79 & 2.13 & 1.40 / 0.77 & 2.17 & 1.41 / 0.78 & \textbf{2.12} & 1.40 / 0.78 & 2.14 & 1.40 / 0.78 & 2.19 & 1.36 / 0.75 & 2.16\,{\scriptsize $\pm$\,0.00} & 1.47 / 0.79\,{\scriptsize $\pm$\,0.00} & 3.05\,{\scriptsize $\pm$\,0.00} & 0.78 / 0.42\,{\scriptsize $\pm$\,0.00} \\
    SZ\_TAXI\_15T & \textbf{0.76} & 0.51 / 0.82 & 0.77 & 0.49 / 0.80 & 0.78 & 0.50 / 0.81 & \textbf{0.76} & 0.49 / 0.81 & 0.77 & 0.48 / 0.81 & 0.78 & 0.47 / 0.79 & 0.77\,{\scriptsize $\pm$\,0.00} & 0.50 / 0.80\,{\scriptsize $\pm$\,0.00} & 0.92\,{\scriptsize $\pm$\,0.00} & 0.39 / 0.69\,{\scriptsize $\pm$\,0.00} \\
    M\_DENSE\_H & \textbf{0.75} & 0.51 / 0.79 & 0.79 & 0.43 / 0.73 & 0.77 & 0.55 / 0.84 & 0.83 & 0.50 / 0.78 & 0.85 & 0.57 / 0.83 & 0.91 & 0.50 / 0.77 & 0.95\,{\scriptsize $\pm$\,0.00} & 0.46 / 0.74\,{\scriptsize $\pm$\,0.00} & 1.11\,{\scriptsize $\pm$\,0.00} & 0.43 / 0.67\,{\scriptsize $\pm$\,0.00} \\
    temperature\_rain & 0.75 & 0.39 / 0.78 & \textbf{0.74} & 0.38 / 0.78 & \textbf{0.74} & 0.37 / 0.80 & \textbf{0.74} & 0.36 / 0.77 & 0.75 & 0.35 / 0.81 & 0.78 & 0.34 / 0.77 & 0.76\,{\scriptsize $\pm$\,0.00} & 0.36 / 0.78\,{\scriptsize $\pm$\,0.00} & 0.86\,{\scriptsize $\pm$\,0.00} & 0.24 / 0.56\,{\scriptsize $\pm$\,0.00} \\
    m4\_hourly & \textbf{0.45} & 0.31 / 0.79 & 0.72 & 0.31 / 0.52 & 0.67 & 0.36 / 0.61 & 0.50 & 0.37 / 0.83 & 0.51 & 0.37 / 0.84 & 0.55 & 0.32 / 0.76 & 1.07\,{\scriptsize $\pm$\,0.00} & 0.86 / 0.90\,{\scriptsize $\pm$\,0.00} & 2.56\,{\scriptsize $\pm$\,0.01} & 1.06 / 0.24\,{\scriptsize $\pm$\,0.00} \\
    m4\_daily & \textbf{0.05} & 0.03 / 0.79 & \textbf{0.05} & 0.04 / 0.76 & \textbf{0.05} & 0.04 / 0.79 & \textbf{0.05} & 0.04 / 0.83 & \textbf{0.05} & 0.04 / 0.84 & \textbf{0.05} & 0.03 / 0.77 & \textbf{0.05}\,{\scriptsize $\pm$\,0.00} & 0.03 / 0.71\,{\scriptsize $\pm$\,0.00} & 0.44\,{\scriptsize $\pm$\,0.00} & 0.40 / 0.69\,{\scriptsize $\pm$\,0.00} \\
    m4\_weekly & \textbf{0.11} & 0.07 / 0.77 & 0.12 & 0.07 / 0.73 & 0.12 & 0.09 / 0.85 & 0.13 & 0.08 / 0.75 & 0.12 & 0.08 / 0.81 & 0.13 & 0.07 / 0.71 & 0.14\,{\scriptsize $\pm$\,0.00} & 0.07 / 0.64\,{\scriptsize $\pm$\,0.01} & 0.34\,{\scriptsize $\pm$\,0.00} & 0.27 / 0.74\,{\scriptsize $\pm$\,0.00} \\
    \addlinespace[2pt]
    \cmidrule(lr){1-17}
    \textit{Bench10 average} & \textbf{0.59} & 0.39 / 0.80 & 0.64 & 0.38 / 0.76 & 0.63 & 0.40 / 0.81 & 0.63 & 0.40 / 0.80 & 0.66 & 0.42 / 0.83 & 0.66 & 0.38 / 0.78 & 0.71 & 0.43 / 0.76 & 1.13 & 0.45 / 0.62 \\
    \midrule
    \multicolumn{17}{c}{\textit{bench22 — additional datasets (12)}} \\
    \addlinespace[2pt]
    solar\_10T & \textbf{0.29} & 0.18 / 0.87 & 0.37 & 0.26 / 0.64 & 0.36 & 0.27 / 0.73 & 0.39 & 0.23 / 0.87 & 0.36 & 0.21 / 0.86 & 0.38 & 0.21 / 0.89 & 0.32\,{\scriptsize $\pm$\,0.00} & 0.16 / 0.88\,{\scriptsize $\pm$\,0.00} & 1.73\,{\scriptsize $\pm$\,0.00} & 0.16 / 0.77\,{\scriptsize $\pm$\,0.00} \\
    LOOP\_SEATTLE\_H & \textbf{1.43} & 0.88 / 0.78 & \textbf{1.43} & 0.83 / 0.75 & 1.50 & 0.98 / 0.82 & 1.48 & 0.92 / 0.80 & 1.44 & 0.90 / 0.81 & 1.48 & 0.89 / 0.75 & 1.54\,{\scriptsize $\pm$\,0.00} & 0.81 / 0.71\,{\scriptsize $\pm$\,0.00} & 1.79\,{\scriptsize $\pm$\,0.00} & 0.61 / 0.53\,{\scriptsize $\pm$\,0.00} \\
    m4\_monthly & \textbf{0.10} & 0.07 / 0.80 & 0.11 & 0.07 / 0.77 & 0.12 & 0.08 / 0.83 & 0.11 & 0.07 / 0.79 & 0.11 & 0.07 / 0.84 & 0.11 & 0.07 / 0.76 & 0.11\,{\scriptsize $\pm$\,0.00} & 0.06 / 0.71\,{\scriptsize $\pm$\,0.00} & 0.38\,{\scriptsize $\pm$\,0.00} & 0.27 / 0.74\,{\scriptsize $\pm$\,0.00} \\
    electricity\_D & 0.17 & 0.10 / 0.79 & 0.17 & 0.09 / 0.75 & 0.18 & 0.11 / 0.80 & \textbf{0.16} & 0.10 / 0.81 & 0.17 & 0.10 / 0.82 & 0.18 & 0.10 / 0.75 & 0.18\,{\scriptsize $\pm$\,0.00} & 0.09 / 0.73\,{\scriptsize $\pm$\,0.00} & 0.28\,{\scriptsize $\pm$\,0.00} & 0.13 / 0.47\,{\scriptsize $\pm$\,0.00} \\
    SZ\_TAXI\_H & \textbf{0.70} & 0.47 / 0.80 & \textbf{0.70} & 0.46 / 0.78 & \textbf{0.70} & 0.47 / 0.80 & 0.71 & 0.47 / 0.79 & 0.72 & 0.44 / 0.78 & 0.73 & 0.44 / 0.75 & 0.74\,{\scriptsize $\pm$\,0.00} & 0.46 / 0.76\,{\scriptsize $\pm$\,0.00} & 0.90\,{\scriptsize $\pm$\,0.00} & 0.38 / 0.62\,{\scriptsize $\pm$\,0.00} \\
    m4\_quarterly & \textbf{0.07} & 0.06 / 0.85 & 0.09 & 0.07 / 0.76 & 0.11 & 0.09 / 0.92 & 0.10 & 0.07 / 0.65 & 0.11 & 0.09 / 0.79 & 0.08 & 0.04 / 0.67 & 0.08\,{\scriptsize $\pm$\,0.00} & 0.05 / 0.72\,{\scriptsize $\pm$\,0.02} & 0.28\,{\scriptsize $\pm$\,0.00} & 0.24 / 0.73\,{\scriptsize $\pm$\,0.02} \\
    saugeenday\_D & \textbf{0.72} & 0.44 / 0.81 & 0.84 & 0.48 / 0.81 & 0.93 & 0.48 / 0.84 & 0.87 & 0.50 / 0.85 & 0.75 & 0.44 / 0.85 & 0.85 & 0.55 / 0.79 & 0.89\,{\scriptsize $\pm$\,0.02} & 0.41 / 0.69\,{\scriptsize $\pm$\,0.01} & 1.14\,{\scriptsize $\pm$\,0.02} & 0.44 / 0.70\,{\scriptsize $\pm$\,0.01} \\
    M\_DENSE\_D & 0.75 & 0.40 / 0.83 & \textbf{0.67} & 0.30 / 0.79 & 0.73 & 0.43 / 0.88 & 0.69 & 0.33 / 0.82 & 0.69 & 0.34 / 0.82 & 0.75 & 0.31 / 0.74 & 0.77\,{\scriptsize $\pm$\,0.01} & 0.41 / 0.78\,{\scriptsize $\pm$\,0.00} & 1.32\,{\scriptsize $\pm$\,0.01} & 0.42 / 0.59\,{\scriptsize $\pm$\,0.01} \\
    us\_births\_D & 1.42 & 0.75 / 0.72 & 1.16 & 0.58 / 0.79 & 1.36 & 1.09 / 0.95 & \textbf{1.09} & 0.72 / 0.87 & 1.24 & 0.83 / 0.91 & 1.34 & 0.67 / 0.73 & 1.56\,{\scriptsize $\pm$\,0.01} & 0.75 / 0.72\,{\scriptsize $\pm$\,0.01} & 4.35\,{\scriptsize $\pm$\,0.03} & 1.17 / 0.66\,{\scriptsize $\pm$\,0.01} \\
    saugeenday\_W & 2.14 & 0.98 / 0.80 & 2.53 & 1.00 / 0.78 & 2.26 & 0.90 / 0.88 & \textbf{2.09} & 0.85 / 0.78 & 2.17 & 0.75 / 0.80 & 2.14 & 0.94 / 0.63 & 2.28\,{\scriptsize $\pm$\,0.05} & 0.89 / 0.65\,{\scriptsize $\pm$\,0.03} & 2.28\,{\scriptsize $\pm$\,0.04} & 0.63 / 0.70\,{\scriptsize $\pm$\,0.03} \\
    us\_births\_W & 2.29 & 1.42 / 0.76 & \textbf{2.14} & 1.37 / 0.79 & 2.44 & 1.48 / 0.72 & 1.39 & 0.65 / 0.55 & 2.41 & 1.46 / 0.76 & 2.70 & 1.36 / 0.66 & 2.62\,{\scriptsize $\pm$\,0.06} & 1.39 / 0.67\,{\scriptsize $\pm$\,0.03} & 3.37\,{\scriptsize $\pm$\,0.06} & 0.96 / 0.58\,{\scriptsize $\pm$\,0.04} \\
    saugeenday\_M & 2.39 & 1.50 / 0.82 & 1.59 & 1.33 / 0.77 & 1.81 & 1.32 / 0.73 & 1.67 & 1.15 / 0.64 & 2.65 & 1.13 / 0.82 & \textbf{1.57} & 1.24 / 0.77 & 1.84\,{\scriptsize $\pm$\,0.06} & 1.25 / 0.85\,{\scriptsize $\pm$\,0.03} & 3.22\,{\scriptsize $\pm$\,0.06} & 0.84 / 0.66\,{\scriptsize $\pm$\,0.02} \\
    \addlinespace[2pt]
    \cmidrule(lr){1-17}
    \textit{Bench22 average} & 0.83 & 0.50 / 0.80 & 0.83 & 0.48 / 0.76 & 0.85 & 0.53 / 0.82 & \textbf{0.78} & 0.46 / 0.78 & 0.88 & 0.50 / 0.83 & 0.86 & 0.48 / 0.76 & 0.91 & 0.50 / 0.75 & 1.47 & 0.49 / 0.64 \\
    \bottomrule
  \end{tabular}%
  }
\end{table}

\clearpage

\subsection{Hyperparameter settings}
\label{apx:hyperparams}

\subsubsection{ResCP baseline}

\begin{table}[h]
  \caption{ResCP hyperparameter configuration used by the CP-methods full-result tables. \emph{Top:} best values from Optuna TPE (100 trials per backbone, ranked by mean nWink on the Bench10 val cache); the same best values are reused on Bench22 and the 100k/full sets. The \textbf{Trial-0} column shows the paper midpoint defaults that seed the TPE search. \emph{Bottom:} parameters held fixed across backbones and benchmarks.}
  \label{tab:rescp_sampling_hpo}
  \centering
  \setlength{\tabcolsep}{4pt}
  \resizebox{\textwidth}{!}{%
  \begin{tabular}{l l c c c c c}
    \toprule
    \textbf{Hyperparameter} & \textbf{Search range} & \textbf{Trial-0} & \textbf{TiRex} & \textbf{TimesFM-2.5} & \textbf{TabPFN-TS} & \textbf{ARIMA} \\
    \midrule
    \multicolumn{7}{l}{\textit{Searched (Optuna TPE, 100 trials per backbone, Bench10 val cache)}} \\
    \addlinespace[2pt]
    spectral\_radius & $[0.5,\,1.5]$ & 0.95 & 0.99 & 0.89 & 1.08 & 0.99 \\
    leak & $[0.5,\,1.0]$ & 0.80 & 0.97 & 0.53 & 0.76 & 0.77 \\
    input\_scaling & $[0.1,\,3.0]$ & 0.40 & 0.86 & 2.81 & 2.56 & 2.55 \\
    temperature & $[0.01,\,2.0]$ (log) & 0.10 & 0.090 & 0.123 & 0.084 & 0.069 \\
    \midrule
    \multicolumn{7}{l}{\textit{Fixed across backbones and benchmarks}} \\
    \addlinespace[2pt]
    $n_{\text{units}}$ (reservoir size) & \multicolumn{6}{c}{512} \\
    connectivity & \multicolumn{6}{c}{0.2} \\
    circle (reservoir topology) & \multicolumn{6}{c}{False (random)} \\
    similarity & \multicolumn{6}{c}{cosine} \\
    time\_decay & \multicolumn{6}{c}{linear\_ramp} \\
    $\beta$ grid size & \multicolumn{6}{c}{100} \\
    horizon & \multicolumn{6}{c}{1} \\
    fifo\_match\_calib & \multicolumn{6}{c}{True} \\
    \bottomrule
  \end{tabular}%
  }
\end{table}

Table \ref{tab:rescp_sampling_hpo} reports the four searched parameters
that tune the echo state reservoir at the core of ResCP, plus the
fixed reservoir and pipeline settings below.
\texttt{spectral\_radius} is the spectral radius of the reservoir
recurrent matrix and the main lever on memory length; values above one
push the dynamics towards instability. \texttt{leak} is the rate of
the leaky integrator and controls how fast the reservoir state blends
new inputs with the previous state. \texttt{input\_scaling} is a
global gain on the inputs that shapes how strongly the reservoir
reacts to each step. \texttt{temperature} is the softmax temperature
applied to similarity scores when forming the weighted residual
quantiles. Small temperatures approach hard nearest neighbour
selection; large temperatures average broadly. 
The fixed block records the methodological identity of ResCP. The
reservoir has 512 units, connectivity 0.2, and a random rather than
circular topology. Reservoir states are compared with cosine
similarity and weighted by a linear ramp time decay. The quantile
head uses a 100 point $\beta$ grid. The pipeline forecasts one step
ahead and the calibration buffer is FIFO matched to the test stream.

\clearpage

\subsubsection{KOWCPI baseline}

\begin{table}[h]
  \caption{KOWCPI hyperparameter configuration used by the CP-methods full-result tables. \emph{Top:} best values from Optuna TPE (100 trials per backbone, ranked by mean nWink on the Bench10 val cache); the same best values are reused on Bench22 and the 100k/full sets. The \textbf{Trial-0} column shows the paper defaults that seed the TPE search. \emph{Bottom:} parameters held fixed across backbones and benchmarks.}
  \label{tab:kowcpi_v2_hpo}
  \centering
  \setlength{\tabcolsep}{4pt}
  \resizebox{\textwidth}{!}{%
  \begin{tabular}{l l c c c c c}
    \toprule
    \textbf{Hyperparameter} & \textbf{Search range} & \textbf{Trial-0} & \textbf{TiRex} & \textbf{TimesFM-2.5} & \textbf{TabPFN-TS} & \textbf{ARIMA} \\
    \midrule
    \multicolumn{7}{l}{\textit{Searched (Optuna TPE, 100 trials per backbone, Bench10 val cache)}} \\
    \addlinespace[2pt]
    $w$ (lag-window length) & $\{5,10,20,50,100\}$ & 50 & 50 & 10 & 5 & 5 \\
    kernel & $\{$epa, gauss$\}$ & epa & epa & epa & epa & epa \\
    bandwidth\_mode & $\{$median, aicc$\}$ & aicc & median & median & median & median \\
    bandwidth\_scale & $[0.25,\,4.0]$ (log) & 1.0 & 0.464 & 0.862 & 0.806 & 0.697 \\
    el\_adjust & $\{$False, True$\}$ & True & False & False & False & False \\
    \midrule
    \multicolumn{7}{l}{\textit{Fixed across backbones and benchmarks}} \\
    \addlinespace[2pt]
    $n_{\beta}$ (\,$\beta$ grid size) & \multicolumn{6}{c}{21} \\
    aicc\_grid\_points & \multicolumn{6}{c}{5} \\
    aicc\_grid\_low & \multicolumn{6}{c}{0.25} \\
    aicc\_grid\_high & \multicolumn{6}{c}{4.0} \\
    aicc\_refit\_every & \multicolumn{6}{c}{500} \\
    fifo\_match\_calib & \multicolumn{6}{c}{True} \\
    \bottomrule
  \end{tabular}%
  }
\end{table}

Table \ref{tab:kowcpi_v2_hpo} reports the five searched parameters of
the KOWCPI kernel CDF estimator and the fixed parameters below.
\texttt{$w$} is the lag window length: the number of recent time steps
that define the local context vector the kernel uses to find similar
past situations. The \texttt{kernel} switch chooses between the
Epanechnikov kernel (compact support, hard cutoff beyond the bandwidth)
and the Gaussian kernel (infinite support, exponential decay).
\texttt{bandwidth\_mode} selects between an AICc tuned bandwidth (data
driven, refit periodically over a small grid) and the median rule (a
fixed rule of thumb based on the median pairwise distance among recent
contexts). \texttt{bandwidth\_scale} multiplies the bandwidth that the
mode produces, so the effective kernel radius is
$h_{\text{eff}} = h_{\text{rule}} \times \texttt{bandwidth\_scale}$.
Values below one tighten the kernel below the rule's pick, sharpening
the local conditioning at the cost of higher variance; values above
one widen it and smooth across more past contexts. \texttt{el\_adjust}
toggles the empirical likelihood correction from the original KOWCPI
paper. When enabled, it reweights the empirical residual distribution
so that the implied conditional CDF matches the target coverage
exactly under the empirical likelihood constraint. This typically
widens intervals to protect coverage when the raw kernel CDF estimate
is over confident. The Trial-0 column lists the paper
defaults: $w = 50$, AICc bandwidth, unit scale, EL adjustment enabled.

All four backbones diverge from these defaults on three of the five
searched knobs. HPO selects \texttt{bandwidth\_mode} = median over
AICc, drives \texttt{bandwidth\_scale} below one ($0.46$ to $0.86$
across backbones), and turns the empirical likelihood adjustment off
on every backbone. The lag window also shrinks below the paper's
$w = 50$ for three of the four predictors. Only the kernel choice
agrees with the paper. The pattern is consistent: the paper defaults
oversmooth in our setting. The AICc bandwidth combined with the EL
adjustment produces effective neighbourhoods wider than the residual
structure of these foundation model forecasters supports, and HPO
prefers tighter local kernels (smaller $w$, smaller bandwidth scale,
no EL widening).

\clearpage

\subsubsection{HopCPT baseline}

\begin{table}[h]
  \caption{HopCPT hyperparameter configuration used by the CP-methods full-result tables. \emph{Top:} best values from Optuna TPE (100 trials per backbone, ranked by mean nWink on the Bench10 val cache); the same best values are reused on Bench22 and the 100k/full sets. \emph{Bottom:} parameters held fixed across backbones and benchmarks. Operational knobs (batch size, dropout, architectural toggles) are searched but omitted here for compactness.}
  \label{tab:hopcpt_hpo}
  \centering
  \setlength{\tabcolsep}{4pt}
  \resizebox{\textwidth}{!}{%
  \begin{tabular}{l l c c c c}
    \toprule
    \textbf{Hyperparameter} & \textbf{Search range} & \textbf{TiRex} & \textbf{TimesFM-2.5} & \textbf{TabPFN-TS} & \textbf{ARIMA} \\
    \midrule
    \multicolumn{6}{l}{\textit{Searched (Optuna TPE, 100 trials per backbone)}} \\
    \addlinespace[2pt]
    \multicolumn{6}{l}{\quad\textit{Hopfield retriever}} \\
    hopfield\_beta & $[0.5,\,24.0]$ (log) & 0.701 & 6.72 & 0.717 & 2.21 \\
    eps\_mem\_size & $\{200,400,800,1500,3000\}$ & 400 & 400 & 400 & 200 \\
    \addlinespace[2pt]
    \multicolumn{6}{l}{\quad\textit{Context encoder}} \\
    ctx\_mode & $\{$uni, uni\_pms, uni\_pms+$\hat{y}$, ms+$\hat{y}\}$ & uni\_pms+$\hat{y}$ & ms+$\hat{y}$ & uni\_pms+$\hat{y}$ & uni\_pms \\
    ctx\_past\_window & $\{1,2,4,8,16\}$ & 8 & 8 & 2 & 1 \\
    ctx\_encode\_hiddenL & $\{0,1,2\}$ & 2 & 1 & 0 & 0 \\
    \addlinespace[2pt]
    \multicolumn{6}{l}{\quad\textit{Training}} \\
    lr & $[2\!\cdot\!10^{-5},\,5\!\cdot\!10^{-3}]$ (log) & $2.11\!\cdot\!10^{-4}$ & $2.45\!\cdot\!10^{-4}$ & $6.38\!\cdot\!10^{-4}$ & $1.83\!\cdot\!10^{-3}$ \\
    train\_encoder\_epochs & $\{0,25,50,100,200,500\}$ & 500 & 200 & 100 & 100 \\
    \addlinespace[2pt]
    \multicolumn{6}{l}{\quad\textit{Online calibration}} \\
    calibration\_window\_size & $\{500,1000,2000,4000\}$ & 2000 & 2000 & 2000 & 2000 \\
    head\_per\_alpha & $\{$False, True$\}$ & True & False & False & True \\
    pos\_encode\_mode & $\{$none, rel-simple, rel-linear$\}$ & none & none & rel-linear & none \\
    \midrule
    \multicolumn{6}{l}{\textit{Fixed across backbones and benchmarks}} \\
    \addlinespace[2pt]
    predict\_abs\_eps & \multicolumn{5}{c}{True} \\
    online\_memory & \multicolumn{5}{c}{True} \\
    train\_on\_calibration & \multicolumn{5}{c}{False} \\
    context\_window & \multicolumn{5}{c}{64} \\
    \bottomrule
  \end{tabular}%
  }
\end{table}

Table \ref{tab:hopcpt_hpo} groups the searched parameters into four
functional blocks. The first block configures the Hopfield retriever
that fetches past residual patterns from memory.
\texttt{hopfield\_beta} sets the sharpness of the soft retrieval. Large
$\beta$ recovers hard nearest neighbour retrieval; small $\beta$
averages broadly across the memory set. \texttt{eps\_mem\_size} is the
size of that memory set. The second block configures the context
encoder that produces the retrieval keys. \texttt{ctx\_mode} selects
what the encoder consumes: \texttt{uni} uses past residuals only,
\texttt{uni\_pms} adds past multistep residuals,
\texttt{uni\_pms+}$\hat y$ adds the model forecasts on top, and
\texttt{ms+}$\hat y$ uses multistep forecasts only.
\texttt{ctx\_past\_window} fixes how many past time steps the encoder
consumes. \texttt{ctx\_encode\_hiddenL} sets its depth. The third block
contains the training settings. \texttt{lr} is the Adam learning rate.
\texttt{train\_encoder\_epochs} is how many epochs the encoder is
trained before online calibration begins. The fourth block configures
the online calibration loop. \texttt{calibration\_window\_size} sets
the rolling calibration window. \texttt{head\_per\_alpha} chooses
between a shared head and a separate head per target $\alpha$ level.
\texttt{pos\_encode\_mode} selects the relative position encoding
style. The fixed block records the methodological identity of HopCPT.
The encoder predicts the absolute residual. The Hopfield memory is
updated online during inference. Training and calibration data are
kept disjoint. The pipeline feeds the encoder a context window of 64
past time steps.

\clearpage

\subsubsection{RareCP}

\begin{table}[h]
  \caption{RareCP hyperparameter configuration used by the CP-methods full-result tables. \emph{Top:} best values from a single 100-trial Optuna TPE search (ranked by mean nWink on the Bench10 val cache); the same best values are reused on Bench22 and the 100k/full sets. \emph{Bottom:} architectural and training settings held fixed across backbones and benchmarks. Operational knobs (epochs, batch sizes, dropouts, gradient clipping, smooth-Winkler annealing schedule) are searched but omitted here for compactness.}
  \label{tab:rarecp_hpo}
  \centering
  \setlength{\tabcolsep}{4pt}
  \resizebox{\textwidth}{!}{%
  \begin{tabular}{l l c c c c}
    \toprule
    \textbf{Hyperparameter} & \textbf{Search range} & \textbf{TiRex} & \textbf{TimesFM-2.5} & \textbf{TabPFN-TS} & \textbf{ARIMA} \\
    \midrule
    \multicolumn{6}{l}{\textit{Searched (Optuna TPE, 100 trials per stage)}} \\
    \addlinespace[2pt]
    \multicolumn{6}{l}{\quad\textit{Base block}} \\
    student\_hidden\_dim & $[32,\,128]$ step 16 & 112 & 96 & 32 & 112 \\
    student\_hidden\_layers & $[1,\,4]$ & 4 & 1 & 1 & 3 \\
    student\_lr & $[2\!\cdot\!10^{-4},\,3\!\cdot\!10^{-3}]$ (log) & $9.59\!\cdot\!10^{-4}$ & $4.98\!\cdot\!10^{-4}$ & $2.76\!\cdot\!10^{-3}$ & $2.40\!\cdot\!10^{-3}$ \\
    student\_target\_lambda & $[1.0,\,20.0]$ (log) & 1.12 & 5.24 & 11.76 & 1.36 \\
    topk & $[8,\,96]$ step 8 & 48 & 32 & 32 & 32 \\
    hopfield\_beta & $[6.0,\,24.0]$ & 7.60 & 10.98 & 13.10 & 12.85 \\
    \addlinespace[2pt]
    \multicolumn{6}{l}{\quad\textit{MoE-gate block}} \\
    gate\_hidden\_dim & $\{1,2,4,8,16,32,64\}$ & 1 & 2 & 32 & 4 \\
    gate\_lr & $[5\!\cdot\!10^{-4},\,8\!\cdot\!10^{-2}]$ (log) & $4.01\!\cdot\!10^{-3}$ & 0.0307 & $7.67\!\cdot\!10^{-4}$ & $3.56\!\cdot\!10^{-3}$ \\
    gate\_entropy\_lambda & $[0,\,0.05]$ & 0.0242 & $9.71\!\cdot\!10^{-3}$ & 0.0199 & 0.0341 \\
    \addlinespace[2pt]
    \multicolumn{6}{l}{\quad\textit{ACI block}} \\
    aci\_gamma & $[10^{-3},\,10^{-1}]$ (log) & 0.0143 & 0.0155 & $4.81\!\cdot\!10^{-3}$ & $9.17\!\cdot\!10^{-3}$ \\
    \midrule
    \multicolumn{6}{l}{\textit{Fixed across backbones and benchmarks}} \\
    \addlinespace[2pt]
    encoder\_type & \multicolumn{5}{c}{hypernetwork\_linear} \\
    hypernetwork\_conditioning & \multicolumn{5}{c}{sample\_plus\_dataset} \\
    hypernetwork\_descriptor & \multicolumn{5}{c}{stats} \\
    smooth\_loss & \multicolumn{5}{c}{cdf\_sigmoid\_winkler} \\
    mixture\_space & \multicolumn{5}{c}{weight} \\
    train\_support\_mode & \multicolumn{5}{c}{loo} \\
    \bottomrule
  \end{tabular}%
  }
\end{table}

Table \ref{tab:rarecp_hpo} groups the searched parameters into four
functional blocks. The first block sizes and trains the student MLP.
\texttt{student\_hidden\_dim} and \texttt{student\_hidden\_layers} set
its width and depth, \texttt{student\_lr} its learning rate.
\texttt{student\_target\_lambda} is the weight on the smoothed Winkler
term in the student loss. The second block configures the Hopfield
retriever that builds the per sample calibration support set.
\texttt{topk} fixes how many nearest neighbours are drawn from the
training set. \texttt{hopfield\_beta} sets the sharpness of the soft
retrieval. Large $\beta$ recovers hard nearest neighbour retrieval;
small $\beta$ averages broadly across the set. The third block
parameterises the MoE gate that routes between expert calibration heads
in weight space. \texttt{gate\_hidden\_dim} sets the gate width and
\texttt{gate\_lr} its learning rate. \texttt{gate\_entropy\_lambda}
regularises the gate distribution and prevents it from collapsing onto
a single expert. The fourth block contains a single online parameter:
\texttt{aci\_gamma}, the step size of the post hoc ACI threshold
update. Larger values adapt faster to nonstationarity but overreact to
short horizon noise. The fixed block records the methodological
identity of RareCP. The encoder is a linear hypernetwork conditioned on
a sample plus dataset \emph{stats} descriptor. The smoothed loss is the
sigmoid CDF surrogate of the Winkler score. The MoE combines experts
in weight space, and the calibration support set during training is
constructed with leave one out splits.

\clearpage

\subsection{Runtime analysis}
\label{apx:runtime}

\begin{table}[h]
  \caption{Wall-clock on Bench10 (10 datasets, 7,500 test samples each). Train (s/run) and Train (s/epoch) are the per-seed offline fit at 100 epochs; non-parametric methods omit the train columns. Predict (s/run) is one full evaluation pass and Predict (s/dataset) displays wall-clock per dataset. $^{\dagger}$RareCP additionally needs a one-time teacher bank trained at 100 epochs (46.34\,s, 0.46\,s/epoch) shared across seeds. Mean and std are reported over ten seeds.}
  \label{tab:wallclock_bench10}
  \centering
  \begin{tabular}{lcccc}
    \toprule
    \textbf{Method} & Train (s/run) $\downarrow$ & Train (s/epoch) $\downarrow$ & Predict (s/run) $\downarrow$ & Predict (s/dataset) $\downarrow$ \\
    \midrule
    Uniform & --- & --- & 4.26\,{\scriptsize $\pm$\,0.04} & 0.43 \\
    ACI     & --- & --- & 4.62\,{\scriptsize $\pm$\,0.04} & 0.46 \\
    DtACI   & --- & --- & 9.96\,{\scriptsize $\pm$\,0.07} & 1.00 \\
    NexCP   & --- & --- & 4.98\,{\scriptsize $\pm$\,0.02} & 0.50 \\
    KOWCPI  & --- & --- & 15.89\,{\scriptsize $\pm$\,0.06} & 1.59 \\
    HopCPT  & 5.21\,{\scriptsize $\pm$\,0.21} & 0.052 & \textbf{1.55}\,{\scriptsize $\pm$\,0.02} & \textbf{0.16} \\
    ResCP   & --- & --- & 31.22\,{\scriptsize $\pm$\,0.16} & 3.12 \\
    RareCP\,\textit{(Ours)} & 175.68\,{\scriptsize $\pm$\,2.01}\,$^{\dagger}$ & 1.76 & 27.33\,{\scriptsize $\pm$\,0.97} & 2.73 \\
    \bottomrule
  \end{tabular}
\end{table}

All wall-clock measurements were collected locally on a single NVIDIA A6000 GPU (48GB VRAM) and a 64-core CPU with 256GB RAM against the shared TiREx prediction cache, so backbone inference is excluded from every entry and only the conformal-prediction layer is timed. The non-parametric baselines (SplitCP Uniform, ACI, DtACI, NexCP, KOWCPI, ResCP) have no offline training step, hence their entire cost is the prediction loop. HopCPT trains one small Hopfield-attention encoder per dataset, while RareCP trains ten hypernetwork-linear students plus a posthoc MoE gate as default. Both methods have the same fixed budget of 100 epochs so per-epoch costs are directly comparable. RareCP runs evaluation as a single batched 10-dataset GPU pass, so its per-dataset entry is the total divided by ten. The 46.34\,s teacher bank in the table footnote is a one-time foundation cost shared across all seeds (analogous to using a pretrained backbone) and is therefore not amortised into the per-seed train column.

\subsection{Parameter count}
\label{apx:param}

\begin{table}[h]
  \caption{Parameter counts of the foundation models and the RareCP used throughout the evaluation. The rarecp wrapper adds at most ${\sim}5.5$M parameters on top of the backbone, two orders of magnitude smaller than the largest FM baseline.}
  \label{tab:model_param_counts}
  \centering
  \small
  \begin{tabular}{@{}l r@{}}
    \toprule
    \textbf{Model / wrapper} & \textbf{Parameters} \\
    \midrule
    TiRex                    & 35M \\
    TimesFM-2.5              & 200M \\
    TabPFN-TS                & 11M \\
    Chronos-2                & 120M \\
    Chronos-Bolt             & 205M \\
    Moirai-2                 & 11.4M \\
    Toto                     & 151M \\
    Lag-Llama                & 2.45M \\
    RareCP (10 experts) & 1.5--5.5M\textsuperscript{$\dagger$} \\
    \bottomrule
  \end{tabular}
  \\[2pt]
  {\footnotesize $\dagger$\,Per backbone: TiRex 5.45M, TimesFM-2.5 4.35M, TabPFN-TS 1.50M, ARIMA 5.33M.}
\end{table}

\clearpage

\subsection{Visualizations}
\label{apx:visual}

\begin{figure}[h]
    \centering
    \includegraphics[width=1.0\linewidth]{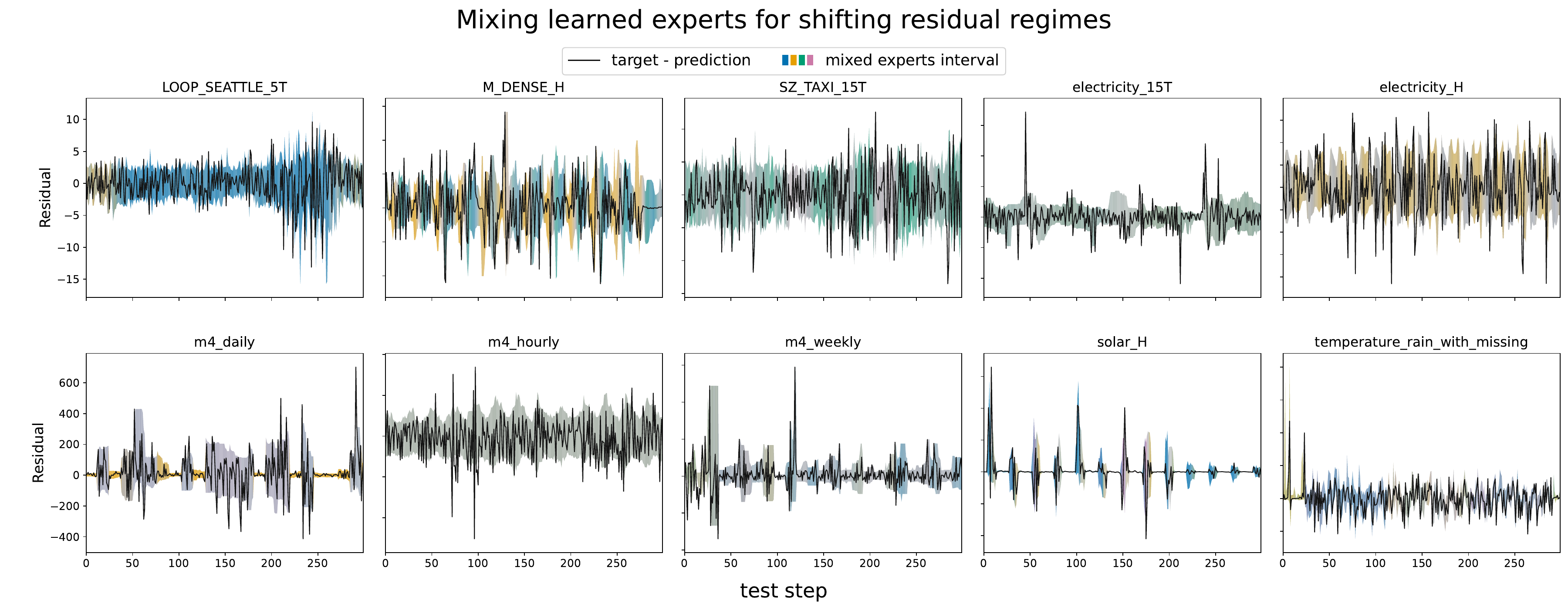}
    \caption{Visualization of 80\% intervals using the gating mixture on the Bench10 datasets.}
    \label{fig:full_gating}
\end{figure}

\begin{figure}[h]
  \centering
  \includegraphics[width=\linewidth]{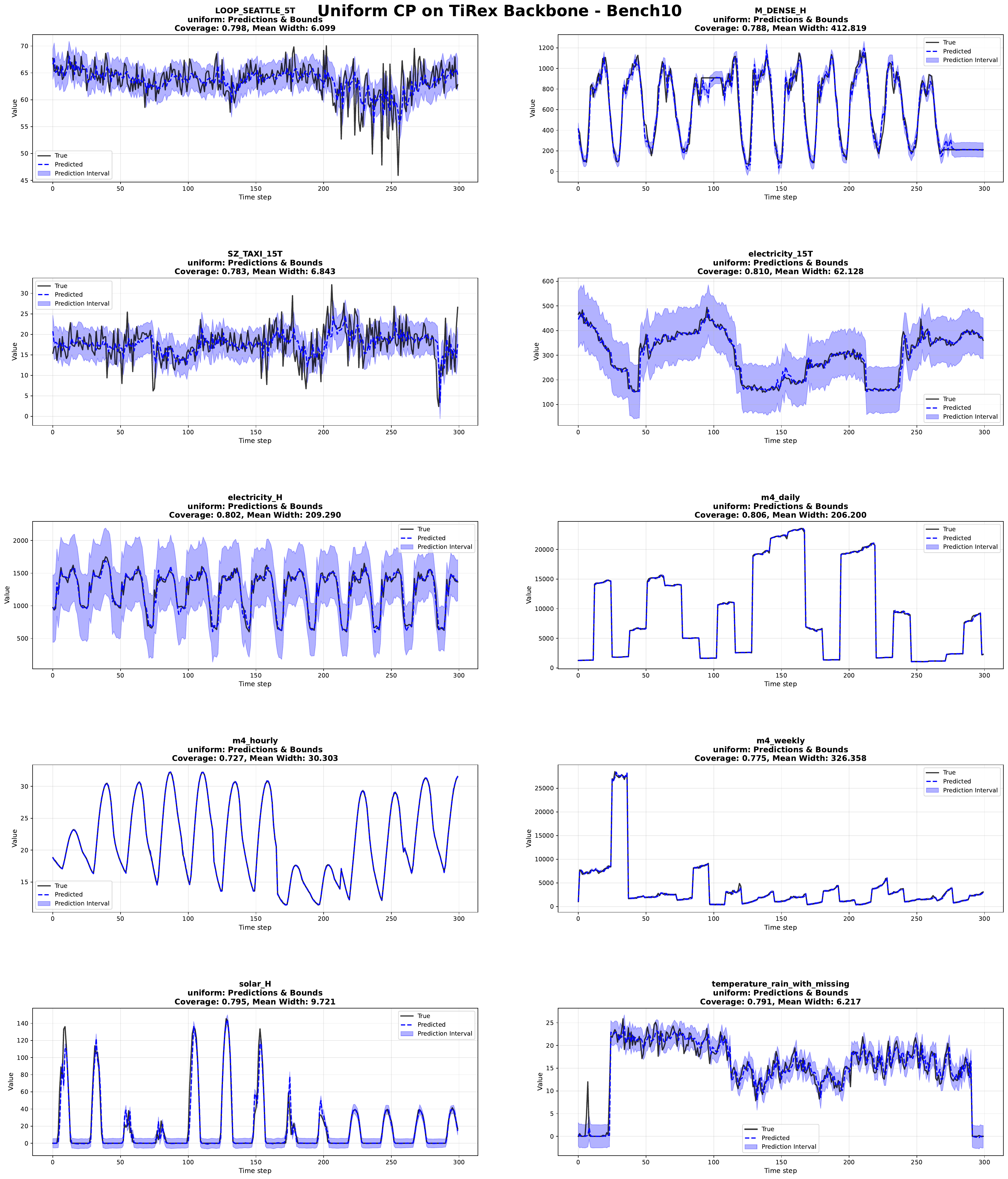}
  \caption{Visualization of SplitCP Uniform predictions on the first 300 test samples per dataset; Bench10 on TiRex backbone.}
  \label{fig:ts_res_tirex_uni}
\end{figure}

\begin{figure}[h]
  \centering
  \includegraphics[width=\linewidth]{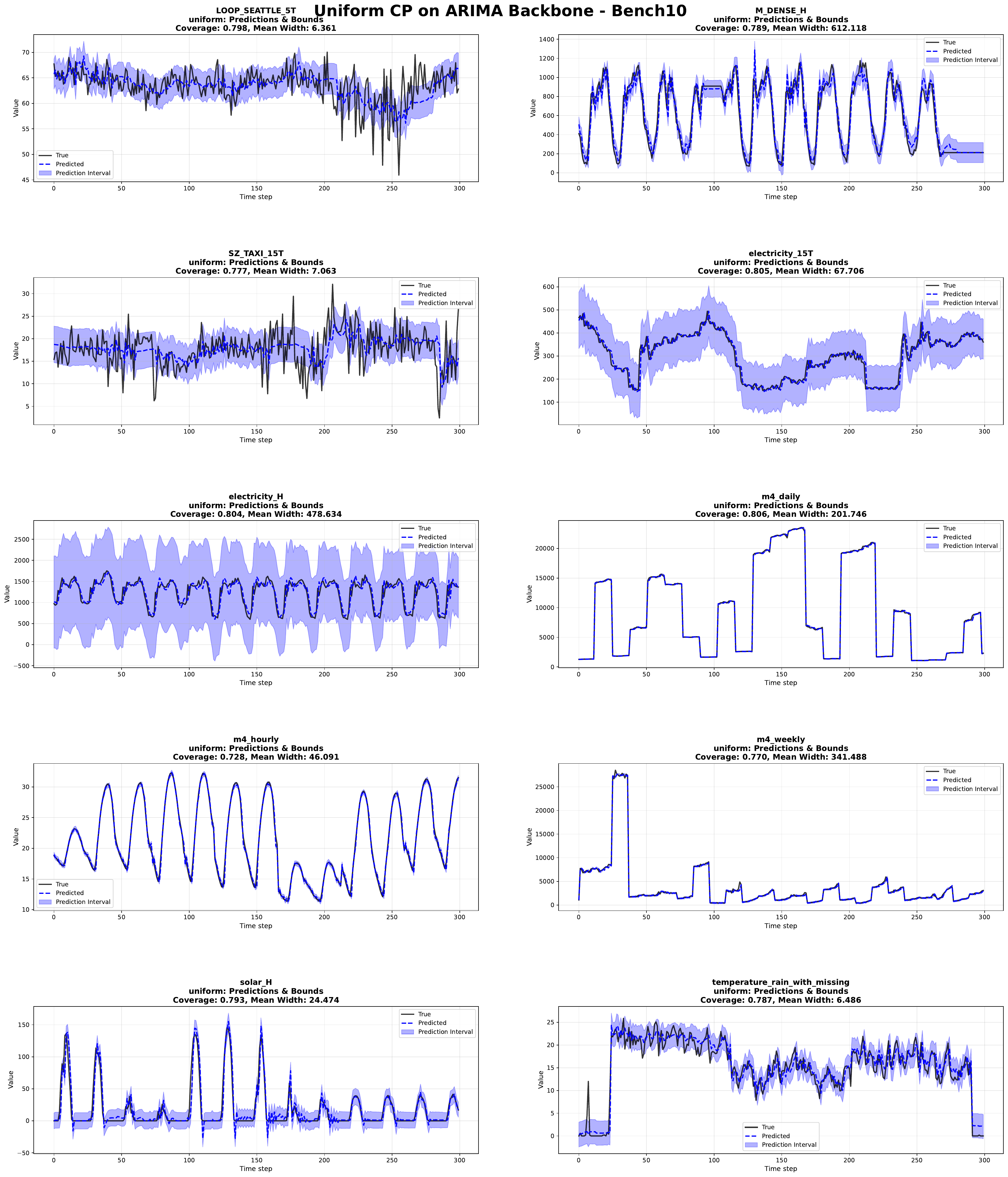}
  \caption{Visualization of SplitCP Uniform predictions on the first 300 test samples per dataset; Bench10 on ARIMA backbone.}
  \label{fig:ts_res_arima_uni}
\end{figure}

\begin{figure}[h]
  \centering
  \includegraphics[width=\linewidth]{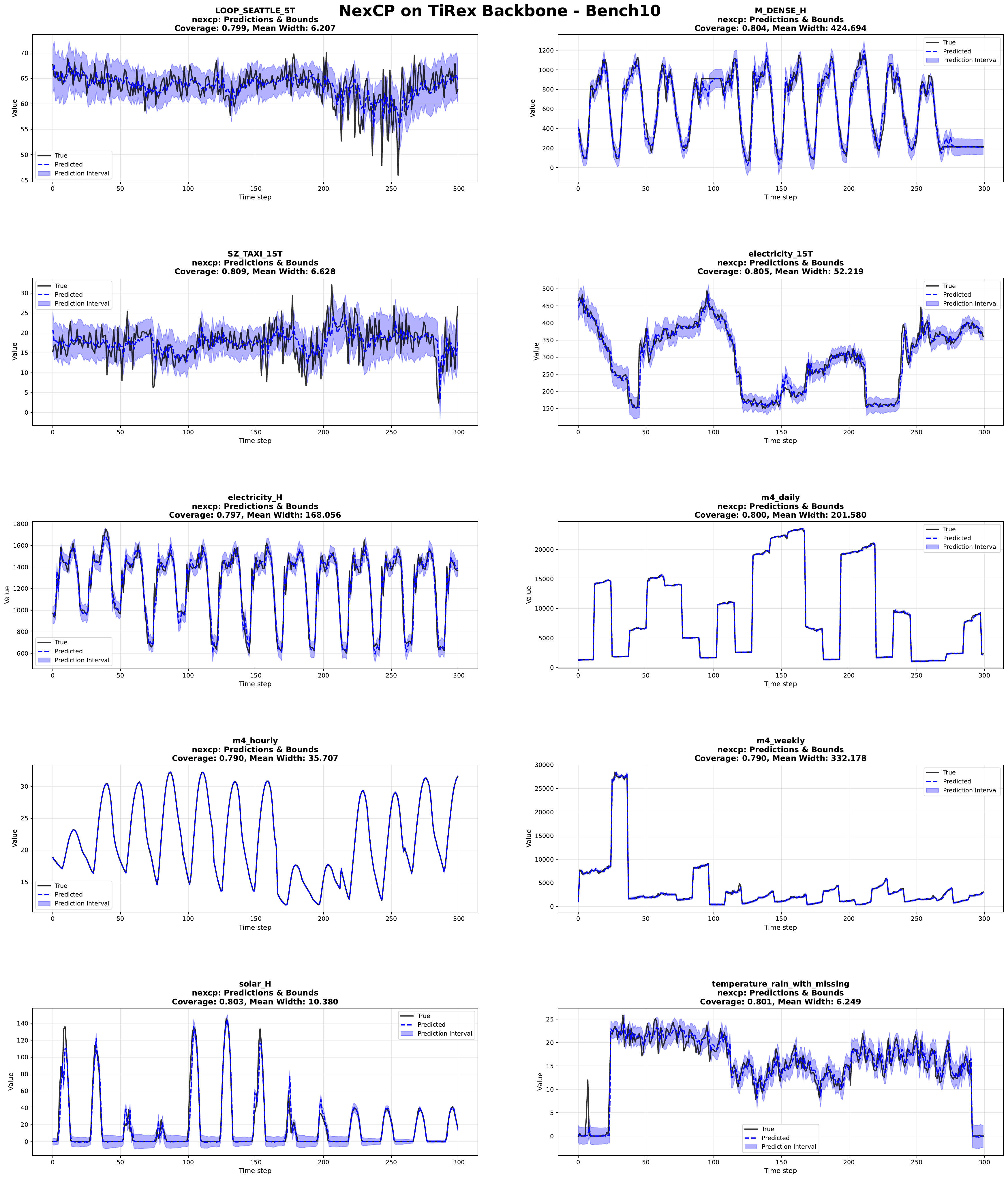}
  \caption{Visualization of NexCP predictions on the first 300 test samples per dataset; Bench10 on TiRex backbone.}
  \label{fig:ts_res_tirex_nex}
\end{figure}

\begin{figure}[h]
  \centering
  \includegraphics[width=\linewidth]{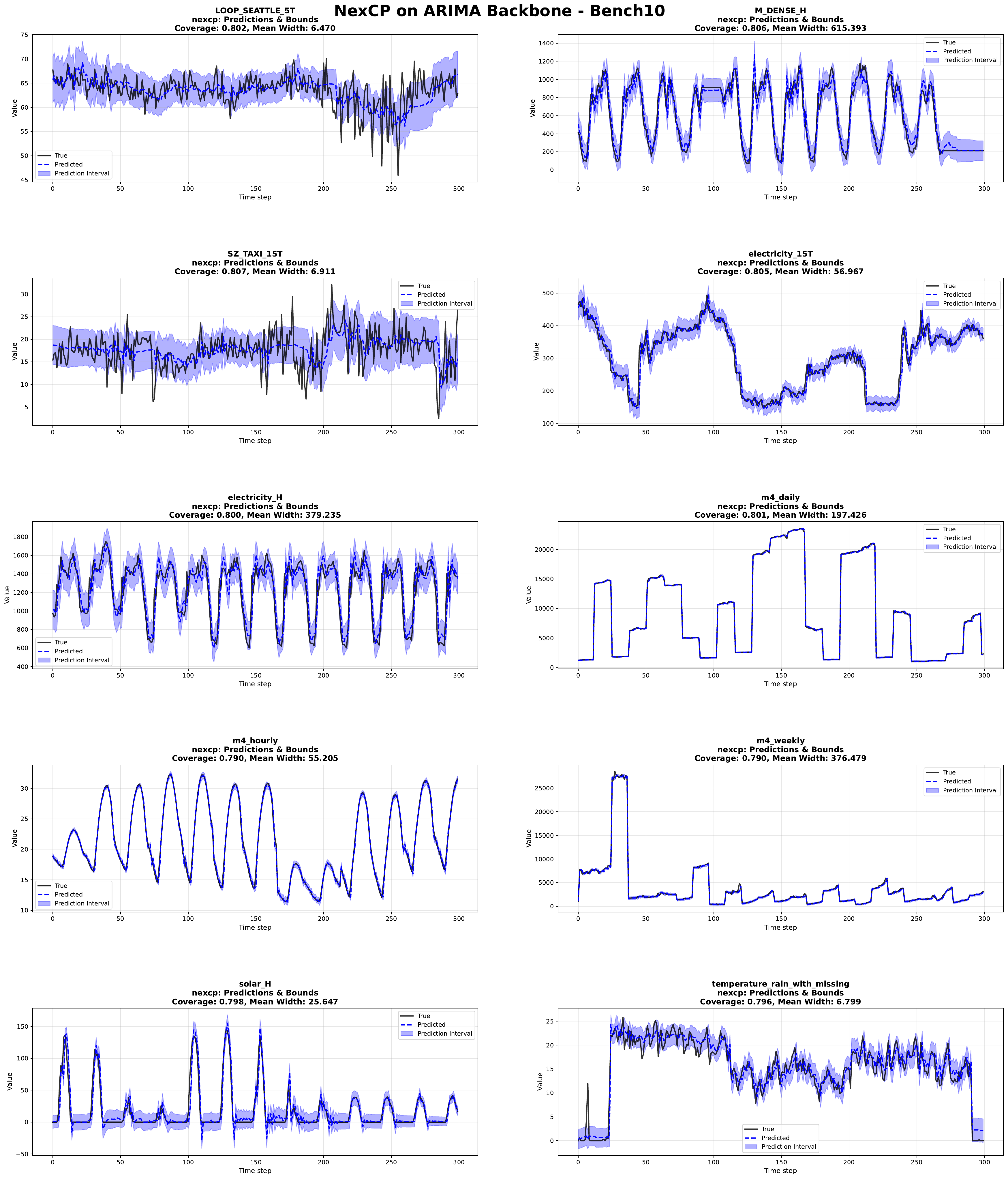}
  \caption{Visualization of NexCP predictions on the first 300 test samples per dataset; Bench10 on ARIMA backbone.}
  \label{fig:ts_res_arima_nex}
\end{figure}

\begin{figure}[h]
  \centering
  \includegraphics[width=\linewidth]{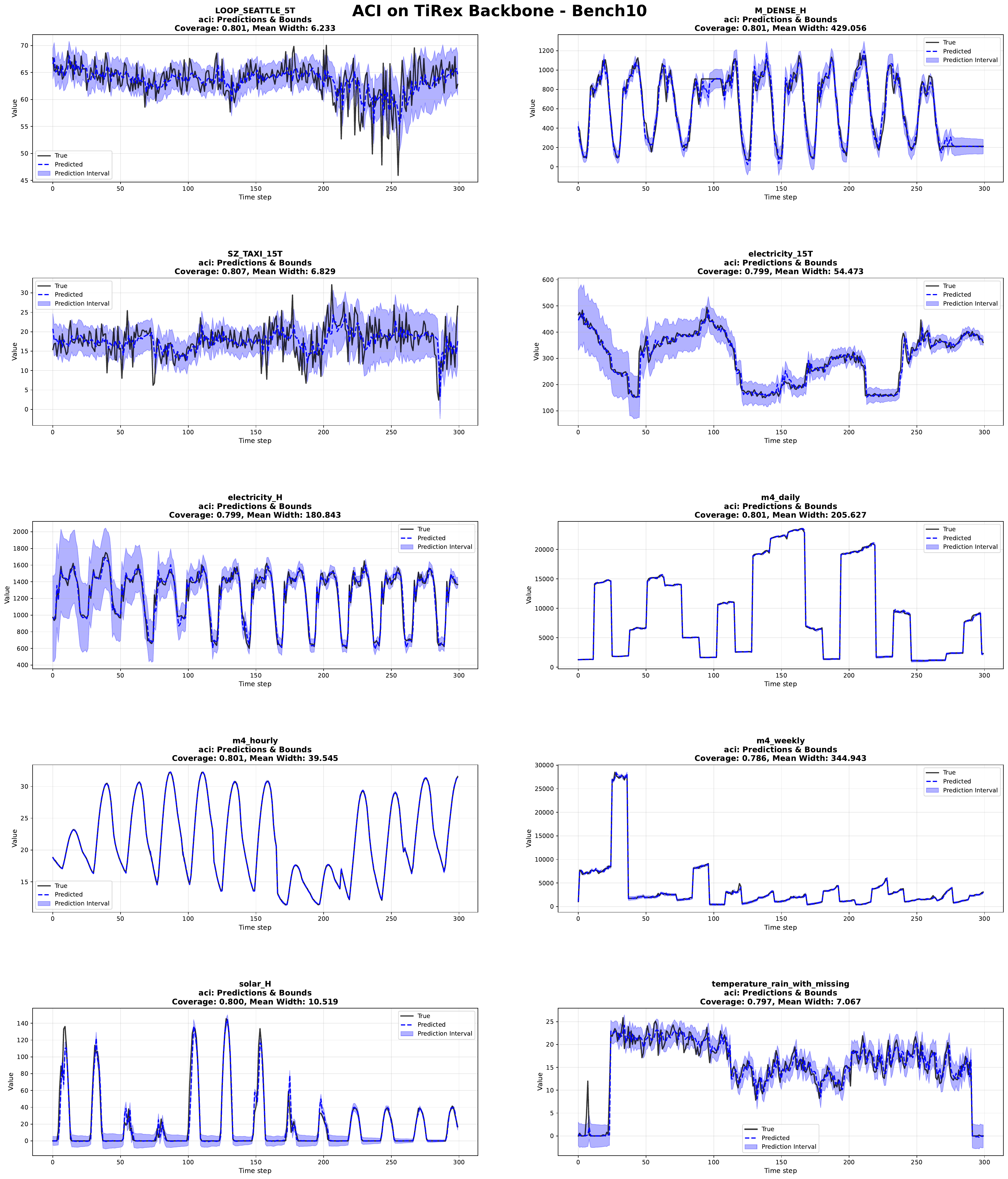}
  \caption{Visualization of ACI predictions on the first 300 test samples per dataset; Bench10 on TiRex backbone.}
  \label{fig:ts_res_tirex_aci}
\end{figure}

\begin{figure}[h]
  \centering
  \includegraphics[width=\linewidth]{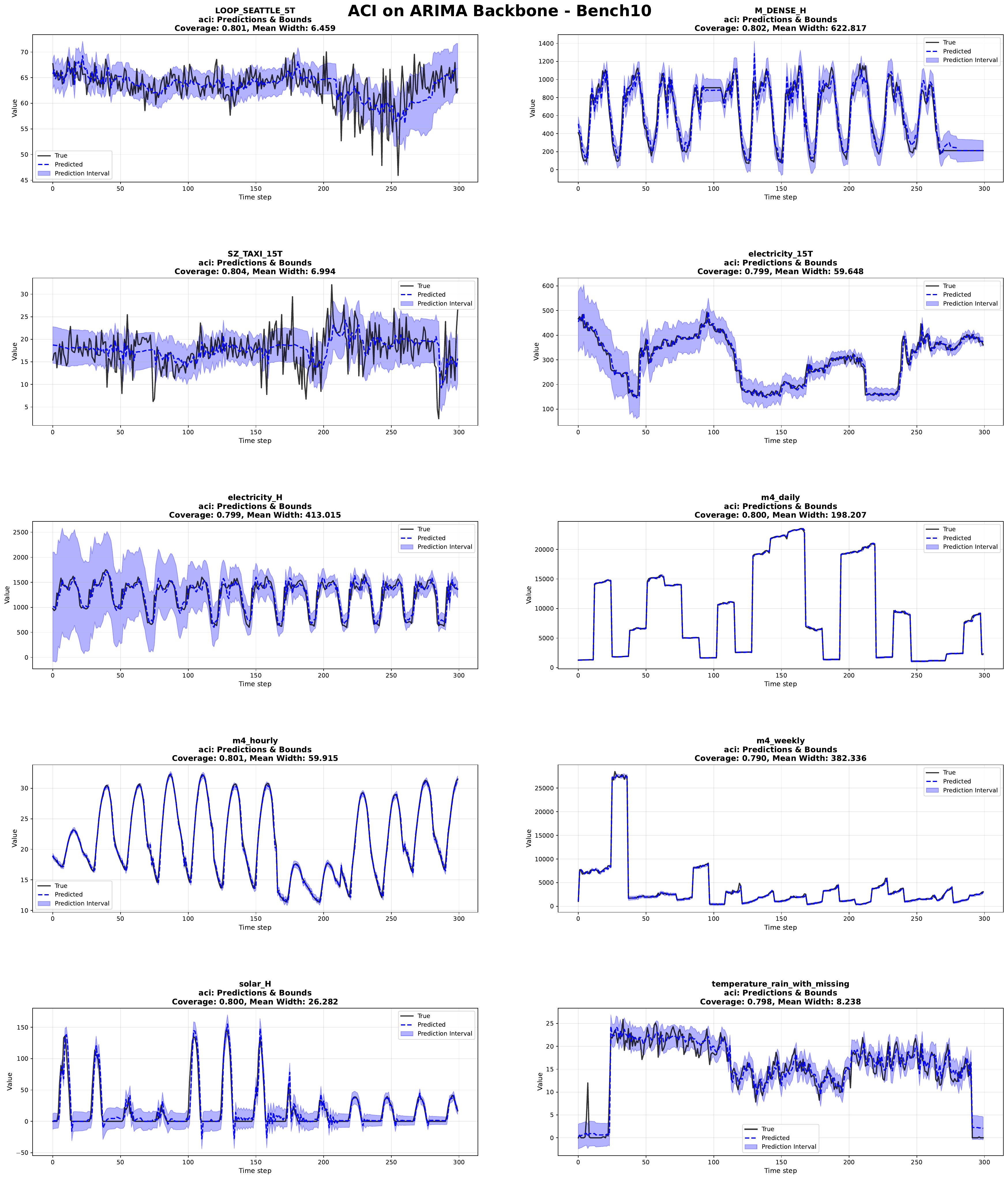}
  \caption{Visualization of ACI predictions on the first 300 test samples per dataset; Bench10 on ARIMA backbone.}
  \label{fig:ts_res_arima_aci}
\end{figure}

\begin{figure}[h]
  \centering
  \includegraphics[width=\linewidth]{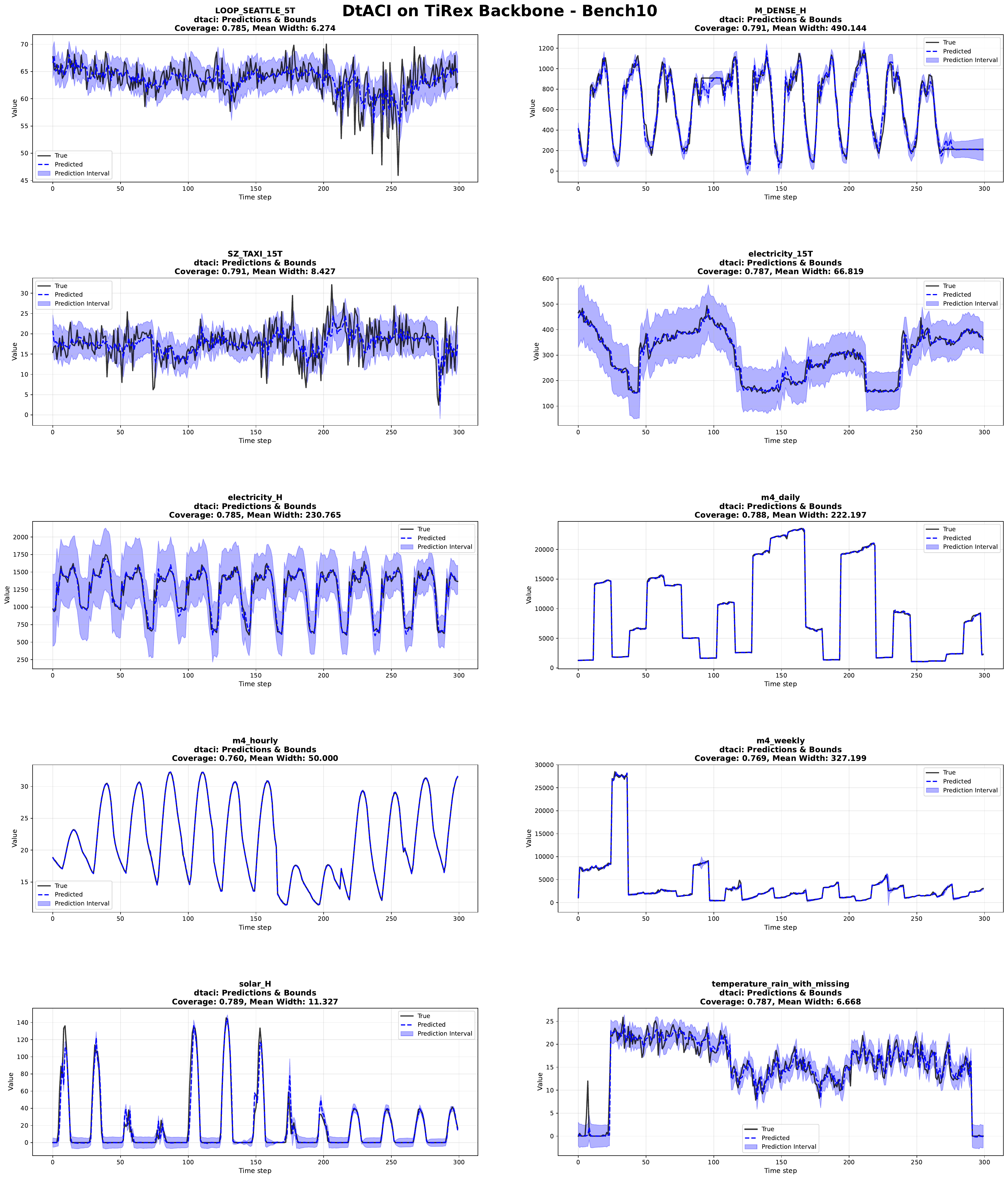}
  \caption{Visualization of dtACI predictions on the first 300 test samples per dataset; Bench10 on TiRex backbone.}
  \label{fig:ts_res_tirex_dtaci}
\end{figure}

\begin{figure}[h]
  \centering
  \includegraphics[width=\linewidth]{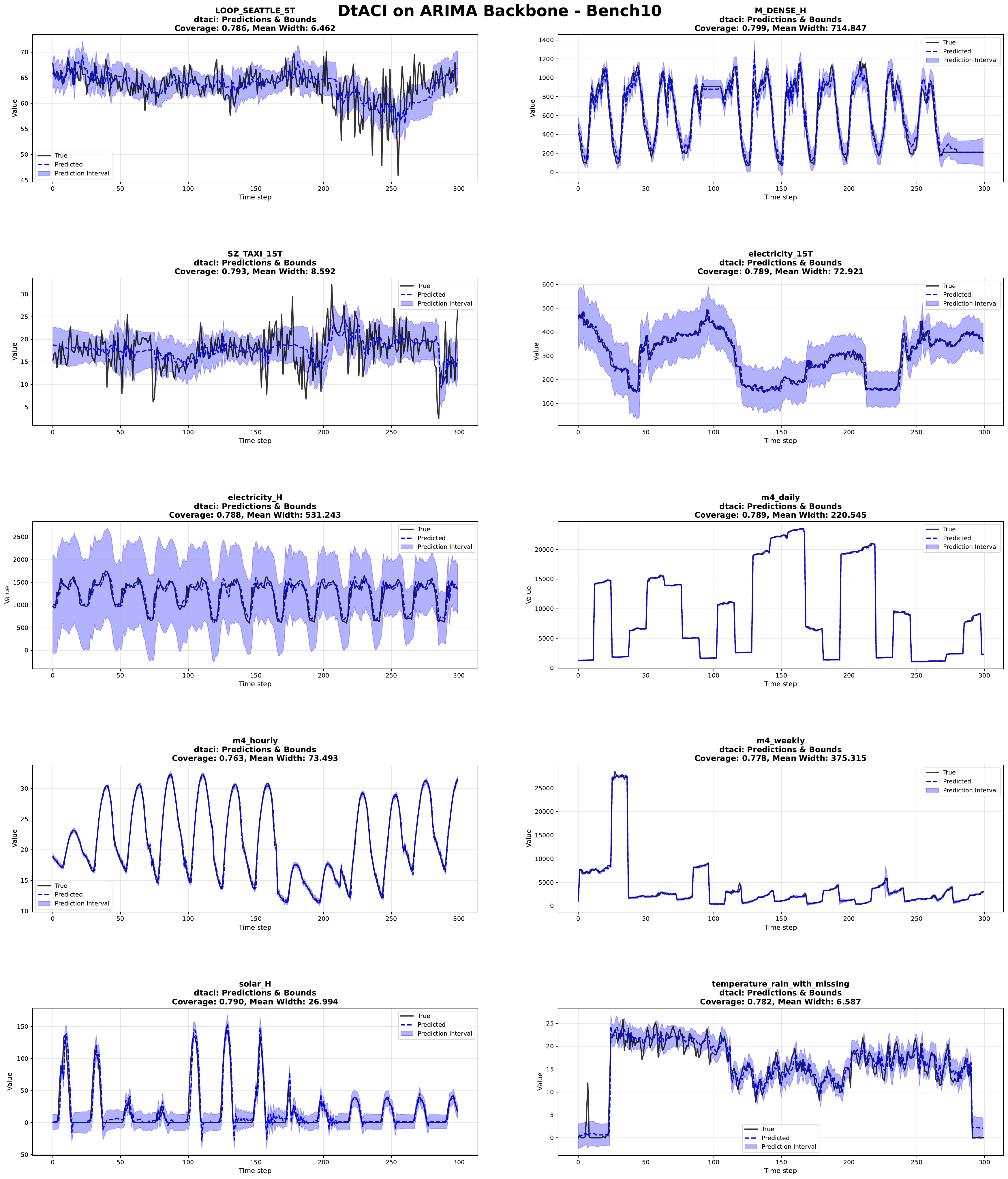}
  \caption{Visualization of dtACI predictions on the first 300 test samples per dataset; Bench10 on ARIMA backbone.}
  \label{fig:ts_res_arima_dtaci}
\end{figure}

\begin{figure}[h]
  \centering
  \includegraphics[width=\linewidth]{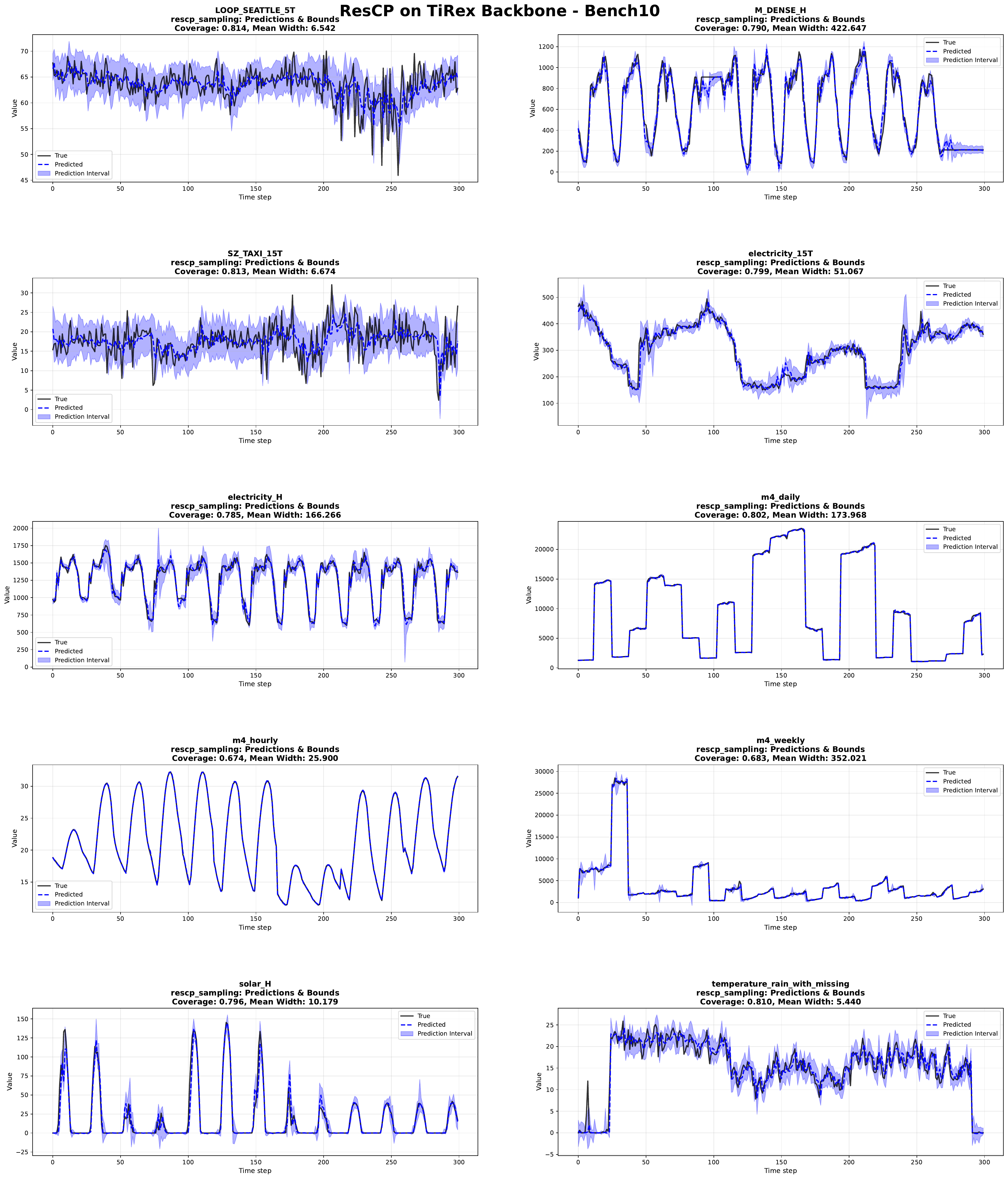}
  \caption{Visualization of ResCP predictions on the first 300 test samples per dataset; Bench10 on TiRex backbone.}
  \label{fig:ts_res_tirex_rescp}
\end{figure}

\begin{figure}[h]
  \centering
  \includegraphics[width=\linewidth]{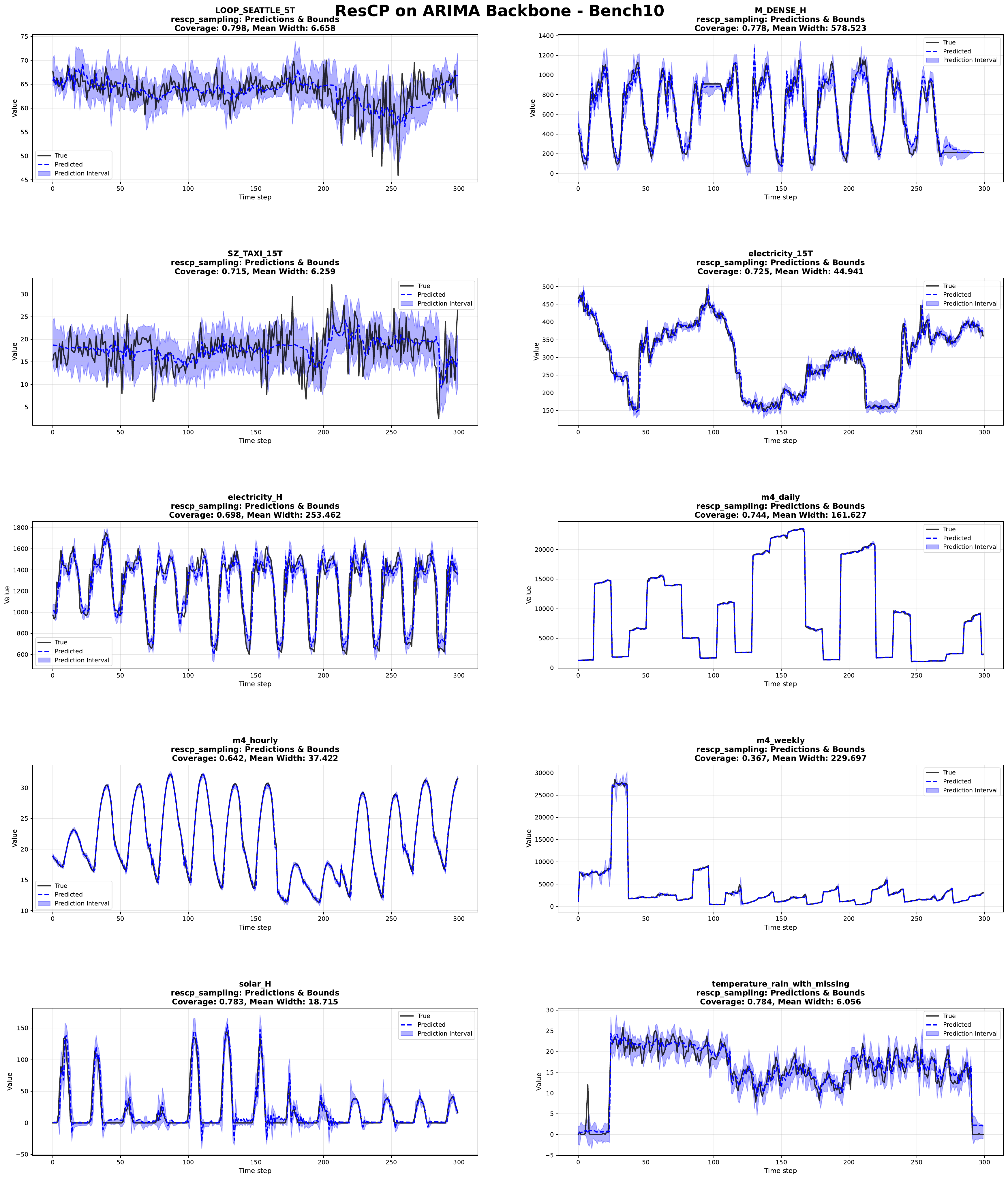}
  \caption{Visualization of ResCP predictions on the first 300 test samples per dataset; Bench10 on ARIMA backbone.}
  \label{fig:ts_res_arima_ResCP}
\end{figure}

\begin{figure}[h]
  \centering
  \includegraphics[width=\linewidth]{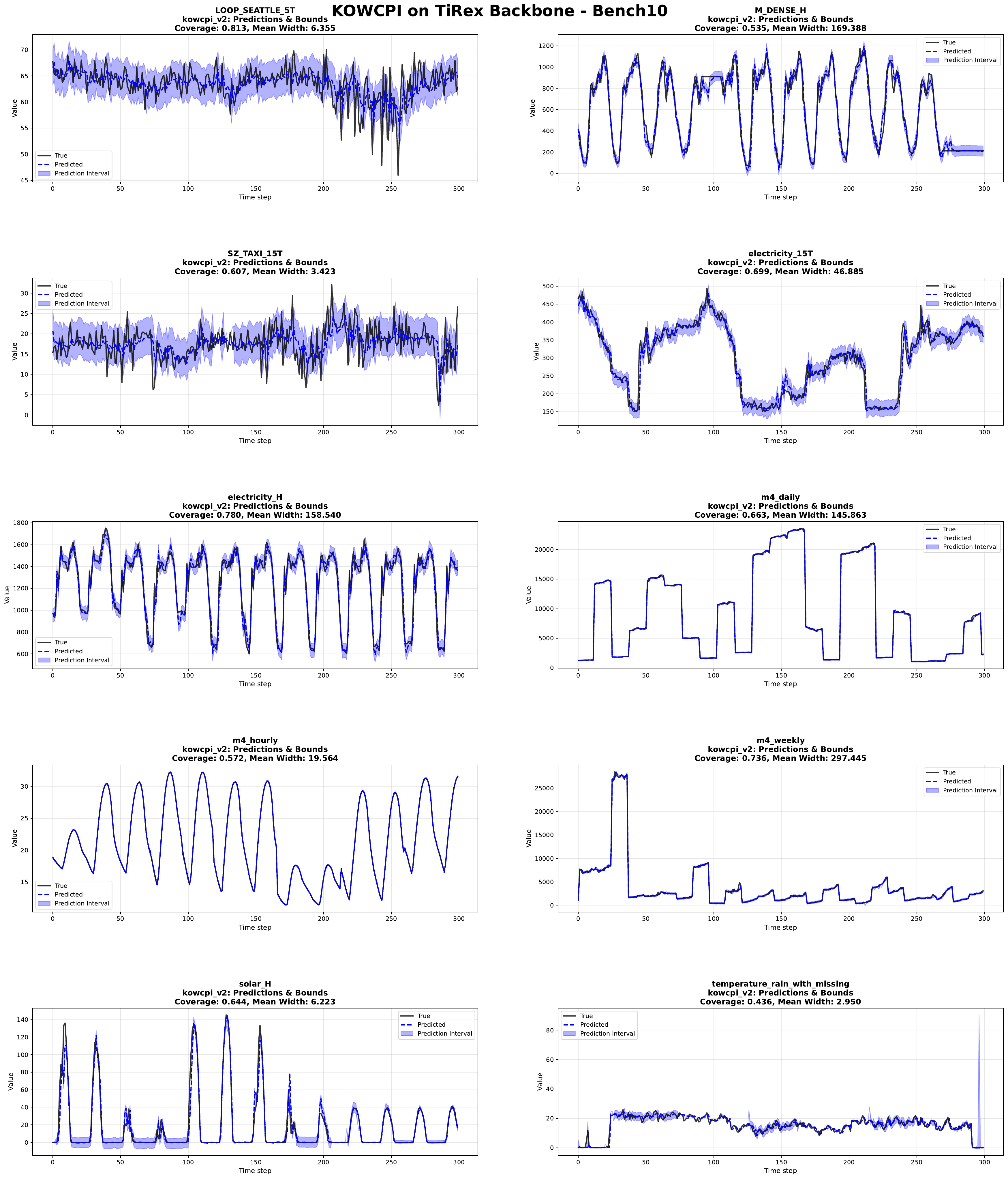}
  \caption{Visualization of KOWCPI predictions on the first 300 test samples per dataset; Bench10 on TiRex backbone.}
  \label{fig:ts_res_tirex_kow}
\end{figure}

\begin{figure}[h]
  \centering
  \includegraphics[width=\linewidth]{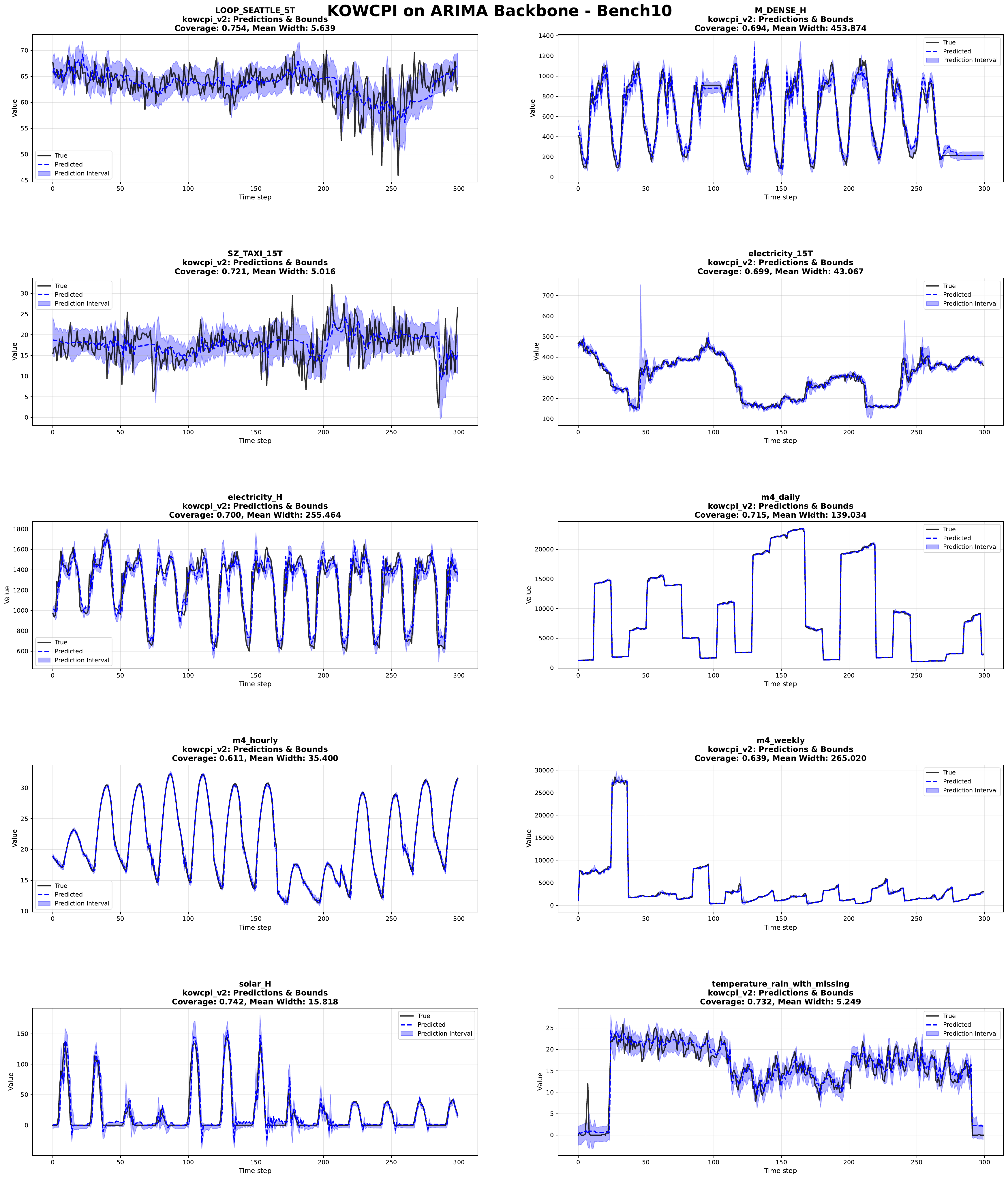}
  \caption{Visualization of KOWCPI predictions on the first 300 test samples per dataset; Bench10 on ARIMA backbone.}
  \label{fig:ts_res_arima_kow}
\end{figure}

\begin{figure}[h]
  \centering
  \includegraphics[width=\linewidth]{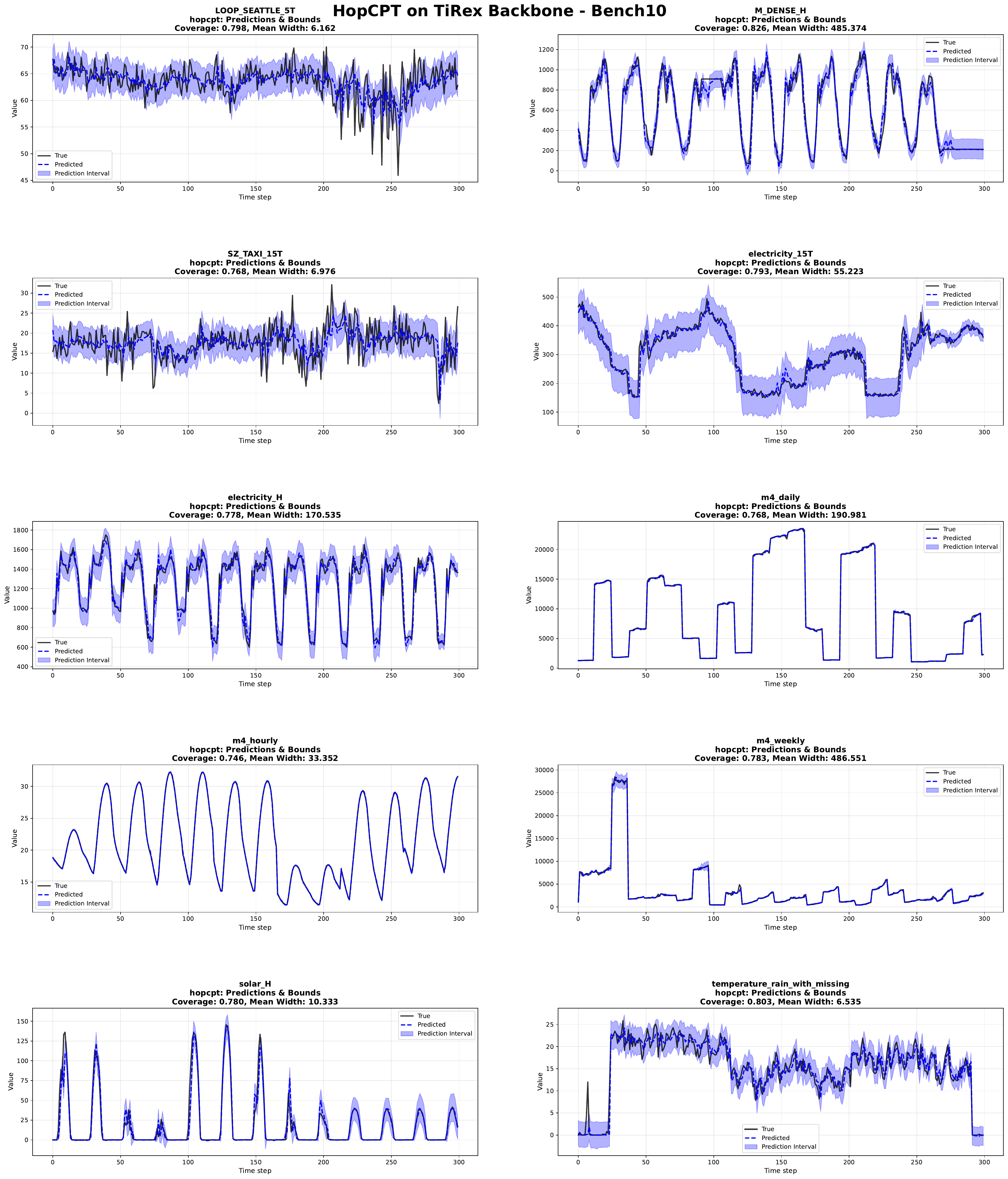}
  \caption{Visualization of HopCPT predictions on the first 300 test samples per dataset; Bench10 on TiRex backbone.}
  \label{fig:ts_res_tirex_hop}
\end{figure}

\begin{figure}[h]
  \centering
  \includegraphics[width=\linewidth]{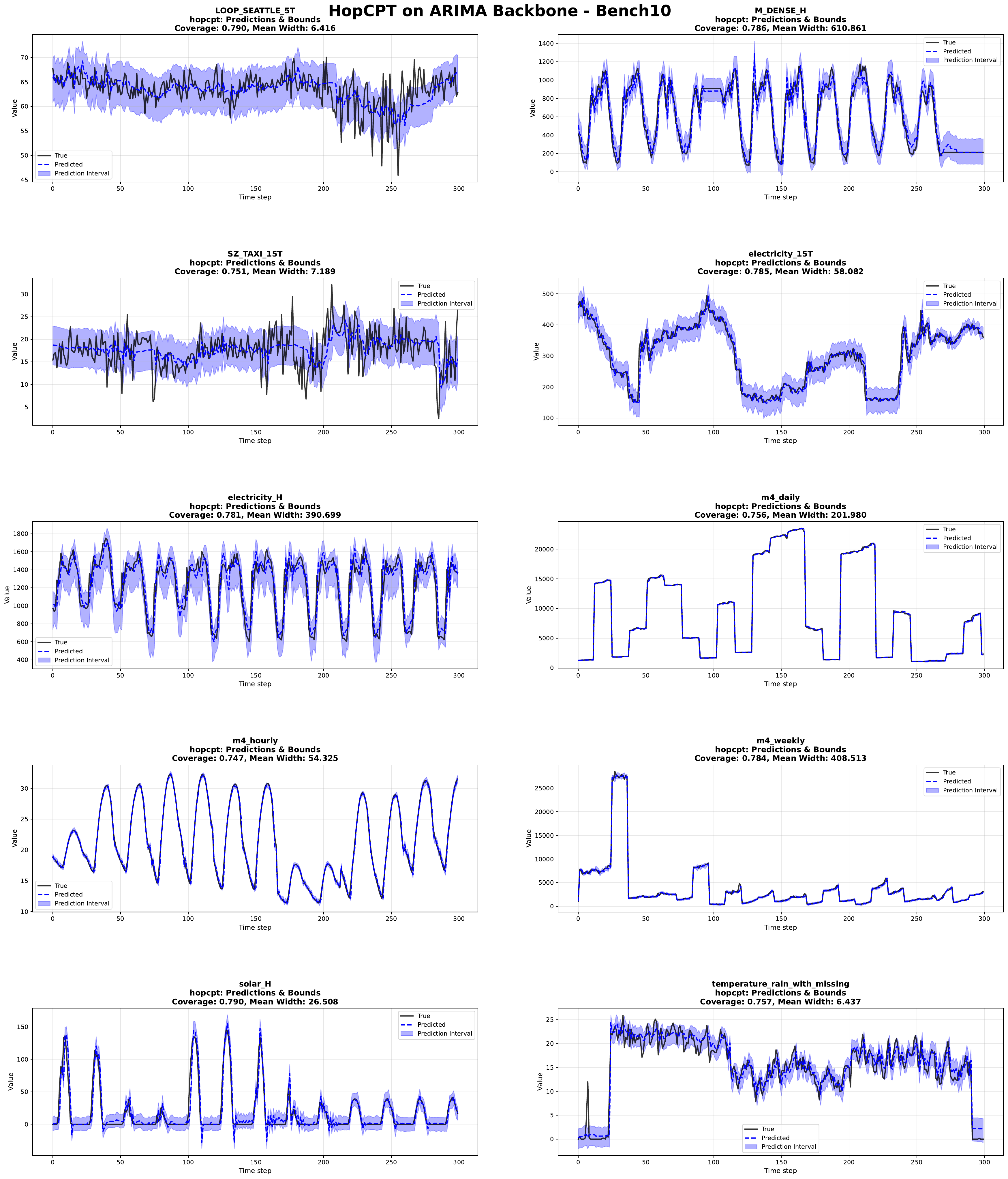}
  \caption{Visualization of HopCPT predictions on the first 300 test samples per dataset; Bench10 on ARIMA backbone.}
  \label{fig:ts_res_arima_hop}
\end{figure}

\begin{figure}[h]
  \centering
  \includegraphics[width=\linewidth]{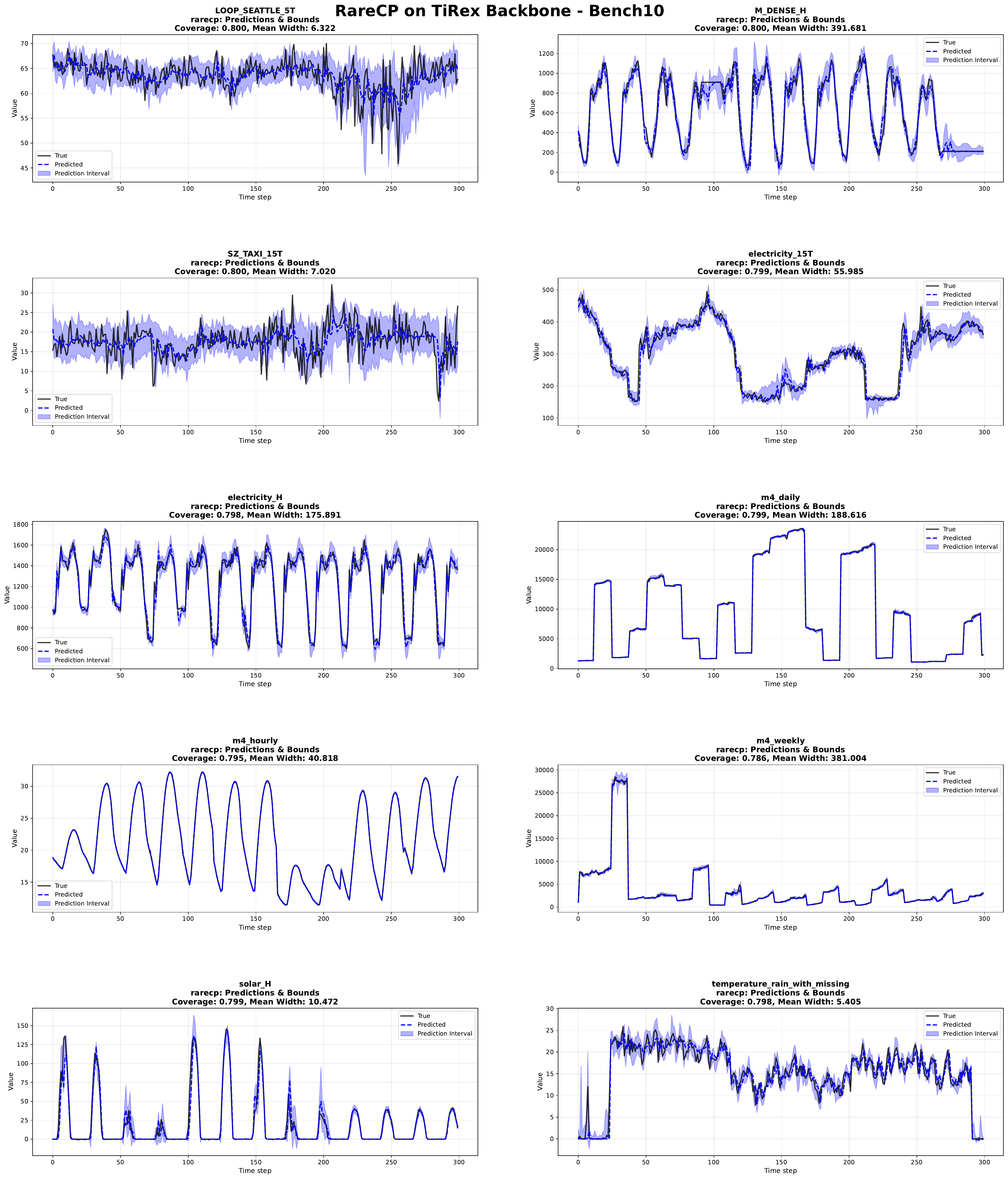}
  \caption{Visualization of RareCP predictions on the first 300 test samples per dataset; Bench10 on TiRex backbone.}
  \label{fig:ts_res_tirex_rare}
\end{figure}

\begin{figure}[h]
  \centering
  \includegraphics[width=\linewidth]{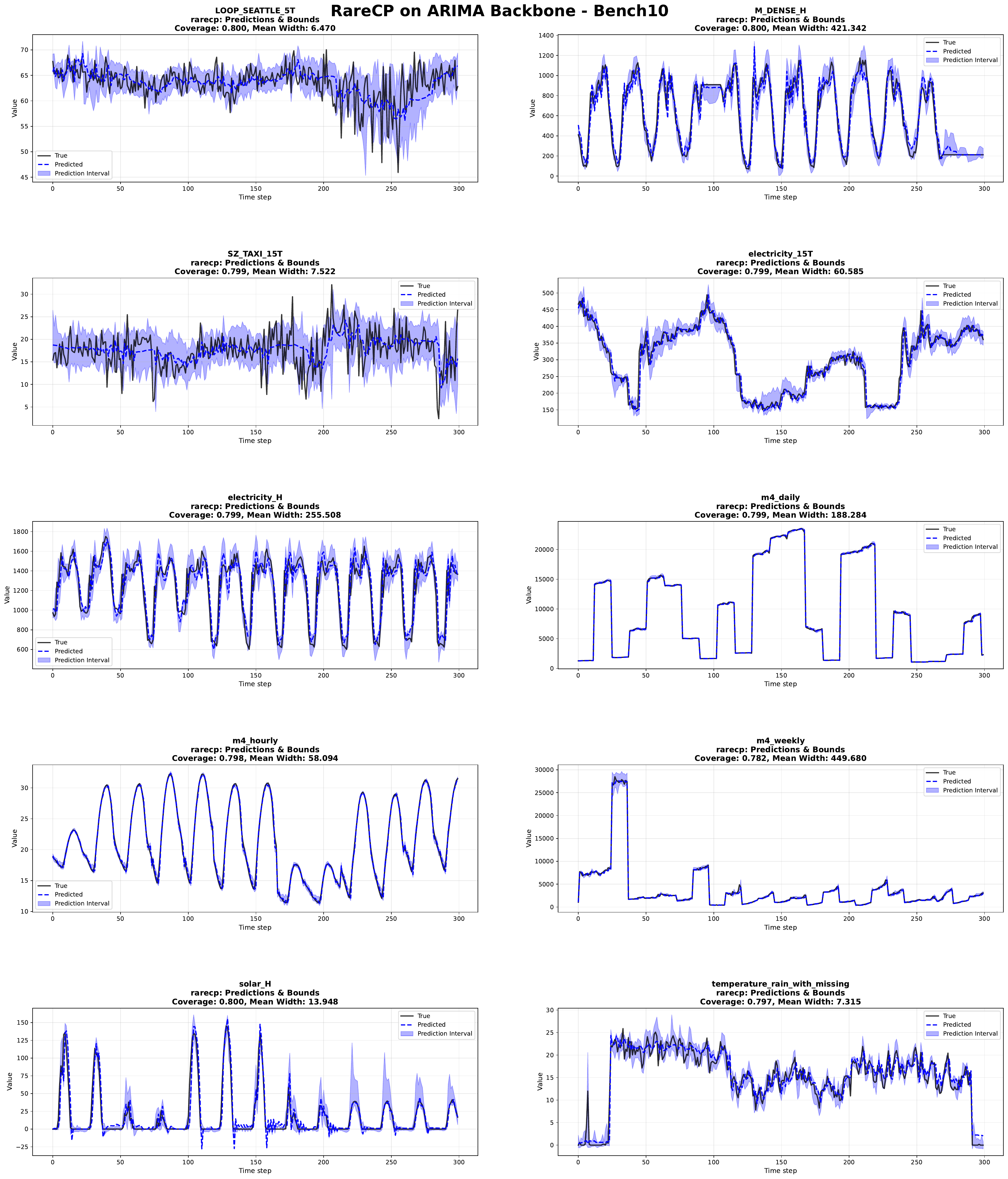}
  \caption{Visualization of RareCP predictions on the first 300 test samples per dataset; Bench10 on ARIMA backbone.}
  \label{fig:ts_res_arima_rare}
\end{figure}


\end{document}